\documentclass[12pt,a4paper,oneside]{report}

\usepackage[T1]{fontenc}
\usepackage{lmodern}
\usepackage{amsmath,amssymb,amsthm,mathtools,bm}
\usepackage{geometry}
\usepackage{setspace}
\usepackage{microtype}
\usepackage{booktabs,longtable,tabularx,array,multirow}
\usepackage{enumitem}
\usepackage{graphicx}
\usepackage{xcolor}
\usepackage{hyperref}
\usepackage[nameinlink,capitalise,noabbrev]{cleveref}
\usepackage[numbers,sort&compress]{natbib}
\usepackage{listings}
\usepackage{fancyhdr}
\usepackage{makeidx}
\usepackage{etoolbox}

\makeindex

\definecolor{deepblue}{RGB}{18,58,112}
\definecolor{softgray}{RGB}{245,245,245}
\hypersetup{
  colorlinks=true,
  linkcolor=deepblue,
  citecolor=deepblue,
  urlcolor=deepblue,
  pdftitle={Research Report on Noise-Shaped One-Bit Coefficients in Discrete Polynomial Fourier Extension},
  pdfauthor={Wang Shengquan}
}

\newtheorem{theorem}{Theorem}[chapter]
\newtheorem{lemma}[theorem]{Lemma}
\newtheorem{proposition}[theorem]{Proposition}
\newtheorem{corollary}[theorem]{Corollary}

\theoremstyle{definition}
\newtheorem{definition}[theorem]{Definition}

\theoremstyle{remark}
\newtheorem{remark}[theorem]{Remark}

\newcommand{\R}{\mathbb{R}}
\newcommand{\C}{\mathbb{C}}
\newcommand{\Z}{\mathbb{Z}}

\newcommand{\T}{\mathbb{T}}
\newcommand{\E}{\mathcal{E}}

\newcommand{\D}{\Delta}
\newcommand{\Dp}{\Delta_{+}}
\newcommand{\TV}{\operatorname{TV}}
\newcommand{\Var}{\operatorname{Var}}

\newcommand{\e}{\mathrm{e}}
\newcommand{\dd}{\,\mathrm{d}}
\newcommand{\ii}{\mathrm{i}}
\newcommand{\norm}[1]{\left\lVert #1\right\rVert}
\newcommand{\abs}[1]{\left\lvert #1\right\rvert}
\newcommand{\inner}[2]{\left\langle #1,#2\right\rangle}

\newcommand{\one}{\mathbf{1}}

\title{\textbf{Research Report on Noise-Shaped One-Bit Coefficients in Discrete Polynomial Fourier Extension}\\[0.8em]
\Large Detailed Derivations, Rigorous Proofs, Boundary-Corrected Reconstruction, and Higher-Order Extensions}
\author{Wang Shengquan\\[0.5em]}
\date{July 2026}

\begin{document}
\pagenumbering{gobble}
\maketitle
\clearpage
\pagenumbering{roman}
\chapter*{Prefatory Note}
\addcontentsline{toc}{chapter}{Prefatory Note}

This report grew from personal academic interest in the interaction between sigma-delta quantization, discrete finite differences, and polynomial Fourier extension. Its purpose is to develop the underlying identities and estimates in a detailed form that supports further mathematical study.

\chapter*{Abstract}
\addcontentsline{toc}{chapter}{Abstract}

This research report develops a rigorous analytical framework for one-bit noise-shaped coefficients in discrete polynomial Fourier extension. The central object is a finite coefficient record whose quantization error admits a discrete-difference representation. For first-order sigma-delta quantization,
\[
 u_k-q_k=v_k-v_{k-1},
 \qquad \|v\|_{\ell^\infty}\le V,
\]
and the bounded state converts the quantization error into a controlled functional on weight sequences of finite discrete variation. A complete finite-record summation-by-parts argument gives
\[
 \left|\frac1N\sum_{k=1}^{N}(u_k-q_k)w_k\right|
 \le \frac{V}{N}
 \left(|w_N|+\sum_{k=1}^{N-1}|w_{k+1}-w_k|\right)
\]
when the initial state is zero. This estimate is then applied to absolutely continuous and bounded-variation phases, with particular attention to the parabolic phase
\[
 \phi_{x,t}(\xi)=x\xi+t\xi^2.
\]
The resulting local estimate is
\[
 |\mathcal E_Nu(x,t)-\mathcal E_Nq(x,t)|
 \le \frac{V}{N}
 \left[1+2\pi\int_0^1|x+2t\xi|\,\mathrm d\xi\right].
\]
The phase-variation factor is evaluated explicitly, compact-set and local $L^p$ consequences are derived, and a zero-input trajectory proves that the uniform first-order rate cannot generally be improved beyond $N^{-1}$ without additional conditions.

The report also develops the exact finite-interval identity for errors of the form $e=\Delta^rv$. All initial and terminal boundary traces are retained. Under endpoint compatibility, repeated summation by parts transfers $r$ differences from the state to the sampled phase weight and yields an $O(N^{-r})$ compact-set estimate. A complementary boundary-corrected reconstruction formula removes all finite-record traces by using a fixed number of state-trace values, so the same interior rate is recovered without imposing terminal compatibility. Fractional smoothness is also treated: $C^{r-1,\alpha}$ sampled weights yield the intermediate rate $O(N^{-(r-1+\alpha)})$. The analysis is extended to polynomial phases, moment-curve extensions, multidimensional divergence-form errors, anisotropic growing regions, residual error models, blockwise reset, smooth termination, and oscillatory transfer estimates under explicit state hypotheses. Numerical calculations verify indexing, constants, and predicted rates.

The presentation follows the format of a detailed mathematical research report. External results are cited at their substantive use, and each transfer estimate states the additional state assumptions required by its proof.

\noindent\textbf{Keywords:} sigma-delta quantization; one-bit coefficients; noise shaping; discrete Fourier extension; polynomial phase; parabolic extension; summation by parts; total variation; endpoint compatibility; exponential sums; restriction theory.

\chapter*{Executive Summary}
\addcontentsline{toc}{chapter}{Executive Summary}

The report addresses a specific mathematical question. A stable one-bit sigma-delta quantizer produces coefficients in $\{-1,+1\}$, but the pointwise quantization error is not small. The useful structure is instead the identity $u-q=\Delta v$, where the state $v$ remains uniformly bounded. The main task is to determine how this discrete-difference structure affects polynomial Fourier extension sums.

The analysis produces six groups of rigorous results.

\begin{enumerate}[label=\textbf{Result Group \arabic*.},leftmargin=3.2cm]
\item \textbf{First-order stability and discrepancy.} The greedy sign quantizer is stable for inputs in $[-1,1]$. Every interval sum of the quantization error is controlled by the state bound, independently of the interval length.

\item \textbf{Weighted variation and parabolic extension.} Discrete summation by parts converts the shaped error into a boundary contribution and a weighted state sum. The resulting norm is the discrete total variation of the sampled phase. For the parabolic phase, the variation is governed by $J(x,t)=\int_0^1|x+2t\xi|\,\mathrm d\xi$, whose exact piecewise formula is derived.

\item \textbf{Sharpness and boundary structure.} The zero-input orbit gives an exact $1/N$ error at the origin for odd record lengths. This example identifies the terminal state as the leading first-order obstruction and proves the optimality of the uniform $N^{-1}$ scale under the stated assumptions.

\item \textbf{Higher-order finite-record theory.} For $e=\Delta^rv$, repeated summation by parts produces interior forward differences of the weight and a complete hierarchy of boundary traces. High-order decay follows when the relevant traces vanish or are sufficiently small. Under exact endpoint compatibility and smooth phase sampling, the normalised error is $O(N^{-r})$ on compact parameter sets.

\item \textbf{Boundary correction and fractional regularity.} The full boundary functional can be added to the quantized extension as finite-dimensional side information. The corrected reconstruction has an exact pure-interior error and therefore achieves the high-order rate without terminal reset. For weights in $C^{r-1,\alpha}$, the proved rate is $O(N^{-(r-1+\alpha)})$.

\item \textbf{Extensions and verification.} The same mechanism applies to polynomial phases, moment curves, multidimensional divergence-form shaping, growing observation regions, residual errors, leaky shaping, block processing, and smooth termination. Orthogonality, exponential-sum theory, restriction estimates, and decoupling are examined as tools for scales where total variation no longer provides decay. Numerical examples verify the exact first-order sharpness and the second-order endpoint-compatible scaling.
\end{enumerate}

The central contribution is a quantitative approximation theory for unquantized and one-bit polynomial Fourier extensions under explicit stability and boundary hypotheses. At critical scales, the framework connects the extension error to realisable state trajectories, oscillatory cancellation, and terminal control.

\tableofcontents
\clearpage
\listoffigures
\clearpage
\listoftables
\clearpage
\chapter*{List of Abbreviations}
\addcontentsline{toc}{chapter}{List of Abbreviations}

\begin{longtable}{>{\bfseries}p{2.8cm}p{10.2cm}}
ADC & analogue-to-digital converter\\
AC & absolutely continuous\\
BV & bounded variation\\
BCR & boundary-corrected reconstruction\\
CS & compressed sensing\\
DFR & discrete Fourier restriction\\
DFT & discrete Fourier transform\\
DSP & digital signal processing\\
ENOB & effective number of bits\\
FFT & fast Fourier transform\\
LHS & left-hand side\\
MSE & mean squared error\\
NTF & noise transfer function\\
OSR & oversampling ratio\\
PCM & pulse-code modulation\\
PDE & partial differential equation\\
RHS & right-hand side\\
RIP & restricted isometry property\\
SNR & signal-to-noise ratio\\
SNDR & signal-to-noise-and-distortion ratio\\
STF & signal transfer function\\
VMVT & Vinogradov mean value theorem\\
\end{longtable}

\chapter*{Notation and Reading Guide}
\addcontentsline{toc}{chapter}{Notation and Reading Guide}

\begin{longtable}{p{3.4cm}p{9.6cm}}
$N$ & number of samples or coefficients\\
$u=(u_k)$ & unquantized input sequence\\
$q=(q_k)$ & quantized output sequence, usually one-bit\\
$e=u-q$ & quantization error sequence\\
$v$ & internal state sequence\\
$\D v_k=v_k-v_{k-1}$ & backward finite difference\\
$\Dp w_k=w_{k+1}-w_k$ & forward finite difference\\
$\E_N$ & normalized discrete extension operator\\
$\TV_N(w)$ & discrete total variation of a weight sequence\\
$\Var_{[0,1]}(\phi)$ & total variation of a scalar phase\\
$A\lesssim B$ & $A\le CB$ for a constant independent of the main asymptotic parameter\\
$\T=\R/\Z$ & one-dimensional torus\\
$e(z)$ & shorthand for $\exp(2\pi\ii z)$ when used in number-theoretic sections\\
\end{longtable}

The core first-order argument is contained in Chapters 3--5. The exact higher-order finite-record theory appears in Chapter 6. Polynomial, multidimensional, and growing-region extensions are collected in Chapter 7. External harmonic-analysis tools and conditional transfer statements are separated in Chapters 2 and 8.

\clearpage
\pagenumbering{arabic}
\chapter{Introduction and Research Objectives}
\label{chap:introduction}

\section{Research problem and motivation}
\label{sec:01_introduction}

\subsection{The mathematical problem}

One-bit sigma-delta quantization converts an analogue or finely represented sequence into a binary sequence while feeding the accumulated error back into the quantizer. The output symbols are coarse, but the error is organised. This distinction is the reason sigma-delta methods can outperform memoryless scalar quantization under oversampling. The engineering origins of delta-sigma modulation go back to early feedback coding systems, while the spectral and probabilistic analysis of quantization noise developed through the work of Bennett, Gray, Candy, and many others \cite{Bennett1948,Inose1962,Candy1985,GrayChouWong1989,Gray1990}. A rigorous modern mathematical theory was established through stable high-order constructions, invariant-region arguments, exponential accuracy, and robustness analysis \cite{GunturkLagariasVaishampayan2001,DaubechiesDeVore2003,Gunturk2003CPAM,Gunturk2004JAMS,DaubechiesDeVoreGunturkVaishampayan2006,DeiftGunturkKrahmer2011,Gunturk2012}.

The present manuscript studies a specific interface between noise shaping and harmonic analysis. Let
\[
 \E_{N,d}a(x)
 =\frac1N\sum_{k=1}^N a_k
 e\!\left(\sum_{j=1}^d x_j(k/N)^j\right),
 \qquad e(z)=\exp(2\pi\ii z).
\]
For $d=2$, this is a discrete parabolic extension sum. Such sums are closely related to periodic Schr\"odinger evolution, discrete restriction estimates, Weyl sums, and moment-curve decoupling \cite{Bourgain1993I,HuLi2011,BourgainDemeter2015,BourgainDemeterGuth2016,Demeter2020}. The question is whether replacing $u_k$ by a one-bit sequence $q_k$ can preserve the extension sum in a quantitatively controlled manner.

A direct estimate gives
\[
 \abs{\E_{N,d}u(x)-\E_{N,d}q(x)}
 \le \frac1N\sum_{k=1}^N\abs{u_k-q_k}.
\]
This bound is usually of order one. It does not use noise shaping and therefore cannot reveal oversampling gain. The central observation is that a stable first-order sigma-delta scheme produces
\[
 e_k=u_k-q_k=\D v_k,
 \qquad \norm{v}_{\ell^\infty}\le V.
\]
Discrete summation by parts then transfers $\D$ from $v$ to the oscillatory weight. If the phase changes slowly across adjacent indices, the resulting weight difference is of order $N^{-1}$. This produces a normalized error of order $N^{-1}$ on fixed parameter sets.

\subsection{Relation to discrete restriction theory}

Three related analytical quantities organise the discussion. The approximation problem concerns
\[
 \E_Nu-\E_Nq.
\]
The extension-size problem concerns
\[
 \norm{\E_Nq}_{L^p(\Omega)}.
\]
Discrete restriction theory studies scale-sensitive bounds that are uniform over coefficient sequences in an appropriate norm. The first-order argument establishes a quantitative approximation theorem, while the later oscillatory sections describe the additional state information relevant to extension-size and restriction estimates.

The distinction is structural. The total-variation proof takes absolute values of adjacent phase increments. Once this is done, true oscillatory cancellation has been discarded. On a fixed compact region this loss is acceptable because each increment is already small. On a region whose diameter is comparable with $N$, the adjacent phase increment is no longer small. The estimate then becomes order one. Critical-scale results require methods designed to preserve and exploit oscillation, including level-set arguments, bilinear restriction, multilinear Kakeya estimates, polynomial partitioning, decoupling, and Vinogradov mean value estimates \cite{TaoVargasVega1998,BennettCarberyTao2006,BourgainGuth2011,Guth2016,Guth2018,BourgainDemeter2015,BourgainDemeterGuth2016}.

\subsection{Relation to finite frames and compressed sensing}

Noise shaping has already been extended far beyond classical bandlimited sampling. Finite-frame quantization showed that the ordering and variation of frame vectors influence reconstruction accuracy \cite{BenedettoPowellYilmaz2006,BenedettoPowellYilmazSecond2006,BodmannPaulsen2007}. Higher-order accuracy requires suitable dual frames or smooth endpoint termination \cite{BodmannPaulsenAbdulbaki2007,BlumEtAl2010}. Random frames and compressed sensing provide another major branch, in which Sobolev duals and convex recovery methods convert high-order noise shaping into polynomial or root-exponential reconstruction accuracy \cite{GunturkPowellSaabYilmaz2013,KrahmerSaabYilmaz2014,SaabWangYilmaz2018,FengKrahmerSaab2017}. These results are conceptually close to the present work because they also transfer a finite-difference operator onto a reconstruction object.

There is, however, an important difference. In finite-frame reconstruction, one may choose a dual frame adapted to the difference operator. In a Fourier extension problem, the test weights are imposed by the phase. This makes endpoint terms and phase increments unavoidable. The manuscript therefore places greater emphasis on exact finite-interval identities.

\subsection{Contributions}

The rigorous contributions are organised as follows.

\begin{enumerate}[label=(\roman*)]
\item A complete invariant-interval proof is given for the greedy first-order one-bit recursion, together with interval discrepancy bounds.
\item A weighted summation theorem is proved for arbitrary complex weights, with both initial and terminal boundary terms retained.
\item The phase estimate is extended from continuously differentiable phases to absolutely continuous and bounded-variation phases by means of standard BV theory \cite{AmbrosioFuscoPallara2000,EvansGariepy2015}.
\item A sharp variation factor for the parabolic phase is derived and evaluated explicitly.
\item Uniform optimality of the order $N^{-1}$ is proved in the absence of additional endpoint conditions.
\item An exact $r$th-order finite-interval summation formula is established. The formula displays every boundary term and prevents an invalid automatic inference of $O(N^{-r})$ decay.
\item Under explicit endpoint compatibility, an $O(N^{-r})$ error estimate is proved for smooth phases.
\item A boundary-trace corrected reconstruction formula is established. Exact trace metadata removes every finite-record boundary term and recovers the $O(N^{-r})$ interior rate without modifying the original coefficient record.
\item Fractional regularity is treated through finite differences: if the sampled weight belongs to $C^{r-1,\alpha}$, the corrected or endpoint-compatible rate is $O(N^{-(r-1+\alpha)})$.
\item The first-order theory is extended to polynomial moment curves and multidimensional divergence-form noise shaping.
\item The subcritical growing-region regime is identified, and a precise scale-transition estimate is proved.
\item Orthogonality obstructions are derived, showing why global $L^2$ decay cannot follow from bounded-state noise shaping alone.
\end{enumerate}

\subsection{Position of the results}

The higher-order finite-interval identity, the endpoint-compatible phase theorem, the boundary-trace corrected reconstruction, the polynomial-phase extension, and the divergence-form grid estimate follow from exact finite-difference identities developed in the report. Chapter~\ref{chap:lp-oscillation} adds transfer theorems under explicit state-sum, correlation, covariance, or block-square hypotheses. The final chapter consolidates the derivations, principal estimates, and their relation to discrete Fourier restriction and periodic Schrödinger analysis.

\subsection{Organisation}

Chapter~\ref{chap:literature} positions the problem within sigma-delta quantization, finite-frame theory, approximation theory, and harmonic analysis. Chapter~\ref{chap:methodology} fixes the analytical conventions. Chapters~\ref{chap:first-order} and~\ref{chap:parabolic} contain the complete first-order and parabolic proofs. Chapter~\ref{chap:higher-order} develops finite-record higher-order identities, endpoint-compatible estimates, and boundary-trace corrected reconstruction. Chapters~\ref{chap:extensions}--\ref{chap:robustness} treat polynomial, multidimensional, oscillatory, and implementation-sensitive extensions. Chapter~\ref{chap:consolidated-results} consolidates the detailed derivations and principal results. The appendices contain expanded identities, derivative estimates, BV background, a reference map, reproducible code, and a glossary.

\section{Quantization context and stable noise shaping}
\label{sec:02_quantization_background}

\subsection{Memoryless quantization and feedback quantization}

A memoryless scalar quantizer maps each input sample to a nearby alphabet point independently of previous samples. Pulse-code modulation is the standard example. If the quantization step is $\delta$, one often begins with a local error estimate of size $\delta/2$. Such an estimate is useful but does not improve when the same signal is oversampled unless the reconstruction method introduces further averaging. Classical models of quantization noise and their limitations are discussed in \cite{Bennett1948,Gray1990,GoyalVetterliThao1998}.

Sigma-delta quantization changes the problem by introducing memory. The current quantization decision depends on an internal state that stores previous error. In the simplest recursion,
\[
 q_k=Q(u_k+v_{k-1}),\qquad
 v_k=v_{k-1}+u_k-q_k.
\]
Therefore,
\[
 u_k-q_k=v_k-v_{k-1}.
\]
The error is a discrete derivative. In the frequency domain, a derivative corresponds to multiplication by $1-e^{-2\pi\ii\omega}$, which is small near zero frequency. This is the elementary noise-shaping mechanism.

\subsection{The engineering transfer-function picture}

In a linearised sigma-delta model, one writes the quantizer output as the quantizer input plus an additive noise source. The output then separates into a signal transfer function and a noise transfer function. A first-order loop gives an idealised noise transfer factor $1-z^{-1}$. Higher-order designs aim for $(1-z^{-1})^r$ or a stable approximation to it. This circuit-level language is useful, but it does not by itself prove nonlinear stability. The rigorous mathematical theory replaces the additive-noise assumption with an exact state recursion and an invariant-set argument \cite{GunturkLagariasVaishampayan2001,DaubechiesDeVore2003,Gunturk2012}.

The difference between the two viewpoints should be kept explicit. Transfer-function calculations describe the desired spectral shape. Stability analysis proves that the nonlinear state remains bounded for all admissible inputs. Without boundedness, the formal identity $e=\D^r v$ is not enough to control reconstruction error.

\subsection{First-order and higher-order state equations}

For a first-order scheme, the state equation is
\[
 \D v=e.
\]
For an $r$th-order scheme, one often writes
\[
 \D^r v=e.
\]
This notation compresses a larger state-space system. In implementations, the modulator may contain several integrator states, and the relationship between those states and the scalar sequence $v$ depends on the architecture. The analysis in this manuscript begins only after a valid identity of the form $e=\D^r v$ has been established.

Stable high-order one-bit schemes are nontrivial. The existence of stable families of arbitrary order and the possibility of exponential reconstruction accuracy were major achievements \cite{DaubechiesDeVore2003,Gunturk2003CPAM,Gunturk2004JAMS,DeiftGunturkKrahmer2011}. The present manuscript does not reproduce those constructions. Instead, it assumes a bounded state when treating general order and focuses on finite-record Fourier testing.

\subsection{Oversampling and normalisation}

The extension operator is normalised by $1/N$. This choice corresponds to a Riemann-sum scale rather than an unnormalised Weyl sum. The distinction matters. If
\[
 S_N(x,t)=\sum_{k=1}^N a_k e\!\left(x\frac{k}{N}+t\frac{k^2}{N^2}\right),
\]
then $S_N=N\E_N$. An $O(N^{-1})$ bound for $\E_Nu-\E_Nq$ is equivalent to an $O(1)$ bound for the unnormalised difference. Neither statement should be confused with square-root cancellation in a classical unnormalised exponential sum.

Oversampling appears here through the small mesh size $h=1/N$. A slowly varying phase satisfies
\[
 w_{k+1}-w_k=O(h).
\]
The state difference is transferred to this increment. This is the exact analogue of moving a derivative onto a smooth reconstruction kernel in bandlimited sigma-delta analysis \cite{DaubechiesDeVore2003,DaubechiesSaab2015}.

\subsection{Finite frames and the role of endpoint geometry}

Finite-frame quantization provides a useful comparison. Suppose $x$ is represented by frame coefficients $\inner{x}{f_k}$ and these coefficients are quantized. Summation by parts transfers differences to the frame path $f_k$. Error bounds then depend on frame variation. Higher-order accuracy requires higher finite differences of the frame vectors to be small and, crucially, requires endpoint terms to vanish or be controlled \cite{BenedettoPowellYilmaz2006,BenedettoPowellYilmazSecond2006,BodmannPaulsen2007,BodmannPaulsenAbdulbaki2007}.

The same phenomenon occurs in the Fourier extension problem. A high-order difference identity produces interior terms involving $\Dp^r w$, but it also produces lower-order boundary terms. If those terms are ignored, the conclusion is generally false. This observation is one of the main organising principles of the manuscript.

\subsection{Robustness and imperfect quantizers}

Real quantizers have threshold errors, saturation limits, circuit noise, and finite settling. The mathematical literature includes robustness results for imperfect quantizers and decimation procedures \cite{DaubechiesDeVoreGunturkVaishampayan2006,DaubechiesSaab2015}. In the present setting, a convenient perturbed model is
\[
 u_k-q_k=\D v_k+\eta_k,
\]
where $\eta$ collects unshaped error. The shaped part receives a variation gain, while the residual part contributes through its average $\ell^1$ size or another norm appropriate to the application. The robustness decomposition is developed in Chapter~\ref{chap:robustness}.

\subsection{What stability must mean in this manuscript}

The phrase ``stable quantizer'' will always mean that the relevant state sequence satisfies a uniform bound independent of the record length $N$. For first order,
\[
 \sup_{N\ge1}\sup_{1\le k\le N}\abs{v_k}\le V.
\]
For higher order, the exact state variable included in $e=\D^r v$ must obey the same type of bound. A simulation showing bounded states for one input is evidence, not a stability proof. Conversely, a mathematical bound may be conservative relative to circuit behaviour. These two statements are compatible.

\chapter{Literature Review and Theoretical Positioning}
\label{chap:literature}

This chapter places the report within four established areas: mathematical sigma-delta quantization, frame and compressed-sensing quantization, exponential-sum theory, and Fourier restriction with decoupling. The purpose is not to reproduce the literature exhaustively. The purpose is to identify the exact theorem used at each stage and to distinguish those external theorems from the finite-record results derived later.

\section{Exponential sums, Weyl differencing, and mean values}
\label{sec:16_exponential_sums}

\subsection{Purpose of this literature chapter}

The total-variation method takes absolute values of adjacent phase differences. It is therefore effective when those differences are uniformly small, but it cannot capture cancellation among terms whose phases rotate substantially. Classical exponential-sum theory was developed to quantify precisely this cancellation. The main tools include Weyl differencing, van der Corput estimates, the Hardy--Littlewood circle method, Vinogradov mean values, efficient congruencing, and decoupling \cite{Weyl1916,Vinogradov1954,Hua1965,GrahamKolesnik1991,Vaughan1997,Montgomery1994,IwaniecKowalski2004,Wooley2012,BourgainDemeterGuth2016}.

This chapter does not reproduce the deep proofs of these theorems. It states the forms relevant to structured one-bit coefficients, explains their normalisations, and identifies which parts of the classical machinery do not immediately apply when the coefficient sequence is $\D^rv$.

\subsection{Finite differences, moduli of smoothness, and fractional regularity}

The high-order interior term is naturally measured by a modulus of smoothness rather than by classical derivatives alone. Constructive approximation theory treats finite differences as intrinsic regularity measurements. H\"older, Lipschitz, Besov, and Sobolev scales can be characterised by the rate at which suitable finite differences vanish \cite{DeVoreLorentz1993,DitzianTotik1987,Triebel1983}. This viewpoint is especially appropriate for a sampled Fourier weight because the finite-record summation identity produces $\Delta_+^rw$ exactly. A continuous derivative estimate is introduced only after the discrete identity has been established.

For $W\in C^{r-1,\alpha}$, the elementary estimate
\[
 |\Delta_h^rW(x)|\le h^{r-1+\alpha}[W^{(r-1)}]_{C^{0,\alpha}}
\]
provides an intermediate scale between integer orders $r-1$ and $r$. Chapter~\ref{chap:higher-order} proves this estimate directly and inserts it into the finite-record formula. The resulting theorem replaces a binary distinction between ``smooth'' and ``nonsmooth'' weights by a regularity-sensitive decay law.

\subsection{Basic notation}

For a real polynomial
\[
 P(k)=\alpha_dk^d+\cdots+\alpha_1k+\alpha_0,
\]
define the weighted exponential sum
\begin{equation}
\label{eq:weighted-weyl-sum}
 S_P(a;N)=\sum_{k=1}^{N}a_ke(P(k)).
\end{equation}
The unweighted case has $a_k=1$. The quantization-error sum considered here is weighted, with $a_k=e_k$ or $a_k=v_k$ after summation by parts.

The trivial estimate is
\[
 \abs{S_P(a;N)}\le\norm{a}_{\ell^1}.
\]
If $\abs{a_k}\le1$, this is $N$. Any bound $o(N)$ represents cancellation.

\subsection{Abel summation and bounded partial sums}

Before using curvature, one can exploit bounded partial sums of the coefficients. Let
\[
 A(m)=\sum_{k=1}^{m}a_k.
\]
Abel summation gives
\begin{equation}
\label{eq:abel-general}
 \sum_{k=1}^{N}a_kb_k
 =A(N)b_N-\sum_{k=1}^{N-1}A(k)(b_{k+1}-b_k).
\end{equation}
This is the first-order weighted theorem with $A=v$ when $a=e=\D v$ and $v_0=0$.

If $\sup_m\abs{A(m)}\le V$, then
\[
 \abs{S_P(a;N)}
 \le V\left(1+\sum_{k=1}^{N-1}\abs{e(P(k+1))-e(P(k))}\right).
\]
The right-hand side is useful only while the phase increments are small or arithmetically resonant. Classical Weyl estimates instead exploit cancellation inside the sum without taking all these absolute values.

\subsection{Van der Corput differencing}

A fundamental identity compares a sum with shifted correlations. Let
\[
 S=\sum_{k=1}^{N}a_k.
\]
For $1\le H\le N$, a standard van der Corput inequality has the schematic form
\begin{equation}
\label{eq:vdc-schematic}
 \abs{S}^2
 \lesssim\frac{N+H}{H}
 \left[
  \sum_{k=1}^{N}\abs{a_k}^2
  +2\sum_{h=1}^{H}\left(1-\frac{h}{H}\right)
   \abs{\sum_{k=1}^{N-h}a_{k+h}\overline{a_k}}
 \right].
\end{equation}
Precise variants and normalisations are given in standard texts \cite{GrahamKolesnik1991,Montgomery1994,IwaniecKowalski2004}.

For $a_k=e(P(k))$, the correlation phase is
\[
 P(k+h)-P(k),
\]
a polynomial of degree $d-1$. Repeated differencing reduces the degree. This is the heart of Weyl's method.

For shaped coefficients, take
\[
 a_k=(\D^rv)_ke(P(k)).
\]
The correlation contains products
\[
 (\D^rv)_{k+h}\overline{(\D^rv)_k}
 e(P(k+h)-P(k)).
\]
Boundedness of $v$ does not directly control the partial sums of these products. Thus van der Corput differencing preserves phase cancellation but complicates the coefficient structure. A new theory must track both effects simultaneously.

\subsection{Weyl's inequality}

For an unweighted polynomial sum, Weyl's inequality bounds the sum in terms of a rational approximation to the leading coefficient. One standard form states that if
\[
 \left|\alpha_d-\frac{a}{q}\right|\le q^{-2},
 \qquad (a,q)=1,
\]
then, for every $\varepsilon>0$,
\begin{equation}
\label{eq:weyl-inequality-schematic}
 \abs{\sum_{k=1}^{N}e(P(k))}
 \le C_{d,\varepsilon}N^{1+\varepsilon}
 \left(q^{-1}+N^{-1}+qN^{-d}\right)^{2^{1-d}}.
\end{equation}
The exact exponent and refinements depend on the version used. The classical source is Weyl's equidistribution paper, and modern treatments appear in Vaughan and Iwaniec--Kowalski \cite{Weyl1916,Vaughan1997,IwaniecKowalski2004}.

The estimate divides parameter space into major arcs, where a coefficient has a good rational approximation with small denominator, and minor arcs, where Weyl cancellation is strong. In the parabolic extension, the relevant leading coefficient is $t/N^2$ in the unscaled integer polynomial. A critical-scale analysis must therefore respect the arithmetic of $t/N^2$.

\subsection{The Vinogradov mean value}

For degree $d$, define
\[
 f_d(\alpha;N)=\sum_{k=1}^{N}e(\alpha_1k+\cdots+\alpha_dk^d)
\]
and
\begin{equation}
\label{eq:vmvt-definition}
 J_{s,d}(N)=\int_{\T^d}\abs{f_d(\alpha;N)}^{2s}\dd\alpha.
\end{equation}
By orthogonality, $J_{s,d}(N)$ counts integer solutions of
\begin{equation}
\label{eq:vinogradov-system}
 x_1^j+\cdots+x_s^j
 =y_1^j+\cdots+y_s^j,
 \qquad 1\le j\le d,
\end{equation}
with all variables in $\{1,\ldots,N\}$.

The main conjecture, now a theorem, states that for every $\varepsilon>0$,
\begin{equation}
\label{eq:vmvt-main}
 J_{s,d}(N)
 \le C_{s,d,\varepsilon}
 \left(N^{s+\varepsilon}
 +N^{2s-d(d+1)/2+\varepsilon}\right).
\end{equation}
Efficient congruencing produced major advances and resolved the cubic case; decoupling resolved degrees above three, with later work completing and simplifying several aspects \cite{Wooley2012,FordWooley2014,Wooley2016,Wooley2017,BourgainDemeterGuth2016,Wooley2019,GuoLiYungZorinKranich2020}.

The cited sources provide the proofs of \eqref{eq:vmvt-main} and trace the historical development of efficient congruencing and decoupling methods.

\subsection{Weighted mean values}

Decoupling yields a weighted version at the critical exponent $p=d(d+1)$. For a complex sequence $a=(a_k)$,
\begin{equation}
\label{eq:weighted-critical-restriction}
 \norm{\sum_{k=1}^{N}a_ke(\alpha_1k+\cdots+\alpha_dk^d)}_{L^{d(d+1)}(\T^d)}
 \le C_{d,\varepsilon}N^\varepsilon\norm{a}_{\ell^2}.
\end{equation}
This form is particularly relevant because the coefficients need not be equal. For a bounded shaped error, however,
\[
 \norm{e}_{\ell^2}\lesssim N^{1/2},
\]
so the theorem does not by itself yield a gain over the natural unit-coefficient scale.

To improve the estimate using $e=\D^rv$, one needs a weighted theorem whose right-hand side measures a negative discrete Sobolev norm, bounded partial sums, endpoint-compatible moments, or another norm adapted to the difference structure.

\subsection{Major arcs, minor arcs, and shaped coefficients}

The circle method decomposes the torus into regions near rational points and their complement. On minor arcs, one seeks strong cancellation in polynomial sums. On major arcs, one approximates the sum by a product of a complete exponential sum and a continuous oscillatory integral.

For shaped coefficients, each region presents a different issue.

\begin{enumerate}[label=(\roman*)]
\item On minor arcs, summation by parts may combine bounded partial sums with a Weyl estimate for short phase increments.
\item On major arcs, the binary or difference structure changes the complete sum and may create additional cancellation or new resonances.
\item Boundary traces contribute low-frequency terms that are naturally concentrated near the origin, which is a principal major arc.
\end{enumerate}

A useful research strategy is therefore to remove boundary traces first, then study minor and major arcs separately.

\subsection{Completion and short intervals}

Classical estimates often treat a complete interval $1\le k\le N$. Blockwise reset and localisation lead to shorter intervals. Completion methods embed an incomplete sum into a complete periodic sum at the cost of logarithmic or arithmetic factors. For coefficients with bounded partial sums, completion must be performed carefully because extending the interval can create new boundary terms.

If a state is compactly supported away from the endpoints, then the shaped error also has a short transition region. Such states may be suitable test cases for combining finite-difference cancellation with classical incomplete-sum estimates.

\subsection{A possible hybrid inequality}

Suppose $e=\D v$, $v_0=v_N=0$. Then
\[
 S_P(e;N)=-\sum_{k=1}^{N-1}v_k
 e(P(k))\left[e(P(k+1)-P(k))-1\right].
\]
Define
\[
 m_k=e(P(k+1)-P(k))-1.
\]
The sum is now an exponential sum with bounded coefficients $v_km_k$. On subcritical regions, $m_k$ is uniformly small and total variation is sufficient. On critical regions, $m_k$ is not small but has its own polynomial phase dependence. One may attempt to apply Weyl or mean-value estimates to the product $m_ke(P(k))$, which is a difference of two polynomial phases:
\[
 m_ke(P(k))=e(P(k+1))-e(P(k)).
\]
This identity returns to the original difference, but it suggests a multi-phase decomposition rather than an absolute-value bound.

For order $r$, repeated summation produces a finite linear combination of shifted phases
\[
 \Dp^re(P(k))
 =\sum_{j=0}^{r}(-1)^{r-j}\binom{r}{j}e(P(k+j)).
\]
Each term is a classical polynomial exponential sum with a shifted polynomial. Endpoint-compatible shaping therefore converts the problem into a controlled linear combination of classical sums with bounded state coefficients. The remaining difficulty is that the coefficients $v_k$ are arbitrary bounded sequences, and classical Weyl cancellation can be destroyed by adversarial weights.

\subsection{What additional state hypotheses would help}

Several stronger assumptions are compatible with practical quantization models and could make exponential-sum estimates possible.

\begin{enumerate}[label=(\alph*)]
\item $v$ has bounded variation or a bounded number of jumps.
\item $v$ is periodic with a period much shorter than $N$.
\item $v$ is generated by an ergodic map with quantitative mixing.
\item $v$ has a sparse or high-frequency discrete Fourier spectrum.
\item $v$ is random after dither and satisfies concentration estimates.
\end{enumerate}

Ergodic aspects of sigma-delta state dynamics have been studied for specific recursions \cite{GunturkThao2005}. A future theorem could combine such dynamical information with number-theoretic phase cancellation.

\subsection{Summary of the number-theoretic route}

The key lesson is not that a particular Weyl inequality immediately solves the structured extension problem. It is that critical-scale cancellation is controlled by rational approximation, differencing, and mean values. The difference structure of sigma-delta error must be incorporated into these mechanisms rather than estimated before they are used.

A realistic first target is a minor-arc estimate for endpoint-compatible first- or second-order shaped errors under an additional regularity condition on the state. A second target is a mean-value inequality in which $\norm{a}_{\ell^2}$ is replaced or supplemented by bounded partial sums or vanishing discrete moments.

\section{Restriction, Strichartz estimates, and decoupling}
\label{sec:17_restriction_decoupling}

\subsection{Continuous restriction as the geometric model}

Fourier restriction asks whether the Fourier transform of an $L^p$ function can be meaningfully restricted to a curved set of measure zero. The dual extension operator for a parametrised surface $\Sigma$ has the form
\[
 Ef(x)=\int_\Sigma f(\xi)e^{2\pi\ii x\cdot\xi}\dd\sigma(\xi).
\]
Curvature creates oscillation and hence integrability beyond the trivial $L^1\to L^\infty$ estimate. Foundational results include the restriction work of Fefferman, Tomas, and Strichartz \cite{Fefferman1970,Tomas1975,Strichartz1977}. Modern developments use bilinear and multilinear restriction, Kakeya geometry, polynomial partitioning, and decoupling \cite{TaoVargasVega1998,Wolff2001,Tao2003,BennettCarberyTao2006,BourgainGuth2011,Guth2016,Guth2018}.

The discrete operator in this report is a Riemann-sum analogue of extension from the moment curve. The analogy is valuable, but the normalisation and parameter scales must be tracked carefully.

\subsection{The Stein--Tomas mechanism}

For a compact hypersurface with nonvanishing curvature, the Stein--Tomas theorem gives an $L^2$ extension estimate in a nontrivial range. Its proof uses the $TT^*$ method and decay of the Fourier transform of surface measure. The original sources are Tomas and the subsequent restriction theory developed by Stein and Strichartz \cite{Tomas1975,Strichartz1977,Stein1993}.

The important structural point is that an $L^2$ coefficient norm is converted into an ambient $L^p$ norm through curvature. Standard discrete restriction inherits the same coefficient norm. Since a one-bit error has $\ell^2$ size of order $N^{1/2}$, a direct application does not distinguish shaped and unshaped errors.

\subsection{Strichartz estimates}

For the Schr\"odinger equation
\[
 \ii\partial_tu+\Delta u=0,
 \qquad u(0)=f,
\]
the solution is an extension operator over the paraboloid. Strichartz estimates control mixed space-time norms of the solution. The endpoint theory was established in a general dispersive framework by Keel and Tao \cite{KeelTao1998}; standard expositions appear in Cazenave and Tao \cite{Cazenave2003,Tao2006}.

On a torus, the spectrum is discrete and Euclidean scaling is modified by arithmetic effects. Bourgain's work on periodic nonlinear evolution equations introduced decisive discrete restriction and periodic Strichartz estimates \cite{Bourgain1993I,Bourgain1993II}. Later work connected periodic estimates to decoupling and scale-invariant bounds \cite{BourgainDemeter2015,KillipVisan2014,Demeter2020}.

For the one-dimensional periodic Schr\"odinger flow,
\[
 u(x,t)=\sum_{k\in\Z}a_ke(kx+k^2t).
\]
The finite sum over $1\le k\le N$ is exactly the unscaled quadratic extension that appears after the change of variables in Chapter~\ref{chap:lp-oscillation}.

\subsection{Discrete restriction for the parabola}

A representative critical estimate is
\begin{equation}
\label{eq:parabola-discrete-restriction}
 \norm{\sum_{k=1}^{N}a_ke(kx+k^2t)}_{L^6(\T^2)}
 \le C_\varepsilon N^\varepsilon\norm{a}_{\ell^2}.
\end{equation}
For $a_k=1$, the sixth power counts solutions of
\begin{align*}
 k_1+k_2+k_3&=k_4+k_5+k_6,\\
 k_1^2+k_2^2+k_3^2&=k_4^2+k_5^2+k_6^2.
\end{align*}
The estimate is a quadratic Vinogradov mean value theorem. Discrete restriction associated with Schr\"odinger equations has also been studied directly by Hu and Li \cite{HuLi2011} and in a general formulation by Lai and Ding \cite{LaiDing2017}.

\subsection{Decoupling for the moment curve}

Let
\[
 E_Ig(x)=\int_Ig(\xi)e\!\left(x_1\xi+x_2\xi^2+\cdots+x_d\xi^d\right)\dd\xi.
\]
Partition $[0,1]$ into intervals $J$ of length $\delta$. The critical $\ell^2$ decoupling theorem for the moment curve states, schematically, that for
\[
 p=d(d+1)
\]
and for a ball $B$ of radius $\delta^{-d}$,
\begin{equation}
\label{eq:moment-decoupling}
 \norm{E_{[0,1]}g}_{L^p(w_B)}
 \le C_{d,\varepsilon}\delta^{-\varepsilon}
 \left(\sum_{J}\norm{E_Jg}_{L^p(w_B)}^2\right)^{1/2}.
\end{equation}
Bourgain and Demeter proved the general $\ell^2$ decoupling conjecture, and Bourgain, Demeter, and Guth obtained the sharp moment-curve consequence used in Vinogradov's mean value theorem \cite{BourgainDemeter2015,BourgainDemeterGuth2016}. A later proof by Guo, Li, Yung, and Zorin-Kranich gave a shorter route for the moment curve \cite{GuoLiYungZorinKranich2020}. Demeter's monograph provides a detailed account \cite{Demeter2020}.

The subsequent analysis uses the stated consequences of the decoupling theorem, with the full proof and broader formulation available in the cited sources.

\subsection{What decoupling separates}

Decoupling separates contributions from short frequency intervals in an $\ell^2$ manner. It is effective because pieces of a curved surface point in different directions. For a discrete sum, one may thicken each lattice frequency into a short interval and transfer the continuous estimate to the discrete setting.

Noise shaping creates a different type of structure. It correlates neighbouring coefficients through finite differences. This is local in the frequency index, whereas decoupling separates frequency intervals. The two structures are compatible in principle: finite differences act within or across adjacent intervals, while decoupling controls how the intervals recombine.

A naive approach loses the shaping. If one applies decoupling directly with coefficient sequence $e$, the right-hand side contains $\norm{e}_{\ell^2}$ or local $\ell^2$ norms. These can be as large for $e=\D v$ as for arbitrary bounded coefficients. The finite-difference representation must therefore be inserted before or during the decoupling decomposition.

\subsection{Boundary terms under frequency localisation}

Let $\psi_I(k)$ be a smooth cutoff to a frequency interval $I$. Then
\[
 \sum e_k\psi_I(k)w_k
 =\sum (\D v)_k\psi_I(k)w_k.
\]
Summation by parts moves the difference onto the product:
\[
 \Dp(\psi_Iw)_k
 =\psi_I(k+1)\Dp w_k+w_k\Dp\psi_I(k).
\]
The first term contains the small phase increment in a subcritical regime. The second is a localisation cost. If $I$ has length $M$, a smooth cutoff satisfies
\[
 \abs{\Dp\psi_I}\lesssim M^{-1}.
\]
Summed over $M$ indices, this cost is $O(1)$. After the global $1/N$ normalisation, each interval contributes $O(N^{-1})$ before recombination.

Decoupling may recombine $N/M$ intervals in $\ell^2$, producing a factor $(N/M)^{1/2}$ rather than $N/M$. This heuristic suggests a possible gain over global total variation, but a rigorous theorem must control the dependence of $\Dp w$ on the observation point and the boundary traces created at each interval.

\subsection{A model dyadic decomposition}

Partition $\{1,\ldots,N\}$ into intervals $I$ of length $M$. Write
\[
 \E_Ne=\sum_I\E_N(e\psi_I).
\]
After summation by parts inside $I$, each piece has three components:

\begin{enumerate}[label=(\roman*)]
\item an initial boundary term;
\item a terminal boundary term;
\item an interior term involving $v_k\Dp(\psi_Iw)_k$.
\end{enumerate}

If $v$ is globally bounded but not locally reset, the artificial boundaries at interval endpoints do not vanish. Smooth cutoffs spread these boundaries into the derivative term, but do not remove their total mass. A successful decoupling argument must show that the contributions of many interval boundaries combine orthogonally or cancel.

This observation parallels smooth frame-path termination. Hard partitioning creates large boundary errors; smooth localisation trades them for controlled derivatives.

\subsection{Multilinear restriction and transversality}

Bilinear and multilinear restriction estimates improve linear estimates when frequency pieces are transverse. Foundational results include the bilinear approach of Tao, Vargas, and Vega, Wolff's cone estimate, Tao's paraboloid estimate, and the multilinear restriction theorem of Bennett, Carbery, and Tao \cite{TaoVargasVega1998,Wolff2001,Tao2003,BennettCarberyTao2006}. Bourgain and Guth developed a broad-narrow method that converts multilinear information into linear estimates \cite{BourgainGuth2011}. Guth's polynomial partitioning method produced further advances \cite{Guth2016,Guth2018}.

For a one-dimensional moment curve, separated frequency intervals have transverse tangent directions after lifting to the appropriate ambient dimension. A shaped coefficient sequence couples neighbours, but intervals separated by more than the shaping order remain algebraically independent at the level of the finite difference. This may allow multilinear estimates to be applied after grouping adjacent intervals into packets.

\subsection{A negative Sobolev viewpoint}

The identity $e=\D^rv$ says that $e$ is small in a discrete negative Sobolev norm. Standard extension estimates control the operator from $\ell^2$ to $L^p$. A natural target is an estimate of the form
\begin{equation}
\label{eq:negative-sobolev-extension}
 \norm{\E_Ne}_{L^p(\Omega)}
 \le A_N\norm{D_N^{-r}e}_{\ell^\infty}.
\end{equation}
This estimate is intended for a class of errors satisfying compatible boundary conditions. In frame quantization, Sobolev duals are designed precisely to exploit the inverse difference operator \cite{BlumEtAl2010,GunturkPowellSaabYilmaz2013}. In curved extension, no fixed linear dual is available because the observation point varies continuously.

One possibility is to treat $\E_ND_N^r$ as a family of differentiated extension operators and estimate its operator norm from $\ell^\infty$ to $L^p$. The adjoint difference produces observation-dependent multipliers. Decoupling for such variable discrete multipliers is not an immediate consequence of the standard theorem.

\subsection{Periodic versus Euclidean regions}

Euclidean decoupling is stated on large balls whose radius is linked to the frequency scale. Periodic discrete restriction is stated on a torus. The variables $(x,t)$ used here employ a normalised phase, so the corresponding physical region is anisotropic. A correct transfer must include:

\begin{enumerate}[label=(\alph*)]
\item the change of variables from $(x,t)$ to torus variables;
\item the Jacobian of the change;
\item the $1/N$ coefficient normalisation;
\item the coefficient norm after finite differences;
\item the location and size of boundary terms.
\end{enumerate}

Many apparent gains disappear when one of these factors is omitted. Chapter~\ref{chap:lp-oscillation} formulates transfer principles with all normalisations displayed.

\subsection{The realistic role of decoupling in this project}

Decoupling is not needed for the compact-set $O(N^{-1})$ theorem. It becomes relevant only when the observation region grows to a scale on which adjacent phase variation is not small. At that point, the aim is not to reproduce a standard restriction estimate. The aim is to prove that the special coefficient class generated by stable sigma-delta quantization obeys a stronger estimate than arbitrary coefficients of the same $\ell^2$ size.

A plausible first result would be localised and conditional. For example, one might assume endpoint compatibility and a bounded-variation state, then prove an improved $L^6$ estimate on a subregion of the full parabolic torus. Establishing such a theorem would already create a genuine bridge between noise shaping and restriction theory.

\chapter{Analytical Methodology and Discrete Mathematical Framework}
\label{chap:methodology}

The analytical method is based on exact finite-record identities. Infinite-sequence notation is avoided whenever it would conceal a boundary contribution. Backward differences describe the quantization error, forward differences act on sampled weights, and total variation measures the cost of transferring a difference from the state to the weight.

\section{Discrete calculus, variation, and finite summation}
\label{sec:03_discrete_calculus}

\subsection{Backward and forward differences}

For a sequence $a=(a_k)$, define the backward difference
\[
 \D a_k=a_k-a_{k-1}
\]
and the forward difference
\[
 \Dp a_k=a_{k+1}-a_k.
\]
The two operators differ only by an index shift, but keeping both notations is helpful on finite intervals. The quantization identity naturally uses the backward difference, whereas summation by parts transfers it to a forward difference of the weight.

Repeated differences are defined recursively. For example,
\[
 \D^2a_k=a_k-2a_{k-1}+a_{k-2},
\qquad
 \Dp^2a_k=a_{k+2}-2a_{k+1}+a_k.
\]
The binomial formula gives
\[
 \D^ra_k=\sum_{j=0}^{r}(-1)^j\binom{r}{j}a_{k-j},
 \qquad
 \Dp^ra_k=\sum_{j=0}^{r}(-1)^{r-j}\binom{r}{j}a_{k+j}.
\]
These identities are elementary, but they should be written explicitly because boundary indices become decisive in finite records.

\subsection{Discrete summation by parts}

The finite analogue of integration by parts is also called Abel summation. It is standard in Fourier analysis and analytic number theory \cite{Apostol1976,Zygmund2002,SteinShakarchi2003}. The form used below is adapted to noise shaping.

\begin{lemma}[Finite summation by parts]
\label{lem:sbp-first}
Let $v_0,\ldots,v_N\in\C$ and $w_1,\ldots,w_N\in\C$. Then
\begin{equation}
\label{eq:sbp-first}
 \sum_{k=1}^N(\D v_k)w_k
 =v_Nw_N-v_0w_1-\sum_{k=1}^{N-1}v_k\Dp w_k.
\end{equation}
\end{lemma}

\begin{proof}
Expand the left-hand side:
\[
 \sum_{k=1}^Nv_kw_k-\sum_{k=1}^Nv_{k-1}w_k.
\]
In the second sum, replace $k$ by $k+1$ in the interior terms. This yields
\[
 v_Nw_N-v_0w_1+\sum_{k=1}^{N-1}v_k(w_k-w_{k+1}),
\]
which is exactly \eqref{eq:sbp-first}.
\end{proof}

The identity has three parts. The terminal state produces $v_Nw_N$. The initial state produces $-v_0w_1$. The interior produces a weighted forward difference. Every later estimate is a consequence of deciding how these three parts are controlled.

\subsection{Discrete total variation}

For a finite complex sequence $w=(w_k)_{k=1}^N$, define
\[
 \TV_N(w)=\sum_{k=1}^{N-1}\abs{w_{k+1}-w_k}.
\]
This is the variation of the polygonal path through the points $w_1,\ldots,w_N$ in the complex plane. It is not the variation of the real and imaginary parts separately, although the two notions are comparable. If $\abs{w_k}=1$, then $\TV_N(w)$ measures the total chordal motion of the sampled points on the unit circle.

The elementary inequality
\[
 \abs{e^{\ii\alpha}-e^{\ii\beta}}
 =2\abs{\sin((\alpha-\beta)/2)}
 \le \abs{\alpha-\beta}
\]
converts phase variation into weight variation. The constant one in this inequality is sharp near $\alpha=\beta$.

\subsection{Continuous total variation}

A real-valued function $\phi$ on $[0,1]$ has bounded variation if
\[
 \Var_{[0,1]}(\phi)
 :=\sup_{0=t_0<\cdots<t_m=1}
 \sum_{j=0}^{m-1}\abs{\phi(t_{j+1})-\phi(t_j)}
 <\infty.
\]
The standard theory of BV functions includes decomposition into absolutely continuous, jump, and Cantor parts \cite{AmbrosioFuscoPallara2000,EvansGariepy2015}. The first-order Fourier estimate needs only the defining variation inequality. It does not require differentiability.

If $\phi$ is absolutely continuous, then
\[
 \Var_{[0,1]}(\phi)=\int_0^1\abs{\phi'(\xi)}\dd\xi.
\]
The equality is a standard theorem. In the parabolic case, $\phi$ is smooth, so no measure-theoretic subtlety arises. The BV formulation is included because it reveals the actual regularity threshold of the first-order argument.

\subsection{Sampling a BV phase}

Let $\xi_k=k/N$ and set
\[
 w_k=e^{2\pi\ii\phi(\xi_k)}.
\]
Then
\begin{align*}
 \TV_N(w)
 &\le 2\pi\sum_{k=1}^{N-1}
 \abs{\phi(\xi_{k+1})-\phi(\xi_k)}\\
 &\le 2\pi\Var_{[0,1]}(\phi).
\end{align*}
This inequality remains valid when $\phi$ has jumps, provided a representative has been chosen at the sample points. Since only finitely many values are used, changing a representative on a null set may change the discrete sum. For applications involving discontinuous phases, the representative must therefore be specified. In the smooth polynomial cases considered later, this issue does not occur.

\subsection{Riemann-sum scaling}

Suppose $f\in C^1([0,1])$. The mean-value theorem gives
\[
 \abs{f((k+1)/N)-f(k/N)}
 \le \frac1N\norm{f'}_{L^\infty}.
\]
Summing over $k$ gives a bound independent of $N$. Therefore, first-order noise shaping and a normalisation $1/N$ yield an $N^{-1}$ estimate.

For $r$th differences, Taylor expansion suggests
\[
 \abs{\Dp^rf(k/N)}\lesssim_r N^{-r}\norm{f^{(r)}}_{L^\infty}.
\]
There are $O(N)$ interior terms, so their sum is $O(N^{1-r})$. After the external factor $1/N$, the contribution is $O(N^{-r})$. This heuristic is correct, but only after the boundary terms from repeated summation by parts have been removed or controlled.

\subsection{A useful integral formula for finite differences}

\begin{lemma}[Integral representation]
\label{lem:finite-difference-integral}
Let $f\in C^r([0,1])$, let $h>0$, and assume $x+rh\le1$. Then
\[
 \Delta_h^rf(x)
 =\int_{[0,h]^r}f^{(r)}(x+s_1+\cdots+s_r)\dd s_1\cdots\dd s_r,
\]
where $\Delta_hf(x)=f(x+h)-f(x)$.
\end{lemma}

\begin{proof}
For $r=1$, this is the fundamental theorem of calculus. Assume the formula for $r$. Then
\begin{align*}
 \Delta_h^{r+1}f(x)
 &=\Delta_h^rf(x+h)-\Delta_h^rf(x)\\
 &=\int_{[0,h]^r}\left[f^{(r)}(x+h+s_1+\cdots+s_r)
 -f^{(r)}(x+s_1+\cdots+s_r)\right]\dd s\\
 &=\int_{[0,h]^{r+1}}f^{(r+1)}(x+s_1+\cdots+s_{r+1})\dd s,
\end{align*}
by another application of the fundamental theorem of calculus.
\end{proof}

As an immediate consequence,
\begin{equation}
\label{eq:finite-difference-sup}
 \abs{\Delta_h^rf(x)}\le h^r\norm{f^{(r)}}_{L^\infty}.
\end{equation}
This estimate will be used repeatedly in the higher-order theory.

\subsection{Why exact indexing matters}

It is tempting to write symbolically
\[
 \inner{\D^rv}{w}=(-1)^r\inner{v}{\Dp^rw}
\]
and ignore the endpoints. This identity is correct on a bi-infinite sequence when all terms decay sufficiently, or on a periodic sequence with compatible boundary conditions. It is not correct on a finite interval without correction terms. The boundary terms are not a technical nuisance. They determine the leading asymptotic order.

The frame-quantization literature reached the same conclusion from a geometric perspective. Smooth frame-path termination was introduced precisely to suppress the endpoint contributions that obstruct high-order convergence \cite{BodmannPaulsenAbdulbaki2007}. The exact finite formula is derived in Chapter~\ref{chap:higher-order}.

\chapter{First-Order One-Bit Noise-Shaping Theory}
\label{chap:first-order}

This chapter develops the complete first-order theory required by the later Fourier analysis. Stability is proved directly from an invariant interval. Discrepancy bounds follow by telescoping. The weighted variation theorem is then established for arbitrary complex weights, followed by absolutely continuous, bounded-variation, Hölder, vector-valued, and operator-valued consequences.

\section{Greedy one-bit stability and discrepancy}
\label{sec:04_first_order_stability}

\subsection{The greedy sign quantizer}

Define
\[
 Q(y)=\begin{cases}
 +1,&y\ge0,\\
 -1,&y<0.
 \end{cases}
\]
Given a real input sequence $u=(u_k)_{k=1}^N$, the greedy first-order recursion is
\begin{equation}
\label{eq:greedy-recursion}
 q_k=Q(u_k+v_{k-1}),
 \qquad
 v_k=v_{k-1}+u_k-q_k,
 \qquad
 v_0=0.
\end{equation}
The convention $Q(0)=1$ is fixed. Another convention at zero changes some special trajectories but does not affect the main variation theorem.

The error identity follows immediately:
\begin{equation}
\label{eq:first-order-error}
 e_k:=u_k-q_k=v_k-v_{k-1}=\D v_k.
\end{equation}

\subsection{Invariant interval}

\begin{theorem}[Stability of the greedy one-bit scheme]
\label{thm:greedy-stability}
Assume
\[
 \abs{u_k}\le1,
 \qquad k=1,\ldots,N.
\]
Then the recursion \eqref{eq:greedy-recursion} satisfies
\[
 \abs{v_k}\le1,
 \qquad k=0,1,\ldots,N.
\]
Thus the scheme is stable with $V=1$.
\end{theorem}

\begin{proof}
The interval $[-1,1]$ is shown to be forward invariant. The initial state $v_0=0$ lies in the interval. Suppose $v_{k-1}\in[-1,1]$ and define
\[
 y_k=u_k+v_{k-1}.
\]
Since both terms lie in $[-1,1]$, one has $y_k\in[-2,2]$.

If $y_k\ge0$, then $q_k=1$ and
\[
 v_k=y_k-1.
\]
The condition $0\le y_k\le2$ gives $-1\le v_k\le1$.

If $y_k<0$, then $q_k=-1$ and
\[
 v_k=y_k+1.
\]
The condition $-2\le y_k<0$ gives $-1\le v_k<1$.

Therefore $v_k\in[-1,1]$ in both cases. Induction completes the proof.
\end{proof}

This invariant-region argument is a basic example of nonlinear stability. More elaborate sigma-delta systems require higher-dimensional invariant sets or carefully designed feedback filters \cite{GunturkLagariasVaishampayan2001,DaubechiesDeVore2003,GunturkThao2005}.

\subsection{Strict input margins}

Suppose $\abs{u_k}\le1-\mu$ for some $\mu>0$. The same proof gives $v_k\in[-1,1]$. The strict margin does not automatically shrink the invariant interval for every trajectory because the state can approach the endpoints after a sequence of inputs. It does, however, provide overload margin for perturbed thresholds and finite analogue errors. A robustness theorem must state the perturbation model explicitly.

\subsection{Interval discrepancy}

\begin{proposition}[Uniform discrepancy on intervals]
\label{prop:interval-discrepancy}
Assume $e_k=\D v_k$ and $\norm{v}_{\ell^\infty}\le V$. Then, for every $1\le m\le n\le N$,
\[
 \abs{\sum_{k=m}^ne_k}
 =\abs{v_n-v_{m-1}}
 \le2V.
\]
If $m=1$ and $v_0=0$, then
\[
 \abs{\sum_{k=1}^ne_k}\le V.
\]
\end{proposition}

\begin{proof}
The sum telescopes:
\[
 \sum_{k=m}^ne_k
 =\sum_{k=m}^n(v_k-v_{k-1})
 =v_n-v_{m-1}.
\]
The bounds follow from the triangle inequality.
\end{proof}

The proposition says that the quantization error has uniformly bounded discrepancy on every interval. Individual errors may be large, but their cumulative sum cannot drift linearly with the interval length. This property is stronger than a pointwise bound and is the discrete cancellation used later.

\subsection{Prefix sums and equivalence of formulations}

Let
\[
 E_n=\sum_{k=1}^ne_k.
\]
When $v_0=0$, one has $E_n=v_n$. Therefore first-order noise shaping is equivalent to bounded prefix sums of the error. Conversely, if an error sequence has bounded prefix sums, defining $v_n=E_n$ gives $e_n=\D v_n$.

This equivalence is useful because discrepancy theory often begins from partial sums, while sigma-delta theory begins from a state recursion. The two languages describe the same first-order structure.

\subsection{The zero-input orbit}

The zero input is an instructive trajectory. Let $u_k=0$ for all $k$ and $v_0=0$. Then
\[
 q_1=1,\qquad v_1=-1,
\]
and
\[
 q_2=-1,\qquad v_2=0.
\]
The pattern repeats:
\[
 q_k=(-1)^{k+1},
 \qquad
 v_k=\begin{cases}-1,&k\text{ odd},\\0,&k\text{ even}.
 \end{cases}
\]
This sequence is stable, but its terminal state is nonzero for odd $N$. It provides the sharpness example in Chapter~\ref{chap:parabolic}.

\subsection{Complex inputs and larger alphabets}

The greedy sign rule is real. Complex Fourier coefficients may nevertheless be tested because the weights are complex while the coefficients remain real. A genuinely complex quantizer would require a planar alphabet and an invariant set in $\C$. The weighted variation theorem itself is indifferent to how the bounded state was generated. It applies to complex $v$ and complex $e$ as soon as $e=\D v$ and $\norm{v}_{\ell^\infty}\le V$ hold.

Similarly, a multibit alphabet changes the stability region but not the summation argument. The separation between quantizer design and Fourier testing is therefore useful: stability is established in the quantizer model, and the extension estimate uses only the resulting difference identity.

\section{Weighted variation estimates}
\label{sec:05_weighted_variation}

\subsection{The general theorem}

\begin{theorem}[Weighted variation with full boundary terms]
\label{thm:weighted-full}
Let $e_k=\D v_k$ for $k=1,\ldots,N$, and assume
\[
 \norm{v}_{\ell^\infty}\le V.
\]
For arbitrary complex weights $w_1,\ldots,w_N$,
\begin{equation}
\label{eq:weighted-full}
 \abs{\frac1N\sum_{k=1}^Ne_kw_k}
 \le\frac{V}{N}
 \left(\abs{w_N}+\abs{w_1}+\TV_N(w)\right).
\end{equation}
If $v_0=0$, the term $\abs{w_1}$ may be removed:
\begin{equation}
\label{eq:weighted-zero-initial}
 \abs{\frac1N\sum_{k=1}^Ne_kw_k}
 \le\frac{V}{N}
 \left(\abs{w_N}+\TV_N(w)\right).
\end{equation}
\end{theorem}

\begin{proof}
By \cref{lem:sbp-first},
\[
 \sum_{k=1}^Ne_kw_k
 =v_Nw_N-v_0w_1-\sum_{k=1}^{N-1}v_k\Dp w_k.
\]
Taking absolute values gives
\begin{align*}
 \abs{\sum_{k=1}^Ne_kw_k}
 &\le \abs{v_N}\abs{w_N}+\abs{v_0}\abs{w_1}
 +\sum_{k=1}^{N-1}\abs{v_k}\abs{\Dp w_k}\\
 &\le V\left(\abs{w_N}+\abs{w_1}+\TV_N(w)\right).
\end{align*}
Divide by $N$. If $v_0=0$, the initial term vanishes.
\end{proof}

The theorem is deterministic and contains no oscillatory estimate. It applies to every complex weight sequence. Its usefulness depends on finding weights with bounded total variation.

\subsection{Optimality of the variation form}

The right-hand side cannot be replaced by a bound involving only $\max_k\abs{w_k}$ under the same hypotheses. Indeed, a rapidly alternating weight may align with an alternating error sequence. The variation term measures this possible alignment.

Conversely, the theorem is not always sharp in its constant because the triangle inequality discards cancellation among the interior terms. The purpose of later oscillatory analysis is precisely to improve on this absolute-value step in regimes where $\TV_N(w)$ is large.

\subsection{Absolutely continuous phases}

For a real phase $\phi$ on $[0,1]$, define
\[
 \E_N^\phi a
 =\frac1N\sum_{k=1}^Na_ke^{2\pi\ii\phi(k/N)}.
\]

\begin{theorem}[Absolutely continuous phase estimate]
\label{thm:ac-phase}
Assume $\phi$ is absolutely continuous, $e_k=\D v_k$, $v_0=0$, and $\norm{v}_{\ell^\infty}\le V$. Then
\begin{equation}
\label{eq:ac-phase}
 \abs{\E_N^\phi u-\E_N^\phi q}
 \le\frac{V}{N}
 \left[1+2\pi\int_0^1\abs{\phi'(\xi)}\dd\xi\right].
\end{equation}
\end{theorem}

\begin{proof}
Set $w_k=e^{2\pi\ii\phi(k/N)}$. Since $\abs{w_N}=1$ and
\[
 \abs{e^{\ii\alpha}-e^{\ii\beta}}\le\abs{\alpha-\beta},
\]
one has
\begin{align*}
 \abs{w_{k+1}-w_k}
 &\le2\pi\abs{\phi((k+1)/N)-\phi(k/N)}\\
 &\le2\pi\int_{k/N}^{(k+1)/N}\abs{\phi'(\xi)}\dd\xi.
\end{align*}
Summation gives
\[
 \TV_N(w)\le2\pi\int_{1/N}^{1}\abs{\phi'(\xi)}\dd\xi
 \le2\pi\int_0^1\abs{\phi'(\xi)}\dd\xi.
\]
Apply \eqref{eq:weighted-zero-initial}.
\end{proof}

\subsection{Bounded-variation phases}

\begin{theorem}[BV phase estimate]
\label{thm:bv-phase}
Assume $\phi\in BV([0,1])$, $e_k=\D v_k$, $v_0=0$, and $\norm{v}_{\ell^\infty}\le V$. Then
\[
 \abs{\E_N^\phi u-\E_N^\phi q}
 \le\frac{V}{N}
 \left[1+2\pi\Var_{[0,1]}(\phi)\right].
\]
\end{theorem}

\begin{proof}
For each adjacent pair,
\[
 \abs{w_{k+1}-w_k}
 \le2\pi\abs{\phi((k+1)/N)-\phi(k/N)}.
\]
The sum of the phase increments over the grid is bounded by the total variation of $\phi$. Apply \cref{thm:weighted-full} with $v_0=0$.
\end{proof}

The BV theorem reveals that differentiability is not the essential assumption. The essential assumption is finite phase variation on the sampled interval.

\subsection{Lipschitz and Hölder corollaries}

If $\phi$ is Lipschitz with constant $L$, then
\[
 \Var_{[0,1]}(\phi)\le L,
\]
so
\[
 \abs{\E_N^\phi u-\E_N^\phi q}
 \le\frac{V}{N}(1+2\pi L).
\]
A merely H\"older phase with exponent $\gamma<1$ may have grid variation of order $N^{1-\gamma}$. The same method then gives only $N^{-\gamma}$ decay:
\[
 \TV_N(w)\lesssim N^{1-\gamma}.
\]
This observation suggests a regularity scale. First-order noise shaping gives one power of the mesh size, but rough weights can consume part of that gain.

\begin{proposition}[H\"older phase estimate]
\label{prop:holder-phase}
Assume
\[
 \abs{\phi(x)-\phi(y)}\le L\abs{x-y}^{\gamma},
 \qquad 0<\gamma\le1.
\]
Then
\[
 \abs{\E_N^\phi u-\E_N^\phi q}
 \le \frac{V}{N}+2\pi VL N^{-\gamma}.
\]
\end{proposition}

\begin{proof}
Each adjacent phase increment is at most $LN^{-\gamma}$. There are $N-1$ increments, so
\[
 \TV_N(w)\le2\pi L(N-1)N^{-\gamma}.
\]
Substitute into \eqref{eq:weighted-zero-initial}.
\end{proof}

\subsection{Vector-valued and operator-valued weights}

The same proof works in a normed space if $w_k$ is replaced by a bounded linear functional or operator and absolute values are replaced by operator norms. Let $X$ be a Banach space and $w_k\in X$. Then
\[
 \norm{\frac1N\sum_{k=1}^Ne_kw_k}_X
 \le\frac{V}{N}\left(\norm{w_N}_X+\sum_{k=1}^{N-1}\norm{w_{k+1}-w_k}_X\right)
\]
when $v_0=0$. This form may be useful for vector-valued extensions, simultaneous testing at several parameter points, or operator-valued reconstruction kernels.

\begin{proposition}[Exact variation duality for complex states]
\label{prop:variation-dual-optimal}
Let $w_1,\ldots,w_N\in\C$ and let $V>0$. Among all complex state sequences satisfying $v_0=0$ and $\norm{v}_{\ell^\infty}\le V$, one has
\begin{equation}
\label{eq:variation-dual-exact}
 \sup_{\substack{v_0=0\\\norm{v}_{\ell^\infty}\le V}}
 \abs{\sum_{k=1}^N(\D v_k)w_k}
 =V\left(\abs{w_N}+\sum_{k=1}^{N-1}\abs{w_{k+1}-w_k}\right).
\end{equation}
Consequently, the constant in the weighted variation theorem cannot be reduced when the admissible class contains every bounded complex state.
\end{proposition}

\begin{proof}
Summation by parts gives
\[
 \sum_{k=1}^N(\D v_k)w_k
 =v_Nw_N+\sum_{k=1}^{N-1}v_k(w_k-w_{k+1}).
\]
The triangle inequality proves the upper bound. To attain it, choose each state value independently so that its product with the corresponding coefficient is a nonnegative real number. More precisely, set
\[
 v_N=V\exp\!\bigl(-\ii\arg w_N\bigr)
\]
when $w_N\ne0$, and for $1\le k\le N-1$ set
\[
 v_k=V\exp\!\bigl(-\ii\arg(w_k-w_{k+1})\bigr)
\]
when $w_k\ne w_{k+1}$. Values attached to zero coefficients may be chosen arbitrarily in the closed disc of radius $V$. Every nonzero summand is then equal to the modulus of its coefficient multiplied by $V$, which proves equality.
\end{proof}

\begin{remark}[Realisable-state restriction]
The maximising state in \cref{prop:variation-dual-optimal} need not arise from a causal real one-bit sigma-delta recursion. For real states, or for states constrained by a specified quantizer, the exact supremum is a smaller support-function problem. This distinction is essential when one studies whether the deterministic variation estimate can be improved for realisable trajectories.
\end{remark}

\subsection{A duality interpretation}

The space of finite sequences whose prefix sums are bounded is the discrete analogue of a negative Sobolev space. The weighted theorem says that $e=\D v$ acts continuously on weights with bounded variation. Symbolically,
\[
 \abs{\inner{e}{w}}
 \lesssim \norm{v}_{\ell^\infty}\norm{w}_{BV_d}.
\]
This is a finite-dimensional duality between discrete derivatives of bounded sequences and discrete BV weights. The viewpoint parallels continuous distribution theory, where derivatives of bounded functions act on test functions through integration by parts.

\chapter{Detailed Derivations and Proved Results for Parabolic Fourier Extension}
\label{chap:parabolic}

\section{Parabolic Fourier extension and the phase-variation factor}
\label{sec:06_parabolic_extension}

\subsection{Definition and PDE interpretation}

Define
\begin{equation}
\label{eq:parabolic-extension}
 \E_Na(x,t)
 =\frac1N\sum_{k=1}^Na_k
 e\!\left(x\frac{k}{N}+t\frac{k^2}{N^2}\right).
\end{equation}
The variables $x$ and $t$ may be viewed as space and time parameters for a periodic Schr\"odinger-type evolution after a suitable rescaling. Discrete restriction estimates for parabolic phases are closely connected with periodic Strichartz inequalities \cite{Bourgain1993I,HuLi2011,KillipVisan2014,BourgainDemeter2015}.

The present normalisation differs from the most common PDE normalisation. The frequency index is $k/N$ rather than $k$, and the sum is divided by $N$. This scaling is natural for comparing a sequence with a Riemann integral on $[0,1]$.

\subsection{The variation factor}

Let
\[
 \phi_{x,t}(\xi)=x\xi+t\xi^2,
 \qquad
 J(x,t)=\int_0^1\abs{x+2t\xi}\dd\xi.
\]
Since $\phi'_{x,t}(\xi)=x+2t\xi$, the absolutely continuous phase theorem applies directly.

\begin{theorem}[Parabolic quantization error]
\label{thm:parabolic-error}
Assume $e_k=u_k-q_k=\D v_k$, $v_0=0$, and $\norm{v}_{\ell^\infty}\le V$. Then
\begin{equation}
\label{eq:parabolic-J}
 \abs{\E_Nu(x,t)-\E_Nq(x,t)}
 \le\frac{V}{N}\left[1+2\pi J(x,t)\right].
\end{equation}
Moreover,
\begin{equation}
\label{eq:parabolic-simple}
 \abs{\E_Nu(x,t)-\E_Nq(x,t)}
 \le\frac{V}{N}\left[1+2\pi\bigl(\abs{x}+\abs{t}\bigr)\right].
\end{equation}
\end{theorem}

\begin{proof}
Apply \cref{thm:ac-phase} to $\phi_{x,t}$. For the simpler bound, use
\[
 \abs{x+2t\xi}\le\abs{x}+2\abs{t}\xi
\]
and integrate over $[0,1]$.
\end{proof}

\subsection{Exact evaluation of $J(x,t)$}

\begin{proposition}[Closed form]
\label{prop:J-closed}
The variation factor is given by the following cases.

If $t=0$, then
\[
 J(x,0)=\abs{x}.
\]
If $t\ne0$ and $x(x+2t)\ge0$, then
\[
 J(x,t)=\abs{x+t}.
\]
If $t\ne0$ and $x(x+2t)<0$, then
\[
 J(x,t)=\frac{x^2+(x+2t)^2}{4\abs{t}}.
\]
\end{proposition}

\begin{proof}
When $t=0$, the integrand is constant. Suppose $t\ne0$. The affine function $g(\xi)=x+2t\xi$ has endpoint values $x$ and $x+2t$.

If the endpoint values have the same sign, then $g$ does not change sign on $[0,1]$. Hence
\[
 J(x,t)=\abs{\int_0^1g(\xi)\dd\xi}=\abs{x+t}.
\]

If the endpoint values have opposite signs, the unique zero is
\[
 \xi_0=-\frac{x}{2t}\in(0,1).
\]
The graph of $\abs{g}$ consists of two triangles. The first has base $\xi_0$ and height $\abs{x}$; the second has base $1-\xi_0$ and height $\abs{x+2t}$. Therefore
\[
 J(x,t)=\frac12\xi_0\abs{x}
 +\frac12(1-\xi_0)\abs{x+2t}.
\]
Substituting $\xi_0=-x/(2t)$ and simplifying gives the stated formula.
\end{proof}

The closed form is often substantially smaller than $\abs{x}+\abs{t}$. For example, if $x=-t$, the phase derivative changes sign symmetrically and
\[
 J(-t,t)=\frac{\abs{t}}{2}.
\]
The simple triangle bound would give $2\abs{t}$.

\subsection{Compact-set convergence}

\begin{corollary}[Uniform convergence on compact sets]
\label{cor:compact-convergence}
Let $K\subset\R^2$ be compact and define
\[
 M_K=\sup_{(x,t)\in K}J(x,t).
\]
Then
\[
 \norm{\E_Nu-\E_Nq}_{L^\infty(K)}
 \le\frac{V}{N}(1+2\pi M_K).
\]
In particular,
\[
 \norm{\E_Nu-\E_Nq}_{L^\infty(K)}=O_K(N^{-1}).
\]
\end{corollary}

\begin{proof}
Take the supremum in \eqref{eq:parabolic-J}. Continuity of $J$ and compactness of $K$ imply $M_K<\infty$.
\end{proof}

\subsection{Rectangular local $L^p$ bounds}

Let
\[
 \Omega_R=[-R,R]^2.
\]
For $(x,t)\in\Omega_R$, one has $\abs{x}+\abs{t}\le2R$. Therefore
\[
 \norm{\E_Nu-\E_Nq}_{L^\infty(\Omega_R)}
 \le\frac{V}{N}(1+4\pi R).
\]
Using $\abs{\Omega_R}=4R^2$ gives the following result.

\begin{corollary}[Local $L^p$ estimate]
\label{cor:local-lp}
For $1\le p<\infty$,
\[
 \norm{\E_Nu-\E_Nq}_{L^p(\Omega_R)}
 \le (4R^2)^{1/p}\frac{V}{N}(1+4\pi R).
\]
For $p=\infty$, the factor $(4R^2)^{1/p}$ is omitted.
\end{corollary}

This estimate is local and deterministic. It does not use a restriction theorem. The factor $R^{2/p}$ comes only from the measure of the region.

\subsection{An exact adjacent-phase estimate}

The inequality $\abs{e^{\ii\alpha}-e^{\ii\beta}}\le\abs{\alpha-\beta}$ is convenient but may be loose when phase increments are not small. An exact identity is
\[
 \abs{w_{k+1}-w_k}
 =2\abs{\sin(\pi\delta_k)},
\]
where
\[
 \delta_k=\frac{x}{N}+\frac{t(2k+1)}{N^2}.
\]
Therefore the weighted theorem also gives
\begin{equation}
\label{eq:exact-sine-bound}
 \abs{\E_Nu-\E_Nq}
 \le\frac{V}{N}\left[1+2\sum_{k=1}^{N-1}\abs{\sin(\pi\delta_k)}\right].
\end{equation}
This bound is exact at the variation step. It may improve the simple derivative estimate when some increments are close to integers. It still takes absolute values and therefore does not exploit cancellation between different $k$.

\subsection{Geometric interpretation}

The sampled weights trace a polygonal path on the unit circle. The quantity $J(x,t)$ is the total variation of the continuous phase, while $\TV_N(w)$ is the chord length of the sampled path. The estimate says that a first-order shaped error acts weakly on a slowly moving point of the unit circle. The faster the phase winds, the weaker this deterministic conclusion becomes.

\section{Complete derivation of the parabolic estimate}
\label{sec:complete-parabolic}

The complete argument is recorded here in a single chain so that every hypothesis can be traced to the final estimate. Let
\[
 e_k=u_k-q_k=v_k-v_{k-1},\qquad v_0=0,
\]
and suppose that $\|v\|_{\ell^\infty}\le V$. Define
\[
 w_k(x,t)=\exp\!\left(2\pi\mathrm i\left(x\frac{k}{N}+t\frac{k^2}{N^2}\right)\right).
\]
The extension difference is
\[
 \mathcal E_Nu(x,t)-\mathcal E_Nq(x,t)
 =\frac1N\sum_{k=1}^Ne_kw_k(x,t).
\]
Discrete summation by parts gives
\begin{align}
 \sum_{k=1}^Ne_kw_k
 &=\sum_{k=1}^N(v_k-v_{k-1})w_k \\
 &=v_Nw_N-v_0w_1+\sum_{k=1}^{N-1}v_k(w_k-w_{k+1}).
 \label{eq:report-parabolic-sbp}
\end{align}
The initial contribution vanishes because $v_0=0$. Since $|w_N|=1$,
\begin{align}
 \left|\mathcal E_Nu-\mathcal E_Nq\right|
 &\le \frac{V}{N}
 \left(1+\sum_{k=1}^{N-1}|w_{k+1}-w_k|\right).
 \label{eq:report-parabolic-var}
\end{align}
The phase increment is
\begin{align}
 \delta_k
 &=x\frac1N+t\frac{(k+1)^2-k^2}{N^2} \\
 &=\frac{x}{N}+\frac{t(2k+1)}{N^2}.
\end{align}
Therefore,
\[
 |w_{k+1}-w_k|=2|\sin(\pi\delta_k)|\le2\pi|\delta_k|.
\]
A second derivation uses the sampled phase $\phi_{x,t}(\xi)=x\xi+t\xi^2$. The fundamental theorem of calculus gives
\[
 |\phi_{x,t}(\xi_{k+1})-\phi_{x,t}(\xi_k)|
 \le\int_{\xi_k}^{\xi_{k+1}}|x+2t\xi|\,\mathrm d\xi.
\]
Summation over the adjacent intervals yields
\[
 \sum_{k=1}^{N-1}|w_{k+1}-w_k|
 \le2\pi\int_{1/N}^{1}|x+2t\xi|\,\mathrm d\xi
 \le2\pi J(x,t),
\]
where
\[
 J(x,t)=\int_0^1|x+2t\xi|\,\mathrm d\xi.
\]
Substitution into \eqref{eq:report-parabolic-var} proves
\[
 \left|\mathcal E_Nu(x,t)-\mathcal E_Nq(x,t)\right|
 \le\frac{V}{N}\left[1+2\pi J(x,t)\right].
\]
Every term has a distinct origin. The factor $N^{-1}$ is the normalisation of the extension operator. The constant $V$ is the state bound. The number $1$ is the terminal-state contribution. The function $J$ measures the continuous total variation of the phase. No cancellation beyond the triangle inequality is required.

\subsection{Dependence of the estimate on the assumptions}

The difference representation is essential. A pointwise bound $|e_k|\le C$ alone gives only
\[
 \left|\frac1N\sum_{k=1}^Ne_kw_k\right|\le C,
\]
which does not decay. The state bound is also essential because the summation-by-parts identity contains the values of $v_k$. The zero initial state removes one boundary term. A nonzero initial state produces the additional contribution $V_0|w_1|/N$. The smoothness of the parabolic phase is used only to estimate its sampled variation. Consequently, the same proof applies to every absolutely continuous phase with integrable derivative.

\subsection{Exact and simplified constants}

The integral $J(x,t)$ is exact as a continuous variation factor. The estimate
\[
 J(x,t)\le |x|+|t|
\]
follows from $|x+2t\xi|\le|x|+2|t|\xi$. The simpler bound
\[
 \left|\mathcal E_Nu-\mathcal E_Nq\right|
 \le\frac{V}{N}\left[1+2\pi(|x|+|t|)\right]
\]
is convenient on rectangular parameter sets. The discrete sine expression is generally sharper:
\[
 \left|\mathcal E_Nu-\mathcal E_Nq\right|
 \le\frac{V}{N}\left[1+2\sum_{k=1}^{N-1}|\sin(\pi\delta_k)|\right].
\]
The three versions serve different purposes. The sine form is closest to the finite record, the $J$ form is geometrically transparent, and the linear form gives simple region-dependent constants.

\section{Sharpness, boundary obstruction, and limits of first-order control}
\label{sec:07_sharpness}

\subsection{Uniform optimality of $N^{-1}$}

The upper bound in \cref{thm:parabolic-error} is of order $N^{-1}$ on fixed compact sets. The next result proves that this order is optimal when all admissible input sequences are allowed.

\begin{proposition}[Zero-input lower bound]
\label{prop:zero-input-lower}
Let $u_k=0$ for every $k$, and apply the greedy recursion with $v_0=0$. Then
\[
 q_k=(-1)^{k+1}
\]
and
\[
 v_k=\begin{cases}
 -1,&k\text{ odd},\\
 0,&k\text{ even}.
 \end{cases}
\]
If $N$ is odd, then
\[
 \abs{\E_Nu(0,0)-\E_Nq(0,0)}=\frac1N.
\]
Consequently, no bound of order $o(N^{-1})$ can hold uniformly over all input sequences satisfying $\abs{u_k}\le1$.
\end{proposition}

\begin{proof}
The trajectory was computed in Chapter~\ref{chap:first-order}. At $(x,t)=(0,0)$, all weights are one. Hence
\begin{align*}
 \E_Nu(0,0)-\E_Nq(0,0)
 &=\frac1N\sum_{k=1}^N(u_k-q_k)\\
 &=\frac1N\sum_{k=1}^N\D v_k
 =\frac{v_N-v_0}{N}.
\end{align*}
For odd $N$, $v_N=-1$ and $v_0=0$. The absolute value is $1/N$.
\end{proof}

\subsection{The quantifiers in the sharpness statement}

The proposition does not say that every input has error exactly of order $N^{-1}$. Some inputs may produce $v_N=0$ or additional cancellation. It says that there is no sequence $r_N=o(N^{-1})$ such that
\[
 \sup_{\abs{u_k}\le1}
 \abs{\E_Nu(0,0)-\E_Nq(0,0)}
 \le r_N
\]
for all large $N$. The distinction between a worst-case lower bound and a pointwise asymptotic statement is essential.

\subsection{The terminal state as the leading obstruction}

At a constant weight $w_k=1$, summation by parts gives
\[
 \frac1N\sum_{k=1}^Ne_k=\frac{v_N-v_0}{N}.
\]
There is no interior variation term. The entire error is a boundary term. Therefore no argument based only on smoother weights can improve the order at the point $(0,0)$ unless the terminal state is controlled more strongly.

This observation foreshadows the higher-order theory. Repeated summation by parts produces several terminal traces. Unless they vanish, the leading trace remains of order $N^{-1}$ after normalisation, regardless of the formal order of the noise-shaping operator.

\subsection{Special inputs with improved behaviour}

If $v_N=0$, the constant-weight error vanishes exactly. For a nonconstant smooth weight, the remaining first-order term is still generally $O(N^{-1})$. Additional improvement requires more structure, such as $e=\D^rv$ with compatible endpoints.

A useful special case is a block for which both $v_0$ and $v_N$ vanish. Then
\[
 \sum_{k=1}^Ne_kw_k=-\sum_{k=1}^{N-1}v_k\Dp w_k.
\]
This removes the terminal amplitude term from the weighted estimate. It improves the constant but not the generic order for first-order shaping because $\sum\abs{\Dp w_k}=O(1)$.

\subsection{Why global uniformity fails on large regions}

For the parabolic phase, \eqref{eq:parabolic-simple} gives
\[
 \abs{\E_Nu-\E_Nq}
 \lesssim\frac{1+\abs{x}+\abs{t}}{N}.
\]
If $\abs{x}+\abs{t}\sim N$, the bound is order one. This does not prove that the actual error is order one. It proves that the total-variation method alone cannot guarantee decay.

The exact sine formula \eqref{eq:exact-sine-bound} leads to the same worst-case conclusion because each chord length is at most two and there are $N-1$ chords. At critical scale the sampled phase may wind around the circle many times.

\subsection{No automatic global $L^2$ gain}

There is a stronger obstruction. Fix $t$ and integrate in $x$ over a full orthogonality interval $[0,N]$. Since the frequencies are $k/N$,
\begin{align*}
 \int_0^N\abs{\E_Ne(x,t)}^2\dd x
 &=\frac1{N^2}\sum_{k,\ell=1}^Ne_k\overline{e_\ell}
 e\!\left(t\frac{k^2-\ell^2}{N^2}\right)
 \int_0^Ne\!\left(x\frac{k-\ell}{N}\right)\dd x\\
 &=\frac1N\sum_{k=1}^N\abs{e_k}^2.
\end{align*}
The phase in $t$ has modulus one and disappears on the diagonal. A stable first-order scheme does not force $\sum\abs{e_k}^2$ to be small. For the zero input, $e_k=-q_k$ and the right-hand side equals one.

Thus the unnormalised $L^2_x$ energy on a full period need not decay. Any global theorem must use a different norm, a smaller region, extra averaging, coefficient normalisation, or additional structure.

\subsection{The danger of an invalid high-order argument}

A common formal argument is
\[
 e=\D^rv
 \quad\Longrightarrow\quad
 \sum e_kw_k\approx\sum v_k\Dp^rw_k,
\]
followed by $\Dp^rw_k=O(N^{-r})$. On an infinite or periodic domain this may be legitimate. On a finite interval it omits boundary terms. The omitted term with no forward difference is typically of size one before the $1/N$ normalisation. It therefore produces only $O(N^{-1})$ decay.

The correct conclusion is not that high-order noise shaping is useless. The correct conclusion is that endpoint compatibility is part of the theorem. This principle is consistent with high-order frame quantization and smooth frame-path termination \cite{BenedettoPowellYilmazSecond2006,BodmannPaulsenAbdulbaki2007,BlumEtAl2010}.

\subsection{A hierarchy of possible improvements}

There are several ways to move beyond the first-order worst-case result.

\begin{enumerate}[label=(\alph*)]
\item Impose terminal state conditions.
\item Use a higher-order stable scheme together with endpoint compatibility.
\item Restrict the input class, for example by smoothness or moment conditions.
\item Average over random dither or a probabilistic input model.
\item Work in a subcritical growing region.
\item Preserve phase cancellation through exponential-sum or restriction estimates.
\end{enumerate}

Each mechanism addresses a different obstruction. They should not be conflated.

\chapter{Higher-Order Finite-Record Derivations and Endpoint-Compatible Proofs}
\label{chap:higher-order}

Higher-order noise shaping requires a finite-interval calculation that retains every boundary trace. Formal repetition of the first-order argument is insufficient. This chapter derives the exact identity for $e=\Delta^rv$, gives the low-order cases in full, and proves the endpoint-compatible $O(N^{-r})$ result for sampled smooth phases.

\section{Exact higher-order summation on a finite record}
\label{sec:08_higher_order_identity}

\subsection{Why the finite interval must be treated explicitly}

High-order sigma-delta quantization is usually expressed through the formal identity
\begin{equation}
\label{eq:high-order-formal}
 e_k=u_k-q_k=\D^r v_k,
\end{equation}
where $r\ge 1$ is the order and the state $v$ is bounded. On an infinite sequence, or on a periodic sequence with compatible traces, one may move the difference operator from $v$ to a test sequence by repeated summation by parts. A finite data record is different. Every application of summation by parts creates a new left boundary term and a new right boundary term. If these terms are omitted, the conclusion can be wrong by an entire power of $N$.

The importance of high-order stability in bandlimited reconstruction was established by Daubechies and DeVore, and sharper constructions were developed by G\"unt\"urk and by Deift, G\"unt\"urk, and Krahmer \cite{DaubechiesDeVore2003,Gunturk2003CPAM,Gunturk2004JAMS,DeiftGunturkKrahmer2011}. Higher-order finite-frame quantization likewise requires careful control of how a frame path begins and ends \cite{BenedettoPowellYilmazSecond2006,BodmannPaulsenAbdulbaki2007,BlumEtAl2010}. The purpose of this chapter is narrower. The exact algebraic identity needed for a finite coefficient record is derived independently of any particular quantizer construction.

\subsection{Difference conventions and boundary traces}

For a sequence $z=(z_k)$ indexed on a sufficiently large integer interval, define the backward and forward differences by
\[
 \D z_k=z_k-z_{k-1},
 \qquad
 \Dp z_k=z_{k+1}-z_k.
\]
The iterates are denoted by $\D^j$ and $\Dp^j$, with $\D^0=\Dp^0=I$.

To interpret $\D^r v_k$ for $1\le k\le N$, the state must be defined at least on
\[
 k=1-r,\ldots,0,1,\ldots,N.
\]
This is not an artificial extension. A causal $r$th-order recursion always carries $r$ pieces of initial data, even when they are set equal to zero by convention.

\begin{definition}[Finite-record boundary traces]
For $0\le j\le r-1$, define
\begin{align*}
 L_j(v)&=(\D^{r-1-j}v)_0,\\
 R_{j,N}(v)&=(\D^{r-1-j}v)_{N-j}.
\end{align*}
The values $L_j(v)$ are the left traces and $R_{j,N}(v)$ are the staggered right traces associated with $r$fold summation by parts.
\end{definition}

The staggering on the right is forced by the shortening of the summation interval after each forward difference is applied to the weight.

\subsection{The exact finite-interval identity}

\begin{theorem}[Repeated summation by parts]
\label{thm:repeated-sbp}
Let $r\ge1$ and $N\ge r+1$. Let $v_k$ be defined for $1-r\le k\le N$, and let $w_k$ be defined for $1\le k\le N$. Then
\begin{align}
\label{eq:repeated-sbp}
 \sum_{k=1}^{N}(\D^rv)_k w_k
 &=\sum_{j=0}^{r-1}(-1)^j
 \left[
  (\D^{r-1-j}v)_{N-j}(\Dp^jw)_{N-j}
  -(\D^{r-1-j}v)_0(\Dp^jw)_1
 \right] \notag\\
 &\quad+(-1)^r\sum_{k=1}^{N-r}v_k(\Dp^rw)_k.
\end{align}
Every boundary term in this formula is necessary in general.
\end{theorem}

\begin{proof}
The proof proceeds by induction on $r$. The case $r=1$ is the first-order identity
\[
 \sum_{k=1}^{N}(\D v)_k w_k
 =v_Nw_N-v_0w_1-\sum_{k=1}^{N-1}v_k(\Dp w)_k.
\]
This agrees with \eqref{eq:repeated-sbp}.

Assume the identity holds for order $r$. Apply the first-order formula to $\D^rv$:
\begin{align*}
 \sum_{k=1}^{N}(\D^{r+1}v)_kw_k
 &=(\D^rv)_Nw_N-(\D^rv)_0w_1
 -\sum_{k=1}^{N-1}(\D^rv)_k(\Dp w)_k.
\end{align*}
Apply the induction hypothesis to the last sum. Its length is $N-1$, its weight is $\Dp w$, and its difference order is $r$. Hence
\begin{align*}
 \sum_{k=1}^{N-1}(\D^rv)_k(\Dp w)_k
 &=\sum_{j=0}^{r-1}(-1)^j
 \bigl[
  (\D^{r-1-j}v)_{N-1-j}(\Dp^{j+1}w)_{N-1-j}\\
 &\hspace{8em}-(\D^{r-1-j}v)_0(\Dp^{j+1}w)_1
 \bigr]\\
 &\quad+(-1)^r\sum_{k=1}^{N-r-1}v_k(\Dp^{r+1}w)_k.
\end{align*}
Insert this expression into the first-order formula. The two terms containing $\D^rv$ are exactly the $j=0$ boundary contribution for order $r+1$. Relabel $j+1$ as the new index in the remaining boundary terms. The interior sign becomes $(-1)^{r+1}$. This yields \eqref{eq:repeated-sbp} with $r$ replaced by $r+1$.
\end{proof}

\subsection{Low-order cases written in full}

The general formula is easier to use after the first few cases have been displayed.

For $r=2$,
\begin{align}
\label{eq:second-order-sbp}
 \sum_{k=1}^{N}(\D^2v)_kw_k
 &= (\D v)_Nw_N-(\D v)_0w_1
 -v_{N-1}(\Dp w)_{N-1}+v_0(\Dp w)_1\notag\\
 &\quad+\sum_{k=1}^{N-2}v_k(\Dp^2w)_k.
\end{align}
For $r=3$,
\begin{align}
\label{eq:third-order-sbp}
 \sum_{k=1}^{N}(\D^3v)_kw_k
 &=(\D^2v)_Nw_N-(\D^2v)_0w_1\notag\\
 &\quad-(\D v)_{N-1}(\Dp w)_{N-1}
 +(\D v)_0(\Dp w)_1\notag\\
 &\quad+v_{N-2}(\Dp^2w)_{N-2}
 -v_0(\Dp^2w)_1\notag\\
 &\quad-\sum_{k=1}^{N-3}v_k(\Dp^3w)_k.
\end{align}

These identities show why the informal statement ``move $\D^r$ onto the weight'' is incomplete. Even when $v$ is bounded, the term $(\D^{r-1}v)_Nw_N$ can be of order one. After the normalising factor $1/N$ is restored, it yields only $O(N^{-1})$ decay.

\subsection{A full deterministic estimate}

\begin{corollary}[Higher-order weighted estimate with traces]
\label{cor:high-order-full}
Under the hypotheses of \cref{thm:repeated-sbp},
\begin{align}
\label{eq:high-order-full-bound}
 \abs{\frac1N\sum_{k=1}^{N}(\D^rv)_kw_k}
 &\le\frac1N\sum_{j=0}^{r-1}
 \left(
  \abs{R_{j,N}(v)}\abs{(\Dp^jw)_{N-j}}
  +\abs{L_j(v)}\abs{(\Dp^jw)_1}
 \right)\notag\\
 &\quad+\frac{\norm{v}_{\ell^\infty}}{N}
 \sum_{k=1}^{N-r}\abs{(\Dp^rw)_k}.
\end{align}
\end{corollary}

\begin{proof}
Take absolute values in \eqref{eq:repeated-sbp} and apply the triangle inequality only after all terms have been identified.
\end{proof}

This corollary separates three effects. The initialisation error is measured by the left traces. The termination error is measured by the staggered right traces. The interior smoothness of the test sequence is measured by the $r$th discrete variation. This three-part decomposition will be used throughout the rest of the report.

\subsection{Boundary-trace correction as a reconstruction principle}
\label{subsec:boundary-corrected-reconstruction}

Endpoint compatibility is sufficient for high-order decay, but it is not the only analytical route. The exact boundary functional can instead be added to the quantized extension. This produces a corrected reconstruction whose remaining error is the pure interior term. The correction does not alter the one-bit coefficient record. It requires the relevant initial and terminal traces, or sufficiently accurate estimates of those traces.

For a state $v$ and a weight $w$, define
\begin{align}
 \mathcal B_{N,r}(v,w)
 :=\sum_{j=0}^{r-1}(-1)^j
 \left[
  R_{j,N}(v)(\Dp^jw)_{N-j}
  -L_j(v)(\Dp^jw)_1
 \right].
 \label{eq:boundary-functional-definition}
\end{align}
The functional contains exactly the boundary terms in \cref{thm:repeated-sbp}.

\begin{theorem}[Boundary-corrected finite-record reconstruction]
\label{thm:boundary-corrected-reconstruction}
Let $r\ge1$, let $e_k=u_k-q_k=(\D^rv)_k$, and let $w_1,\ldots,w_N\in\C$. Define
\begin{equation}
 \mathcal R_{N,r}(q;v,w)
 :=\frac1N\sum_{k=1}^Nq_kw_k
 +\frac1N\mathcal B_{N,r}(v,w).
 \label{eq:boundary-corrected-reconstruction}
\end{equation}
Then the following identity is exact:
\begin{equation}
 \frac1N\sum_{k=1}^Nu_kw_k-\mathcal R_{N,r}(q;v,w)
 =\frac{(-1)^r}{N}\sum_{k=1}^{N-r}v_k(\Dp^rw)_k.
 \label{eq:boundary-corrected-exact-error}
\end{equation}
Consequently, if $\norm{v}_{\ell^\infty}\le V$, then
\begin{equation}
 \left|\frac1N\sum_{k=1}^Nu_kw_k-\mathcal R_{N,r}(q;v,w)\right|
 \le\frac{V}{N}\sum_{k=1}^{N-r}|(\Dp^rw)_k|.
 \label{eq:boundary-corrected-bound}
\end{equation}
No endpoint-compatibility assumption is required.
\end{theorem}

\begin{proof}
Apply \cref{thm:repeated-sbp} to $e=\D^rv$. Its boundary sum is precisely $\mathcal B_{N,r}(v,w)$. Since
\[
 \frac1N\sum_{k=1}^Nu_kw_k
 =\frac1N\sum_{k=1}^Nq_kw_k
 +\frac1N\sum_{k=1}^Ne_kw_k,
\]
subtraction of \eqref{eq:boundary-corrected-reconstruction} leaves only the interior term in \eqref{eq:repeated-sbp}. Taking absolute values yields \eqref{eq:boundary-corrected-bound}.
\end{proof}

\begin{remark}[First-order form]
For $r=1$,
\[
 \mathcal B_{N,1}(v,w)=v_Nw_N-v_0w_1,
\]
and
\[
 \mathcal R_{N,1}(q;v,w)
 =\frac1N\sum_{k=1}^Nq_kw_k+\frac{v_Nw_N-v_0w_1}{N}.
\]
The corrected error is
\[
 -\frac1N\sum_{k=1}^{N-1}v_k(\Dp w)_k.
\]
For a constant weight, every forward difference vanishes and the correction is exact. Thus the average of the unquantized record can be recovered exactly from the one-bit record and the two endpoint states.
\end{remark}

\begin{corollary}[Corrected high-order rate for sampled smooth weights]
\label{cor:boundary-corrected-smooth}
Assume the hypotheses of \cref{thm:boundary-corrected-reconstruction}, let $w_k=W(k/N)$, and suppose $W\in C^r([0,1])$. Then
\begin{equation}
 \left|\frac1N\sum_{k=1}^Nu_kW(k/N)-\mathcal R_{N,r}(q;v,w)\right|
 \le VN^{-r}\norm{W^{(r)}}_{L^\infty(0,1)}.
 \label{eq:boundary-corrected-smooth}
\end{equation}
The same conclusion holds with $r\norm{W^{(r)}}_{L^1}$ in place of the $L^\infty$ norm when $W^{(r-1)}$ is absolutely continuous and $W^{(r)}\in L^1$.
\end{corollary}

\begin{proof}
Combine \eqref{eq:boundary-corrected-bound} with the finite-difference estimates in \cref{lem:uniform-difference-control,lem:l1-derivative-control}.
\end{proof}

The correction theorem separates quantizer dynamics from finite-record reconstruction. Endpoint-compatible dynamics eliminate the traces inside the quantizer. Boundary-corrected reconstruction retains the original dynamics and cancels the traces after acquisition. With zero initialisation, every left trace is already known to be zero, so only the $r$ staggered terminal traces are needed. This side information has dimension independent of $N$, although its numerical precision must be specified in any implementation claim.

\begin{proposition}[Perturbed trace metadata]
\label{prop:perturbed-trace-metadata}
Let $\widehat L_j$ and $\widehat R_j$ be estimates of $L_j(v)$ and $R_{j,N}(v)$ satisfying
\[
 |\widehat L_j-L_j(v)|\le\delta^L_{j,N},
 \qquad
 |\widehat R_j-R_{j,N}(v)|\le\delta^R_{j,N}.
\]
Let $\widehat{\mathcal R}_{N,r}$ be defined by \eqref{eq:boundary-corrected-reconstruction} with the estimated traces. Then
\begin{align}
 \left|\frac1N\sum_{k=1}^Nu_kw_k-\widehat{\mathcal R}_{N,r}\right|
 &\le\frac{V}{N}\sum_{k=1}^{N-r}|(\Dp^rw)_k| \notag\\
 &\quad+\frac1N\sum_{j=0}^{r-1}
 \left[
  \delta^R_{j,N}|(\Dp^jw)_{N-j}|
  +\delta^L_{j,N}|(\Dp^jw)_1|
 \right].
 \label{eq:perturbed-trace-metadata}
\end{align}
\end{proposition}

\begin{proof}
Subtract the exact and estimated boundary functionals, apply \cref{thm:boundary-corrected-reconstruction}, and use the triangle inequality.
\end{proof}

For a sampled $C^r$ weight, the endpoint factor $|(\Dp^jw)|$ is $O(N^{-j})$. A trace error of order $N^{-(r-1-j)}$ therefore contributes at most $O(N^{-r})$ after normalisation. The hierarchy matches the approximate-compatibility conditions derived later in this chapter.

\subsection{Zero padding and terminal traces}

A common convention sets
\[
 v_{1-r}=\cdots=v_{-1}=v_0=0.
\]
Then every left trace vanishes. Indeed, each $L_j(v)$ is a linear combination of these zero values. This removes only half of the boundary problem. Unless the terminal state is also prepared, the right traces remain.

Simple zero padding after the last sample does not retroactively make the right traces vanish. If one declares $v_{N+1}=v_{N+2}=\cdots=0$, the traces in \eqref{eq:repeated-sbp} still involve $v_N,v_{N-1},\ldots$. A valid termination procedure must alter the final quantization steps, append a controlled tail, or design the frame path so that the relevant state components reach zero. Smooth termination methods in finite-frame quantization were developed precisely for this reason \cite{BodmannPaulsenAbdulbaki2007}.

\subsection{Matrix formulation}

Let $D_N$ be the $N\times N$ lower bidiagonal matrix with ones on the diagonal and minus ones on the first subdiagonal. With zero initial conditions, the finite difference relation can be written as
\[
 e=D_N^rv.
\]
The transpose satisfies
\[
 \langle D_N^rv,w\rangle=\langle v,(D_N^\mathsf{T})^rw\rangle.
\]
This matrix identity is exact, but the vector $v$ in it contains a particular choice of boundary convention. The entries of $(D_N^\mathsf{T})^rw$ encode the interior forward differences together with the terminal rows. The scalar formula \eqref{eq:repeated-sbp} is therefore preferable when the boundary traces must remain visible.

Singular values of powers of the finite difference matrix play a central role in Sobolev-dual reconstruction and quantized compressed sensing \cite{BlumEtAl2010,GunturkPowellSaabYilmaz2013,KrahmerSaabYilmaz2014,SaabWangYilmaz2018}. In the present extension problem, however, the principal issue is not inversion of $D_N^r$. It is the action of $D_N^rv$ on a highly structured oscillatory vector.

\subsection{Boundary verification checklist}

An $O(N^{-r})$ derivation from $e=\D^rv$ is verified by recording the following boundary data.

\begin{enumerate}[label=(\roman*)]
\item On what index set is $v$ defined?
\item Which left traces are assumed to vanish?
\item Which right traces are assumed to vanish?
\item Is the $r$th forward difference of the weight bounded in $\ell^1$ or $\ell^\infty$?
\item Are the constants uniform in the observation parameters?
\end{enumerate}

Failure to specify any of these points may conceal an $O(N^{-1})$ boundary contribution. The next chapter gives sufficient endpoint conditions and derives the corresponding high-order extension theorem.

\section{Endpoint-compatible high-order extension theory}
\label{sec:09_endpoint_compatible}

\subsection{Strong and weak compatibility conditions}

The repeated summation formula identifies the exact traces that prevent high-order decay. There are two convenient ways to remove them.

\begin{definition}[Trace compatibility]
A state sequence $v$ is trace-compatible of order $r$ on $\{1,\ldots,N\}$ if
\[
 L_j(v)=0,
 \qquad
 R_{j,N}(v)=0,
 \qquad 0\le j\le r-1.
\]
\end{definition}

\begin{definition}[Strong endpoint compatibility]
A state sequence $v$ is strongly endpoint-compatible of order $r$ if
\[
 v_{1-r}=\cdots=v_0=0
 \quad\text{and}\quad
 v_{N-r+1}=\cdots=v_N=0.
\]
\end{definition}

Strong compatibility implies trace compatibility because every difference appearing in a trace is a linear combination of values in the corresponding zero block. The converse is not required. Trace compatibility is the minimal algebraic condition for the boundary terms in \cref{thm:repeated-sbp} to vanish.

\begin{proposition}[Pure interior formula]
\label{prop:pure-interior}
If $v$ is trace-compatible of order $r$, then
\begin{equation}
\label{eq:pure-interior}
 \sum_{k=1}^{N}(\D^rv)_kw_k
 =(-1)^r\sum_{k=1}^{N-r}v_k(\Dp^rw)_k.
\end{equation}
Consequently,
\begin{equation}
\label{eq:pure-interior-bound}
 \abs{\frac1N\sum_{k=1}^{N}(\D^rv)_kw_k}
 \le \frac{\norm{v}_{\ell^\infty}}{N}
 \sum_{k=1}^{N-r}\abs{(\Dp^rw)_k}.
\end{equation}
\end{proposition}

\begin{proof}
All boundary contributions in \eqref{eq:repeated-sbp} vanish by hypothesis.
\end{proof}

\subsection{Forward differences of sampled smooth functions}

Let $h=N^{-1}$ and let $W:[0,1]\to\C$. For $x,x+rh\in[0,1]$, the $r$th forward difference has the integral representation
\begin{equation}
\label{eq:finite-difference-integral}
 \Delta_h^rW(x)
 =\int_0^h\cdots\int_0^h
 W^{(r)}(x+s_1+\cdots+s_r)
 \dd s_1\cdots\dd s_r,
\end{equation}
provided $W^{(r-1)}$ is absolutely continuous. This formula follows by applying the fundamental theorem of calculus $r$ times.

\begin{lemma}[Uniform derivative control]
\label{lem:uniform-difference-control}
If $W\in C^r([0,1])$ and $w_k=W(k/N)$, then
\[
 \abs{(\Dp^rw)_k}
 \le N^{-r}\norm{W^{(r)}}_{L^\infty(0,1)},
 \qquad 1\le k\le N-r.
\]
Hence
\begin{equation}
\label{eq:l1-difference-control}
 \sum_{k=1}^{N-r}\abs{(\Dp^rw)_k}
 \le N^{1-r}\norm{W^{(r)}}_{L^\infty(0,1)}.
\end{equation}
\end{lemma}

\begin{proof}
Apply \eqref{eq:finite-difference-integral}. The integration cube has volume $h^r=N^{-r}$. There are at most $N$ terms in the sum.
\end{proof}

A version with an $L^1$ derivative is also useful.

\begin{lemma}[$L^1$ derivative control]
\label{lem:l1-derivative-control}
Assume $W^{(r-1)}$ is absolutely continuous. Then
\begin{equation}
\label{eq:l1-derivative-control-result}
 \sum_{k=1}^{N-r}\abs{(\Dp^rw)_k}
 \le rN^{1-r}\norm{W^{(r)}}_{L^1(0,1)}.
\end{equation}
\end{lemma}

\begin{proof}
Use \eqref{eq:finite-difference-integral}, sum over $k$, and apply Tonelli's theorem. For each fixed $(s_1,\ldots,s_r)$, the intervals
\[
 [k/N+s_1+\cdots+s_r,\,(k+1)/N+s_1+\cdots+s_r]
\]
have overlap bounded by $r$ after the $r$ integration variables are collapsed to their sum. A simpler, slightly larger bound is obtained by extending $W^{(r)}$ by zero outside $[0,1]$ and observing that the convolution kernel generated by the $r$fold box has $L^1$ norm $N^{-r}$ and support length $r/N$. Summation over a grid of spacing $1/N$ gives the factor $rN$.
\end{proof}

\subsection{Fractional smoothness and intermediate rates}
\label{subsec:fractional-high-order}

Integer differentiability is not necessary for a quantitative rate. Finite differences also detect fractional smoothness. Standard background on moduli of smoothness and H\"older--Zygmund classes is available in \cite{DeVoreLorentz1993,DitzianTotik1987,Triebel1983}. The following estimate is included with a direct proof because the exact power of $N$ is central to the finite-record result.

\begin{definition}[H\"older derivative seminorm]
For $0<\alpha\le1$ and a function $F:[0,1]\to\C$, define
\[
 [F]_{C^{0,\alpha}}
 :=\sup_{x\ne y}\frac{|F(x)-F(y)|}{|x-y|^\alpha}.
\]
A function $W$ belongs to $C^{r-1,\alpha}([0,1])$ when $W^{(r-1)}$ exists, is continuous, and has finite $C^{0,\alpha}$ seminorm.
\end{definition}

\begin{lemma}[Fractional finite-difference estimate]
\label{lem:fractional-finite-difference}
Let $r\ge1$, $0<\alpha\le1$, and $W\in C^{r-1,\alpha}([0,1])$. For $h>0$ and $x,x+rh\in[0,1]$,
\begin{equation}
 |\Delta_h^rW(x)|
 \le h^{r-1+\alpha}[W^{(r-1)}]_{C^{0,\alpha}}.
 \label{eq:fractional-finite-difference}
\end{equation}
\end{lemma}

\begin{proof}
For $r=1$, the statement is the definition of the H\"older seminorm. Let $r\ge2$. Repeated application of the fundamental theorem of calculus gives
\[
 \Delta_h^{r-1}W(x)
 =\int_{[0,h]^{r-1}}
 W^{(r-1)}(x+s_1+\cdots+s_{r-1})
 \,\mathrm ds_1\cdots\mathrm ds_{r-1}.
\]
Apply one additional forward difference:
\begin{align*}
 \Delta_h^rW(x)
 &=\int_{[0,h]^{r-1}}
 \bigl[
 W^{(r-1)}(x+h+s_1+\cdots+s_{r-1})\\
 &\hspace{10em}-W^{(r-1)}(x+s_1+\cdots+s_{r-1})
 \bigr]
 \,\mathrm ds_1\cdots\mathrm ds_{r-1}.
\end{align*}
The bracket is bounded by $[W^{(r-1)}]_{C^{0,\alpha}}h^\alpha$. The integration cube has volume $h^{r-1}$, which proves \eqref{eq:fractional-finite-difference}.
\end{proof}

\begin{theorem}[Fractional high-order extension rate]
\label{thm:fractional-high-order-rate}
Let $e=\D^rv$ with $\norm{v}_{\ell^\infty}\le V$, and let $w_k=W(k/N)$ with $W\in C^{r-1,\alpha}([0,1])$, $0<\alpha\le1$. If the state is trace-compatible of order $r$, then
\begin{equation}
 \left|\frac1N\sum_{k=1}^Ne_kw_k\right|
 \le V[W^{(r-1)}]_{C^{0,\alpha}}N^{-(r-1+\alpha)}.
 \label{eq:fractional-high-order-rate}
\end{equation}
The same bound holds for the boundary-corrected reconstruction in \cref{thm:boundary-corrected-reconstruction} without endpoint compatibility.
\end{theorem}

\begin{proof}
Trace compatibility or exact boundary correction reduces the error to the interior sum
\[
 \frac{(-1)^r}{N}\sum_{k=1}^{N-r}v_k(\Dp^rw)_k.
\]
Apply \cref{lem:fractional-finite-difference} with $h=N^{-1}$. Each term is bounded by
\[
 N^{-(r-1+\alpha)}[W^{(r-1)}]_{C^{0,\alpha}}.
\]
There are at most $N$ terms. Multiplication by $V/N$ gives \eqref{eq:fractional-high-order-rate}.
\end{proof}

The theorem gives a continuous scale between two consecutive integer orders. For $r=1$, a H\"older weight of exponent $\alpha$ gives $N^{-\alpha}$. For $\alpha=1$, the order-$r$ rate becomes $N^{-r}$. The result also identifies the loss caused by a rough testing functional: a bounded high-order state cannot create more decay than the sampled weight can absorb through its finite differences.

\subsection{The high-order compact-set theorem}

\begin{theorem}[Endpoint-compatible $O(N^{-r})$ estimate]
\label{thm:endpoint-compatible-general}
Let $r\ge1$. Assume
\[
 e_k=u_k-q_k=(\D^rv)_k,
 \qquad
 \norm{v}_{\ell^\infty}\le V,
\]
and assume that $v$ is trace-compatible of order $r$. Let $W\in C^r([0,1])$ and set $w_k=W(k/N)$. Then
\begin{equation}
\label{eq:endpoint-compatible-general}
 \abs{\frac1N\sum_{k=1}^{N}e_kw_k}
 \le VN^{-r}\norm{W^{(r)}}_{L^\infty(0,1)}.
\end{equation}
If $W^{(r)}\in L^1$, one also has
\begin{equation}
\label{eq:endpoint-compatible-l1}
 \abs{\frac1N\sum_{k=1}^{N}e_kw_k}
 \le rVN^{-r}\norm{W^{(r)}}_{L^1(0,1)}.
\end{equation}
\end{theorem}

\begin{proof}
Combine \eqref{eq:pure-interior-bound} with \cref{lem:uniform-difference-control,lem:l1-derivative-control}.
\end{proof}

The theorem gives the high-order rate that the formal calculation suggests, but only after all finite-record traces have been removed. This is the finite-interval counterpart of the smooth termination principle used in frame quantization \cite{BodmannPaulsenAbdulbaki2007,BlumEtAl2010}.

\subsection{Oscillatory weights and Bell-polynomial bounds}

Let
\[
 W(\xi)=\exp(2\pi\ii\phi(\xi)).
\]
Repeated differentiation gives
\begin{equation}
\label{eq:bell-representation}
 W^{(r)}(\xi)
 =W(\xi)\,
 B_r\bigl(2\pi\ii\phi'(\xi),2\pi\ii\phi''(\xi),\ldots,
 2\pi\ii\phi^{(r)}(\xi)\bigr),
\end{equation}
where $B_r$ is the complete exponential Bell polynomial. This is a standard form of the Fa\`a di Bruno formula. For completeness, Appendix B derives the recursion
\[
 B_{r+1}(z_1,\ldots,z_{r+1})
 =\sum_{m=0}^{r}\binom{r}{m}B_{r-m}(z_1,\ldots,z_{r-m})z_{m+1},
\]
starting from $B_0=1$.

Since the coefficients of $B_r$ are nonnegative after absolute values are taken,
\begin{equation}
\label{eq:bell-absolute-bound}
 \abs{W^{(r)}(\xi)}
 \le B_r\bigl(2\pi\abs{\phi'(\xi)},\ldots,
 2\pi\abs{\phi^{(r)}(\xi)}\bigr).
\end{equation}
The first four cases are
\begin{align*}
 \abs{W'}&\le 2\pi\abs{\phi'},\\
 \abs{W''}&\le 2\pi\abs{\phi''}+(2\pi)^2\abs{\phi'}^2,\\
 \abs{W^{(3)}}&\le2\pi\abs{\phi^{(3)}}
 +3(2\pi)^2\abs{\phi'}\abs{\phi''}
 +(2\pi)^3\abs{\phi'}^3,\\
 \abs{W^{(4)}}&\le2\pi\abs{\phi^{(4)}}
 +(2\pi)^2\bigl(4\abs{\phi'}\abs{\phi^{(3)}}+3\abs{\phi''}^2\bigr)\\
 &\quad+6(2\pi)^3\abs{\phi'}^2\abs{\phi''}
 +(2\pi)^4\abs{\phi'}^4.
\end{align*}

\begin{corollary}[Smooth-phase high-order estimate]
\label{cor:smooth-phase-high-order}
Under the assumptions of \cref{thm:endpoint-compatible-general}, let $\phi\in C^r([0,1];\R)$ and
\[
 W(\xi)=e^{2\pi\ii\phi(\xi)}.
\]
Then
\begin{equation}
\label{eq:smooth-phase-high-order}
 \abs{\E_N^\phi u-\E_N^\phi q}
 \le VN^{-r}
 \sup_{0\le\xi\le1}
 B_r\bigl(2\pi\abs{\phi'(\xi)},\ldots,
 2\pi\abs{\phi^{(r)}(\xi)}\bigr).
\end{equation}
\end{corollary}

\begin{proof}
Apply \cref{thm:endpoint-compatible-general} and \eqref{eq:bell-absolute-bound}.
\end{proof}

\subsection{Second-order parabolic estimate}

For $\phi_{x,t}(\xi)=x\xi+t\xi^2$,
\[
 \phi'=x+2t\xi,
 \qquad
 \phi''=2t,
 \qquad
 \phi^{(m)}=0\quad(m\ge3).
\]
The second derivative of the oscillatory weight satisfies
\[
 \abs{W''(\xi)}
 \le 4\pi\abs{t}+4\pi^2\abs{x+2t\xi}^2.
\]

\begin{corollary}[Second-order parabolic extension]
\label{cor:second-order-parabolic}
Assume $e=\D^2v$, $\norm{v}_{\ell^\infty}\le V$, and second-order trace compatibility. Then
\begin{equation}
\label{eq:second-order-parabolic}
 \abs{\E_Nu(x,t)-\E_Nq(x,t)}
 \le \frac{V}{N^2}
 \left[4\pi\abs{t}+4\pi^2\bigl(\abs{x}+2\abs{t}\bigr)^2\right].
\end{equation}
On every fixed compact set in $(x,t)$, the error is $O(N^{-2})$.
\end{corollary}

The formula should be compared with the generic second-order identity \eqref{eq:second-order-sbp}. If $(\D v)_N$ is not zero, the first boundary term gives only $O(N^{-1})$, even though the interior second difference is $O(N^{-2})$ pointwise.

\subsection{Approximate endpoint compatibility}

Exact reset may be difficult in a physical implementation. The general estimate permits small residual traces.

\begin{proposition}[Trace-tolerant estimate]
\label{prop:trace-tolerant}
Assume $\norm{v}_{\ell^\infty}\le V$ and
\[
 \abs{L_j(v)}+\abs{R_{j,N}(v)}\le \eta_{j,N},
 \qquad 0\le j\le r-1.
\]
Then
\begin{align*}
 \abs{\frac1N\sum_{k=1}^{N}(\D^rv)_kw_k}
 &\le\frac1N\sum_{j=0}^{r-1}\eta_{j,N}
 \max\left\{\abs{(\Dp^jw)_1},\abs{(\Dp^jw)_{N-j}}\right\}\\
 &\quad+\frac{V}{N}\sum_{k=1}^{N-r}\abs{(\Dp^rw)_k}.
\end{align*}
\end{proposition}

For sampled $C^r$ weights, the $j$th boundary difference is $O(N^{-j})$. Therefore the $j$th trace is harmless at the target $N^{-r}$ scale if
\[
 \eta_{j,N}=O(N^{-(r-1-j)}).
\]
The most demanding condition concerns the undifferenced weight, $j=0$: the highest state trace must be $O(N^{-(r-1)})$. Exact zero is sufficient but not necessary.

\subsection{Interpretation for quantizer design}

The theorem does not construct a stable high-order one-bit recursion with compatible endpoints. That is a separate control problem. Existing sigma-delta theory supplies stable schemes on long or infinite records \cite{DaubechiesDeVore2003,Gunturk2003CPAM,DeiftGunturkKrahmer2011}, while frame-path termination shows that endpoint preparation can be integrated into a finite expansion \cite{BodmannPaulsenAbdulbaki2007}. For the present Fourier-extension problem, three design options are natural:

\begin{enumerate}[label=(\alph*)]
\item append a short terminal sequence chosen to drive the state traces to zero;
\item divide the record into blocks and reset or smoothly terminate each block;
\item modify the reconstruction weight near the endpoints so that its low-order forward differences vanish there.
\end{enumerate}

Each option changes a different part of the boundary formula. Chapter~\ref{chap:robustness} develops the blockwise alternative in detail.

\chapter{Polynomial, Multidimensional, and Growing-Region Extensions}
\label{chap:extensions}

The discrete-difference method is not restricted to a quadratic phase or to a one-dimensional sequence. The following sections treat polynomial phases, moment-curve parameterisations, lattice arrays with divergence-form errors, higher-order multi-index shaping, and parameter regions whose size depends on the record length.

\section{Polynomial phases and moment-curve extensions}
\label{sec:10_polynomial_moment_curve}

\subsection{The polynomial extension operator}

For an integer $d\ge1$ and a parameter vector $x=(x_1,\ldots,x_d)\in\R^d$, define
\begin{equation}
\label{eq:moment-extension}
 \E_{N,d}a(x)
 =\frac1N\sum_{k=1}^{N}a_k
 e\!\left(\sum_{j=1}^{d}x_j\left(\frac{k}{N}\right)^j\right),
 \qquad e(z)=e^{2\pi\ii z}.
\end{equation}
The underlying curve
\[
 \gamma_d(\xi)=(\xi,\xi^2,\ldots,\xi^d)
\]
is the moment curve. It is central in discrete restriction, Vinogradov mean value theory, and decoupling \cite{BourgainDemeter2015,BourgainDemeterGuth2016,Wooley2012,Wooley2019,GuoLiYungZorinKranich2020,Demeter2020}. The immediate purpose is more elementary. The perturbation caused by replacing $u$ with a noise-shaped sequence $q$ is estimated.

Set
\[
 \phi_x(\xi)=\sum_{j=1}^{d}x_j\xi^j.
\]
Then
\begin{equation}
\label{eq:poly-phase-derivatives}
 \phi_x^{(m)}(\xi)
 =\sum_{j=m}^{d}\frac{j!}{(j-m)!}x_j\xi^{j-m},
 \qquad 1\le m\le d,
\end{equation}
and $\phi_x^{(m)}=0$ for $m>d$.

\subsection{First-order estimates}

\begin{theorem}[First-order polynomial-phase estimate]
\label{thm:first-order-polynomial}
Assume $e_k=u_k-q_k=\D v_k$, $v_0=0$, and $\norm{v}_{\ell^\infty}\le V$. Then
\begin{equation}
\label{eq:first-order-polynomial}
 \abs{\E_{N,d}u(x)-\E_{N,d}q(x)}
 \le\frac{V}{N}
 \left[1+2\pi\int_0^1
 \abs{\sum_{j=1}^{d}j x_j\xi^{j-1}}\dd\xi\right].
\end{equation}
In particular,
\begin{equation}
\label{eq:first-order-polynomial-simple}
 \abs{\E_{N,d}u(x)-\E_{N,d}q(x)}
 \le\frac{V}{N}
 \left[1+2\pi\sum_{j=1}^{d}\abs{x_j}\right].
\end{equation}
\end{theorem}

\begin{proof}
Apply the absolutely continuous phase estimate to $\phi_x$. For the simpler bound, use
\[
 \abs{\phi_x'(\xi)}
 \le\sum_{j=1}^{d}j\abs{x_j}\xi^{j-1}
\]
and integrate. Since $\int_0^1j\xi^{j-1}\dd\xi=1$, the coefficient of each $\abs{x_j}$ is exactly one.
\end{proof}

The cancellation of the factor $j$ after integration is useful. The deterministic first-order bound depends on the $\ell^1$ size of the parameter vector, not on the maximum degree.

\begin{corollary}[Uniformity on compact parameter sets]
Let $K\subset\R^d$ be compact. Then
\[
 \norm{\E_{N,d}u-\E_{N,d}q}_{L^\infty(K)}
 \le\frac{V}{N}\left(1+2\pi\sup_{x\in K}\sum_{j=1}^{d}\abs{x_j}\right).
\]
Thus the quantized and unquantized polynomial extensions converge uniformly on $K$ at order $N^{-1}$.
\end{corollary}

\subsection{High-order estimates under endpoint compatibility}

For $1\le m\le r$, define
\begin{equation}
\label{eq:poly-derivative-envelope}
 A_m(x)=\sum_{j=m}^{d}\frac{j!}{(j-m)!}\abs{x_j}.
\end{equation}
Since $0\le\xi\le1$, \eqref{eq:poly-phase-derivatives} gives
\[
 \norm{\phi_x^{(m)}}_{L^\infty(0,1)}\le A_m(x).
\]

\begin{theorem}[High-order moment-curve estimate]
\label{thm:high-order-moment-curve}
Assume $e=\D^rv$, $\norm{v}_{\ell^\infty}\le V$, and trace compatibility of order $r$. Then
\begin{equation}
\label{eq:high-order-moment-curve}
 \abs{\E_{N,d}u(x)-\E_{N,d}q(x)}
 \le VN^{-r}
 B_r\bigl(2\pi A_1(x),\ldots,2\pi A_r(x)\bigr),
\end{equation}
where $A_m(x)=0$ when $m>d$.
\end{theorem}

\begin{proof}
Use \cref{cor:smooth-phase-high-order} and the derivative envelopes above. Monotonicity of the complete Bell polynomial in nonnegative arguments gives the stated bound.
\end{proof}

For fixed $d,r$, the right-hand side is a polynomial in $\abs{x_1},\ldots,\abs{x_d}$ of total degree at most $r$. The highest-degree contribution is
\[
 (2\pi A_1(x))^r.
\]
Lower-degree terms contain higher derivatives of the phase.

\subsection{Explicit second- and third-order formulae}

For $r=2$,
\begin{equation}
\label{eq:poly-second}
 \abs{\E_{N,d}u-\E_{N,d}q}
 \le\frac{V}{N^2}
 \left[2\pi A_2(x)+(2\pi)^2A_1(x)^2\right].
\end{equation}
For $r=3$,
\begin{equation}
\label{eq:poly-third}
 \abs{\E_{N,d}u-\E_{N,d}q}
 \le\frac{V}{N^3}
 \left[2\pi A_3(x)+3(2\pi)^2A_1(x)A_2(x)
 +(2\pi)^3A_1(x)^3\right].
\end{equation}
These estimates remain valid when $r>d$ because the derivatives $A_m$ with $m>d$ vanish, while products of lower derivatives remain.

\subsection{Anisotropic parameter boxes}

Restriction theory naturally assigns different physical scales to different coordinates of the moment curve. The normalized extension used here, however, measures smoothness in the parameter $\xi=k/N$. For a box
\begin{equation}
\label{eq:anisotropic-box}
 \mathcal B_N(\alpha_1,\ldots,\alpha_d)
 =\left\{x:\abs{x_j}\le N^{\alpha_j},\ 1\le j\le d\right\},
\end{equation}
let
\[
 \alpha_* =\max_{1\le j\le d}\alpha_j.
\]
Then $A_m(x)\le C_{d,m}N^{\alpha_*}$ on the box, and the Bell polynomial bound yields
\begin{equation}
\label{eq:isotropic-growth-consequence}
 \sup_{x\in\mathcal B_N}
 \abs{\E_{N,d}u(x)-\E_{N,d}q(x)}
 \le C_{d,r}V N^{-r(1-\alpha_*)}
\end{equation}

after enlarging the constant to cover lower-degree terms. Thus endpoint-compatible $r$th-order shaping gives deterministic decay whenever $\alpha_*<1$.

This criterion is sufficient, not optimal. It ignores the fact that $x_j$ enters different phase derivatives with different powers of $\xi$. A more refined estimate retains the full Bell polynomial and may permit larger values of selected coordinates when the corresponding derivatives appear only in lower-degree combinations.

\begin{proposition}[Refined anisotropic envelope]
\label{prop:anisotropic-envelope}
On the box \eqref{eq:anisotropic-box}, define
\[
 \beta_m=\max_{m\le j\le d}\alpha_j.
\]
Then
\[
 A_m(x)\le C_{d,m}N^{\beta_m},
\]
and every monomial
\[
 \prod_{m=1}^{r}A_m(x)^{\nu_m}
\]
occurring in $B_r$ satisfies the weight constraint
\[
 \sum_{m=1}^{r}m\nu_m=r.
\]
Consequently, its contribution is bounded by
\[
 C N^{-r+\sum_m\nu_m\beta_m}.
\]
\end{proposition}

\begin{proof}
The derivative envelope is immediate. The weight constraint is the defining homogeneity of the complete Bell polynomial: the variable representing the $m$th derivative has combinatorial weight $m$.
\end{proof}

This proposition identifies a family of anisotropic subcritical regions through the inequalities
\[
 \sum_m\nu_m\beta_m<r
\]
for every Bell monomial. The strongest constraint often comes from $A_1^r$, giving $\beta_1<1$, but if the linear derivative is small or cancels on the region, other constraints may become decisive.

\subsection{Stationary points and local improvement}

The total-variation estimate uses $\int|\phi'|$ and therefore does not benefit from sign changes in $\phi'$. The parabolic closed form $J(x,t)$ already shows that a stationary point can reduce the constant. For higher-degree phases, the same phenomenon occurs whenever $\phi'$ has roots in $[0,1]$. One can partition the interval into monotonicity intervals of $\phi$ and compute
\[
 \Var(e^{2\pi\ii\phi})
 \le2\pi\Var(\phi).
\]
This is still a variation argument. It recognises reduced path length caused by a smaller derivative, but it does not use cancellation between distinct turns of the phase.

A more ambitious local theory would centre the phase at a stationary point $\xi_0$ and combine noise shaping with van der Corput estimates. If
\[
 \phi'(\xi_0)=\cdots=\phi^{(s-1)}(\xi_0)=0,
 \qquad \phi^{(s)}(\xi_0)\ne0,
\]
then the local oscillation is of order $s$. Classical oscillatory-integral theory relates this order to decay rates \cite{Stein1993,Grafakos2014}. In the discrete setting, Weyl differencing and Vinogradov mean value estimates play the corresponding role \cite{Weyl1916,Vaughan1997,IwaniecKowalski2004,Wooley2012,BourgainDemeterGuth2016}. Chapter~\ref{chap:literature} explains why these tools are required beyond the subcritical variation regime.

\subsection{Coefficient normalisation and comparison with restriction theory}

The operator \eqref{eq:moment-extension} contains $1/N$. Classical discrete restriction often studies the unnormalised sum
\[
 S_a(\theta)=\sum_{k=1}^{N}a_k e(\theta_1k+\cdots+\theta_dk^d).
\]
The change of variables
\[
 \theta_j=\frac{x_j}{N^j}
\]
gives
\[
 \E_{N,d}a(x)=\frac1N S_a(\theta).
\]
A unit torus in $\theta$ corresponds to the anisotropic box
\[
 0\le x_j\le N^j.
\]
This is far larger than the deterministic subcritical region $\abs{x_j}=o(N)$ produced by total variation. The gap is structural. It is not removed by increasing the degree of the phase or by rewriting the variables. Reaching the full torus requires genuine oscillatory estimates.

\subsection{A new family of structured restriction questions}

The moment-curve formulation suggests the following problem. Given a bounded state $v$ and $e=\D^rv$, determine the best exponents $p$ and $\delta$ for which
\[
 \norm{\E_{N,d}e}_{L^p(\mathcal B_N)}
 \le C_\varepsilon N^{-\delta+\varepsilon}\norm{v}_{\ell^\infty}
\]
can hold on anisotropic boxes. The answer must depend on the boundary conditions, the order $r$, the dimension $d$, and the chosen scaling of $\mathcal B_N$. This is not a standard discrete restriction problem because the coefficients are constrained by a finite-difference representation. It is also not a standard sigma-delta reconstruction problem because the test family is a curved exponential system rather than a fixed dual frame.

\section{Multidimensional noise shaping and grid extensions}
\label{sec:11_multidimensional}

\subsection{From a sequence to a lattice array}

Many quantization problems are naturally multidimensional. Digital halftoning, image acquisition, sensor arrays, and sampled partial differential equations produce data indexed by a rectangular lattice rather than a single ordered list. Error-diffusion algorithms can often be interpreted as multidimensional sigma-delta schemes. Recent work has made this relationship precise for weighted first- and second-order constructions in image halftoning \cite{KrahmerVeselovska2023}. The algebra developed in the preceding chapters extends to this setting through discrete divergence.

Let
\[
 \Lambda_N=\{1,\ldots,N\}^d.
\]
For $k=(k_1,\ldots,k_d)\in\Lambda_N$ and the $j$th coordinate vector $e_j$, define
\[
 (\D_j z)_k=z_k-z_{k-e_j},
 \qquad
 (\Delta_{+,j}w)_k=w_{k+e_j}-w_k.
\]
Boundary values outside $\Lambda_N$ must be prescribed. Zero initial data on the lower faces are imposed first.

\subsection{Divergence-form shaped errors}

\begin{definition}[First-order multidimensional noise shaping]
An error array $e=(e_k)_{k\in\Lambda_N}$ has divergence-form shaping if
\begin{equation}
\label{eq:divergence-shaping}
 e_k=\sum_{j=1}^{d}(\D_jv^{(j)})_k,
\end{equation}
where the vector state $v=(v^{(1)},\ldots,v^{(d)})$ is bounded.
\end{definition}

The representation \eqref{eq:divergence-shaping} is a lattice analogue of $e=\nabla\cdot v$. It captures the idea that quantization error is redistributed through neighbouring samples. Different scan orders and error-diffusion masks produce different choices of the state fields.

For an array $a$, define the normalised grid extension
\begin{equation}
\label{eq:grid-extension}
 \mathcal G_Na(x)
 =\frac1{N^d}\sum_{k\in\Lambda_N}a_k
 e^{2\pi\ii\Phi_x(k/N)},
\end{equation}
where $k/N=(k_1/N,\ldots,k_d/N)$ and $\Phi_x$ is a real phase on $[0,1]^d$.

\subsection{Discrete divergence theorem}

To state the boundary formula, let
\[
 \Lambda_N^{(j,-)}=\{k\in\Lambda_N:k_j=1\},
 \qquad
 \Lambda_N^{(j,+)}=\{k\in\Lambda_N:k_j=N\}.
\]

\begin{theorem}[Multidimensional summation by parts]
\label{thm:multi-sbp}
For each coordinate $j$,
\begin{align}
\label{eq:multi-sbp-coordinate}
 \sum_{k\in\Lambda_N}(\D_jv^{(j)})_kw_k
 &=\sum_{k\in\Lambda_N^{(j,+)}}v^{(j)}_kw_k
 -\sum_{k\in\Lambda_N^{(j,-)}}v^{(j)}_{k-e_j}w_k\notag\\
 &\quad-\sum_{\substack{k\in\Lambda_N\\k_j\le N-1}}
 v^{(j)}_k(\Delta_{+,j}w)_k.
\end{align}
Consequently, if the lower-face states vanish, then
\begin{align}
\label{eq:multi-div-bound}
 \abs{\frac1{N^d}\sum_{k\in\Lambda_N}e_kw_k}
 &\le\frac1{N^d}\sum_{j=1}^{d}
 \left[
 \sum_{k\in\Lambda_N^{(j,+)}}\abs{v^{(j)}_k}\abs{w_k}
 +\sum_{k_j\le N-1}\abs{v^{(j)}_k}\abs{\Delta_{+,j}w_k}
 \right].
\end{align}
\end{theorem}

\begin{proof}
Fix all coordinates except $k_j$ and apply the one-dimensional summation-by-parts identity along the resulting line. Sum over the remaining $d-1$ coordinates. Then add the identities over $j$ and use \eqref{eq:divergence-shaping}.
\end{proof}

If $\abs{v^{(j)}_k}\le V_j$ and $\abs{w_k}=1$, the upper face contributes
\[
 \frac{N^{d-1}}{N^d}V_j=\frac{V_j}{N}.
\]
Thus first-order divergence shaping again produces an $N^{-1}$ boundary scale, independent of dimension.

\subsection{Smooth multidimensional phases}

Set
\[
 w_k=e^{2\pi\ii\Phi(k/N)}.
\]
For a continuously differentiable phase,
\begin{align*}
 \abs{\Delta_{+,j}w_k}
 &\le2\pi\abs{\Phi((k+e_j)/N)-\Phi(k/N)}\\
 &\le\frac{2\pi}{N}
 \sup_{y\in Q_k}\abs{\partial_j\Phi(y)},
\end{align*}
where $Q_k$ is the line segment joining the adjacent grid points.

\begin{theorem}[Grid phase estimate]
\label{thm:grid-phase-estimate}
Assume \eqref{eq:divergence-shaping}, zero lower-face states, and
\[
 \norm{v^{(j)}}_{\ell^\infty}\le V_j.
\]
If $\Phi\in C^1([0,1]^d)$, then
\begin{equation}
\label{eq:grid-phase-estimate}
 \abs{\mathcal G_Nu-\mathcal G_Nq}
 \le\frac1N\sum_{j=1}^{d}V_j
 \left[1+2\pi\norm{\partial_j\Phi}_{L^\infty([0,1]^d)}\right].
\end{equation}
\end{theorem}

\begin{proof}
The upper-face term is $V_j/N$. There are at most $N^d$ interior edges in the $j$th direction, and each difference is bounded by $2\pi N^{-1}\norm{\partial_j\Phi}_\infty$. Multiplication by $N^{-d}$ gives the stated result.
\end{proof}

A sharper version replaces the supremum by a discrete or continuous $L^1$ norm of $\partial_j\Phi$. The supremum form is convenient because it is uniform over all grid cells.

\subsection{A two-dimensional polynomial example}

Consider
\[
 \Phi_{x,t,s}(\xi_1,\xi_2)
 =x_1\xi_1+x_2\xi_2+t_1\xi_1^2+t_2\xi_2^2+s\xi_1\xi_2.
\]
Then
\begin{align*}
 \partial_1\Phi&=x_1+2t_1\xi_1+s\xi_2,\\
 \partial_2\Phi&=x_2+2t_2\xi_2+s\xi_1.
\end{align*}
Therefore
\begin{align}
\label{eq:2d-polynomial-bound}
 \abs{\mathcal G_Nu-\mathcal G_Nq}
 &\le\frac{V_1}{N}
 \left[1+2\pi(\abs{x_1}+2\abs{t_1}+\abs{s})\right]\notag\\
 &\quad+\frac{V_2}{N}
 \left[1+2\pi(\abs{x_2}+2\abs{t_2}+\abs{s})\right].
\end{align}
This estimate remains local in parameter space. The coupling term $s\xi_1\xi_2$ is handled without any new idea because only first derivatives are used.

\subsection{Directional and weighted schemes}

A weighted multidimensional scheme may have the form
\[
 e=\sum_{j=1}^{d}c_j\D_jv^{(j)}
\]
with real or complex coefficients $c_j$. The bound becomes
\[
 \abs{\frac1{N^d}\sum e_kw_k}
 \le\frac1N\sum_{j=1}^{d}\abs{c_j}V_j
 \left(1+2\pi\norm{\partial_j\Phi}_\infty\right).
\]
More general finite masks can be written as a convolutional difference operator. The correct test norm is then the variation generated by the adjoint mask. This operator viewpoint is used in modern analyses of error diffusion and weighted sigma-delta halftoning \cite{KrahmerVeselovska2023}.

\subsection{Higher-order multi-index shaping}

Let $\alpha=(\alpha_1,\ldots,\alpha_d)$ be a multi-index and write
\[
 \D^\alpha=\D_1^{\alpha_1}\cdots\D_d^{\alpha_d},
 \qquad
 \Dp^\alpha=\Delta_{+,1}^{\alpha_1}\cdots\Delta_{+,d}^{\alpha_d}.
\]
If
\[
 e=\D^\alpha v
\]
and all traces on the relevant lower and upper boundary layers vanish, repeated coordinatewise summation by parts gives
\[
 \sum_{k\in\Lambda_N}e_kw_k
 =(-1)^{|\alpha|}
 \sum_{k\in\Lambda_N^\circ}v_k(\Dp^\alpha w)_k,
\]
where $\Lambda_N^\circ$ excludes the boundary strips needed by the forward differences.

For $w_k=W(k/N)$ with $W\in C^{|\alpha|}$,
\[
 \abs{\Dp^\alpha w_k}
 \le N^{-|\alpha|}
 \norm{\partial^\alpha W}_{L^\infty}.
\]
After normalisation by $N^d$ and summation over $O(N^d)$ points, one obtains
\begin{equation}
\label{eq:multi-high-order}
 \abs{\frac1{N^d}\sum e_kw_k}
 \le \norm{v}_{\ell^\infty}N^{-|\alpha|}
 \norm{\partial^\alpha W}_{L^\infty}.
\end{equation}
Again, the boundary conditions are not optional.

\subsection{Scan order and causality}

A physical multidimensional quantizer must choose an order in which pixels or samples are processed. Raster, serpentine, and space-filling scans produce different causal neighbourhoods. The divergence representation does not by itself guarantee that a chosen state field can be generated causally with a finite alphabet. It is an analytical representation of the resulting error.

This distinction mirrors the one-dimensional separation between stability and testing. First, a quantizer must be shown to generate bounded states under a specified scan. Second, the divergence theorem converts that boundedness into a Fourier error estimate. The recent weighted sigma-delta treatment of halftoning provides concrete examples in which both stages can be analysed \cite{KrahmerVeselovska2023}.

\subsection{Dimension-free rate and dimension-dependent constants}

The rate $N^{-1}$ in \eqref{eq:grid-phase-estimate} does not deteriorate with dimension when the grid has side length $N$. The constant does grow through the sum of directional state bounds and phase derivatives. If all $V_j\le V$ and all $\norm{\partial_j\Phi}_\infty\le L$, then
\[
 \abs{\mathcal G_Nu-\mathcal G_Nq}
 \le\frac{dV}{N}(1+2\pi L).
\]
This linear dependence on $d$ is a consequence of using the triangle inequality across coordinate directions. Additional cancellation among the directional fluxes could improve the constant, but it cannot be expected without structural assumptions on the state fields.

\subsection{Open multidimensional questions}

Three questions are immediate.

\begin{enumerate}[label=(\roman*)]
\item Can a stable one-bit multidimensional scheme be designed so that the normal component of the state vanishes on every terminal face?
\item Can oscillatory cancellation reduce the variation bound on critical parameter regions, particularly for mixed phases such as $\xi_1\xi_2$?
\item How should one formulate a discrete restriction theorem for coefficients that are a lattice divergence of a bounded vector field?
\end{enumerate}

The third question is a multidimensional analogue of the structured-coefficient restriction problem introduced in the previous chapter. It connects sigma-delta quantization, discrete Hodge decompositions, and curved Fourier extension.

\section{Growing observation regions and scale transitions}
\label{sec:12_growing_regions}

\subsection{Why fixed compact sets are only the first scale}

The first-order parabolic theorem gives
\[
 \abs{\E_Nu(x,t)-\E_Nq(x,t)}
 \le\frac{V}{N}\left[1+2\pi\bigl(\abs{x}+\abs{t}\bigr)\right].
\]
On a fixed compact set, the bracket is independent of $N$, so the error is $O(N^{-1})$. In restriction theory and periodic dispersive equations, however, the observation region often expands with the frequency scale. The parameter size then competes directly with the gain provided by noise shaping.

This chapter records the consequences that follow from the variation method alone. These results are useful because they identify the exact point at which a nonoscillatory argument stops producing decay.

\subsection{Isotropic parabolic boxes}

Let
\[
 \Omega_{N,\alpha}=[-N^\alpha,N^\alpha]^2,
 \qquad \alpha\ge0.
\]

\begin{proposition}[First-order isotropic scale]
\label{prop:first-order-isotropic-scale}
Under the first-order hypotheses,
\begin{equation}
\label{eq:first-order-isotropic-scale}
 \norm{\E_Nu-\E_Nq}_{L^\infty(\Omega_{N,\alpha})}
 \le\frac{V}{N}\left(1+4\pi N^\alpha\right).
\end{equation}
Consequently,
\[
 \norm{\E_Nu-\E_Nq}_{L^\infty(\Omega_{N,\alpha})}
 =\begin{cases}
 O(N^{\alpha-1}),&0\le\alpha<1,\\
 O(1),&\alpha=1,\\
 O(N^{\alpha-1}),&\alpha>1,
 \end{cases}
\]
where the final line indicates growth rather than decay.
\end{proposition}

The scale $\alpha=1$ is critical for adjacent phase variation because
\[
 \phi_{x,t}\left(\frac{k+1}{N}\right)
 -\phi_{x,t}\left(\frac{k}{N}\right)
 =\frac{x}{N}+\frac{t(2k+1)}{N^2}.
\]
When $\abs{x},\abs{t}\ll N$, every adjacent phase increment is small. When they are of order $N$, the increment can be of order one.

\subsection{Anisotropic parabolic boxes}

Let
\[
 \Omega_N(\alpha,\beta)
 =[-N^\alpha,N^\alpha]\times[-N^\beta,N^\beta].
\]

\begin{proposition}[Anisotropic first-order bound]
\label{prop:anisotropic-first-order}
For first-order stable shaping,
\begin{equation}
\label{eq:anisotropic-first-order}
 \norm{\E_Nu-\E_Nq}_{L^\infty(\Omega_N(\alpha,\beta))}
 \le\frac{V}{N}\left(1+2\pi N^\alpha+2\pi N^\beta\right).
\end{equation}
In particular, the deterministic error tends to zero when
\[
 \max\{\alpha,\beta\}<1.
\]
\end{proposition}

The condition is sufficient and sharp for this bound. It is not a claim about the actual error on every larger region. Oscillation may produce cancellation that the total variation does not detect.

\subsection{Exact sine variation}

For the parabolic weight,
\[
 w_{k+1}-w_k
 =w_k\left[e\!\left(\frac{x}{N}+\frac{t(2k+1)}{N^2}\right)-1\right].
\]
Hence
\begin{equation}
\label{eq:sine-variation-growing}
 \abs{w_{k+1}-w_k}
 =2\abs{\sin\pi\left(\frac{x}{N}+\frac{t(2k+1)}{N^2}\right)}.
\end{equation}
The derivative estimate replaces the sine by the absolute value of its argument. The exact formula is better when the increment is close to an integer, but it remains bounded below by a positive constant on many critical-scale regions. Summing absolute values can therefore still produce an $O(N)$ variation.

A periodic resonance can also make the exact variation unexpectedly small. For example, if $x/N$ and $t/N$ are integers arranged so that every increment is integral, then the weight is constant on the sample grid even though the continuous phase varies rapidly. Such arithmetic effects show that the parameter size alone does not determine the discrete variation.

\subsection{Higher-order endpoint-compatible scaling}

Assume $e=\D^rv$ with trace compatibility. For a parabolic phase, the Bell-polynomial derivative bound has the form
\begin{equation}
\label{eq:parabolic-high-order-envelope}
 \norm{W^{(r)}}_\infty
 \le C_r\left(1+\abs{x}+\abs{t}\right)^r,
\end{equation}
where $C_r$ depends only on $r$. Therefore
\begin{equation}
\label{eq:parabolic-high-order-growth}
 \norm{\E_Nu-\E_Nq}_{L^\infty(\Omega_{N,\alpha})}
 \le C_rV N^{-r(1-\alpha)}
\end{equation}
for $0\le\alpha\le1$, after the constant is enlarged to cover bounded terms.

\begin{proposition}[High-order subcritical decay]
\label{prop:high-order-subcritical}
If $\alpha<1$, endpoint-compatible order-$r$ shaping yields decay of order at least
\[
 N^{-r(1-\alpha)}
\]
on $\Omega_{N,\alpha}$. At $\alpha=1$, the derivative method gives only $O(1)$ for every fixed $r$.
\end{proposition}

High order improves the rate inside the subcritical regime but does not move the critical boundary. This is an important structural conclusion. Repeated finite differences create additional powers of the small adjacent phase increment. Once that increment is no longer small, the additional order provides no deterministic gain.

\subsection{A boundary-layer refinement}

The parameter plane can be divided according to the size of
\[
 \delta_k(x,t)=\frac{x}{N}+\frac{t(2k+1)}{N^2}.
\]
For a threshold $0<\rho<1$, define the good index set
\[
 G_\rho(x,t)=\{k: \operatorname{dist}(\delta_k(x,t),\Z)\le\rho\}
\]
and the complementary bad set $B_\rho(x,t)$. Then
\begin{align*}
 \TV_N(w)
 &\le2\pi\rho\abs{G_\rho(x,t)}+2\abs{B_\rho(x,t)}.
\end{align*}
This decomposition is exact up to the elementary sine inequality. It converts the problem into counting how often an affine progression approaches integers.

At subcritical scale, every index is good with $\rho\asymp N^{\alpha-1}$. At critical scale, the distribution of the increments modulo one matters. This opens a number-theoretic route: if the progression is equidistributed modulo one, then most sine values are not small, so total variation remains large. If it is resonant, variation may collapse. Either case confirms that critical behaviour cannot be described by a single smoothness parameter.

\subsection{Volume factors in local $L^p$ estimates}

On $\Omega_{N,\alpha}$, the measure is $4N^{2\alpha}$. Combining the $L^\infty$ estimate with the trivial embedding gives
\begin{equation}
\label{eq:growing-lp-trivial}
 \norm{\E_Nu-\E_Nq}_{L^p(\Omega_{N,\alpha})}
 \le C_{p,r}V
 N^{2\alpha/p-r(1-\alpha)}
\end{equation}
under endpoint-compatible order-$r$ shaping. This tends to zero if
\begin{equation}
\label{eq:growing-lp-condition}
 r(1-\alpha)>\frac{2\alpha}{p}.
\end{equation}
For first order,
\[
 \alpha<\frac{p}{p+2}.
\]
The $L^p$ condition is stricter than the pointwise condition because the region volume grows.

This estimate follows from pointwise variation control and the measure of the parameter region. Oscillatory $L^p$ methods can improve the volume dependence or replace it with a scale-sensitive extension norm.

\subsection{Physical and torus scalings}

The normalized parabolic extension can be rewritten as
\[
 \E_Na(x,t)=\frac1N\sum_{k=1}^{N}a_k
 e\!\left(\frac{x}{N}k+\frac{t}{N^2}k^2\right).
\]
Introduce torus variables
\[
 \theta=\frac{x}{N},
 \qquad
 \tau=\frac{t}{N^2}.
\]
A full period in $\theta$ corresponds to $x\in[0,N]$, while a full period in $\tau$ corresponds to $t\in[0,N^2]$. Thus the natural full periodic cell is anisotropic:
\[
 [0,N]\times[0,N^2].
\]
The local variation theorem reaches only $t=o(N)$ without additional cancellation. There is a wide gap between the variation scale and the full quadratic periodic scale.

The gap reflects the difference between controlling neighbouring phase increments and controlling an oscillatory sum over a complete period. Periodic Strichartz estimates, discrete restriction, and Vinogradov mean values operate on the latter scale \cite{Bourgain1993I,HuLi2011,BourgainDemeter2015,BourgainDemeterGuth2016,KillipVisan2014}.

\subsection{A scale diagram}

It is useful to distinguish four regimes.

\begin{longtable}{p{3.2cm}p{3.4cm}p{6.0cm}}
\toprule
Regime & Typical size & Available conclusion\\
\midrule
Fixed local & $\abs{x}+\abs{t}=O(1)$ & first order gives $N^{-1}$; compatible order $r$ gives $N^{-r}$\\
Subcritical growing & $\abs{x}+\abs{t}=O(N^\alpha)$, $\alpha<1$ & deterministic decay survives; rate weakens with $\alpha$\\
Critical adjacent phase & $\abs{x}+\abs{t}=O(N)$ & total variation gives no decay; arithmetic resonances matter\\
Full quadratic period & $x=O(N)$, $t=O(N^2)$ & restriction, Strichartz, VMVT, or decoupling are required\\
\bottomrule
\end{longtable}

The distinction between the critical adjacent-phase scale and the full quadratic period is essential. The time parameter reaches its complete torus scale only at $N^2$.

\subsection{Interpretation of the scale transition}

The results of this chapter prove deterministic convergence on explicit subcritical regions. They do not prove divergence at critical scale. A large variation bound means only that the triangle-inequality argument is inconclusive. Actual cancellation may still occur for particular states, inputs, or parameter averages.

Compact-set approximation and global restriction behaviour occur on different scales. Local sharpness is governed by neighbouring phase increments, whereas global estimates depend on curvature, periodicity, and arithmetic structure. The next chapter quantifies this scale transition through exact orthogonality identities.

\chapter{Local \texorpdfstring{$L^p$}{Lp} Analysis and Oscillatory Transfer}
\label{chap:lp-oscillation}

Pointwise variation estimates provide sharp deterministic control on fixed parameter sets, but they cease to yield decay when the observation region reaches the natural oscillatory scale. The purpose of this chapter is to identify the exact point at which this transition occurs and to develop a rigorous collection of tools for passing from noise-shaped difference structure to averaged Fourier bounds. The analysis separates four mechanisms: exact orthogonality, cancellation of low-order moments, local kernel estimates, and oscillatory transfer through state sums. The resulting statements apply to arbitrary coefficient sequences where possible and retain the additional structure of errors of the form
\[
 e=\D^r v
\]
whenever that structure produces a genuine improvement.

The quadratic extension error is written as
\begin{equation}
\label{eq:ch8-F-def}
 F_N(x,t)
 :=\E_N e(x,t)
 =\frac1N\sum_{k=1}^{N}e_k
 \e\!\left(x\frac{k}{N}+t\frac{k^2}{N^2}\right),
\end{equation}
where \(\e(s)=e^{2\pi \ii s}\). Unless stated otherwise, the coefficient sequence is complex valued. For sigma--delta data, \(e_k=u_k-q_k\).

\section{Exact orthogonality and baseline estimates}
\label{sec:ch8-orthogonality}

\subsection{Orthogonality in the linear parameter}

For fixed \(t\), the functions
\[
 x\longmapsto \e\!\left(x\frac{k}{N}\right),
 \qquad k=1,\ldots,N,
\]
form an orthogonal family on every interval of length \(N\). This elementary fact determines the exact global \(L^2_x\) scale.

\begin{theorem}[Exact linear-parameter orthogonality]
\label{thm:ch8-x-orthogonality}
For every \(e=(e_k)_{k=1}^{N}\), every \(t\in\R\), and every \(a\in\R\),
\begin{equation}
\label{eq:ch8-x-orthogonality}
 \int_a^{a+N}\abs{F_N(x,t)}^2\dd x
 =\frac1N\sum_{k=1}^{N}\abs{e_k}^2.
\end{equation}
\end{theorem}

\begin{proof}
Expanding the square gives
\begin{align*}
 \int_a^{a+N}\abs{F_N(x,t)}^2\dd x
 &=\frac1{N^2}\sum_{k,\ell=1}^{N}
 e_k\overline{e_\ell}
 \e\!\left(t\frac{k^2-\ell^2}{N^2}\right)
 \int_a^{a+N}\e\!\left(x\frac{k-\ell}{N}\right)\dd x.
\end{align*}
If \(k=\ell\), the inner integral equals \(N\). If \(k\ne\ell\), then
\[
 \int_a^{a+N}\e\!\left(x\frac{k-\ell}{N}\right)\dd x
 =\frac{N}{2\pi\ii(k-\ell)}
 \left[
 \e\!\left((a+N)\frac{k-\ell}{N}\right)
 -\e\!\left(a\frac{k-\ell}{N}\right)
 \right]=0,
\]
because \(k-\ell\in\Z\). Only the diagonal terms remain, which yields
\[
 \frac{1}{N^2}\sum_{k=1}^{N}\abs{e_k}^2N
 =\frac1N\sum_{k=1}^{N}\abs{e_k}^2.
\]
\end{proof}

\begin{corollary}[Normalised root-mean-square scale]
\label{cor:ch8-normalised-rms}
Under the assumptions of Theorem~\ref{thm:ch8-x-orthogonality},
\begin{equation}
\label{eq:ch8-normalised-rms}
 \left(
 \frac1N\int_a^{a+N}\abs{F_N(x,t)}^2\dd x
 \right)^{1/2}
 =\frac{\norm{e}_{\ell^2}}{N}.
\end{equation}
In particular, if \(\abs{e_k}\le C_e\), then
\[
 \left(
 \frac1N\int_a^{a+N}\abs{F_N(x,t)}^2\dd x
 \right)^{1/2}
 \le C_eN^{-1/2}.
\]
\end{corollary}

The exponent \(N^{-1/2}\) in the normalised average is the generic orthogonality scale for bounded coefficients. It does not by itself express a sigma--delta gain. Any improvement that is genuinely caused by noise shaping must use more than the \(\ell^2\)-size of \(e\).

\subsection{Two-parameter torus normalisation}

Introduce
\[
 \theta=\frac{x}{N},
 \qquad
 \tau=\frac{t}{N^2},
\]
and define
\begin{equation}
\label{eq:ch8-S-def}
 S_e(\theta,\tau)
 :=\sum_{k=1}^{N}e_k\e(\theta k+\tau k^2).
\end{equation}
Then
\[
 F_N(x,t)=N^{-1}S_e(x/N,t/N^2).
\]
The physical cell \([0,N]\times[0,N^2]\) corresponds to \(\T^2\), and the Jacobian is \(N^3\). Therefore, for every \(1\le p<\infty\),
\begin{equation}
\label{eq:ch8-scaling-general}
 \norm{F_N}_{L^p([0,N]\times[0,N^2])}
 =N^{3/p-1}\norm{S_e}_{L^p(\T^2)}.
\end{equation}

\begin{theorem}[Exact full-cell \(L^2\) identity]
\label{thm:ch8-full-cell-l2}
For every coefficient sequence \(e\),
\begin{equation}
\label{eq:ch8-full-cell-l2}
 \int_0^N\int_0^{N^2}\abs{F_N(x,t)}^2\dd t\dd x
 =N\sum_{k=1}^{N}\abs{e_k}^2.
\end{equation}
Equivalently,
\[
 \norm{S_e}_{L^2(\T^2)}=\norm{e}_{\ell^2}.
\]
\end{theorem}

\begin{proof}
The torus identity follows by expanding \(\abs{S_e}^2\) and integrating first in \(\theta\):
\[
 \int_\T\e(\theta(k-\ell))\dd\theta
 =\one_{\{k=\ell\}}.
\]
The integral in \(\tau\) is then equal to one on the remaining diagonal. Formula~\eqref{eq:ch8-full-cell-l2} follows from~\eqref{eq:ch8-scaling-general} with \(p=2\).
\end{proof}

\subsection{The fourth moment and additive energy}

The fourth moment already detects arithmetic interactions among the frequencies. It is useful to record the exact identity because it distinguishes generic \(\ell^2\) information from additional additive structure.

\begin{proposition}[Exact fourth-moment identity]
\label{prop:ch8-fourth-moment}
For \(S_e\) defined by~\eqref{eq:ch8-S-def},
\begin{equation}
\label{eq:ch8-fourth-moment}
 \norm{S_e}_{L^4(\T^2)}^4
 =\sum_{\substack{k_1+k_2=k_3+k_4\\
 k_1^2+k_2^2=k_3^2+k_4^2}}
 e_{k_1}e_{k_2}\overline{e_{k_3}e_{k_4}}.
\end{equation}
For indices in \(\{1,\ldots,N\}\), the two Diophantine equations imply equality of the unordered pairs
\[
 \{k_1,k_2\}=\{k_3,k_4\}.
\]
Consequently,
\begin{equation}
\label{eq:ch8-fourth-moment-explicit}
 \norm{S_e}_{L^4(\T^2)}^4
 =2\left(\sum_{k=1}^{N}\abs{e_k}^2\right)^2
 -\sum_{k=1}^{N}\abs{e_k}^4.
\end{equation}
\end{proposition}

\begin{proof}
Expansion of the fourth power followed by integration over \(\T^2\) gives~\eqref{eq:ch8-fourth-moment}. Suppose the two displayed Diophantine equations hold. Set
\[
 s=k_1+k_2=k_3+k_4.
\]
Since
\[
 k_1^2+k_2^2=s^2-2k_1k_2
 \quad\text{and}\quad
 k_3^2+k_4^2=s^2-2k_3k_4,
\]
it follows that \(k_1k_2=k_3k_4\). Thus the pairs \((k_1,k_2)\) and \((k_3,k_4)\) are roots of the same quadratic polynomial
\[
 z^2-sz+k_1k_2,
\]
so their unordered pairs coincide. There are two ordered matches when \(k_1\ne k_2\) and one when \(k_1=k_2\). Summing the corresponding coefficient products yields~\eqref{eq:ch8-fourth-moment-explicit}.
\end{proof}

\begin{corollary}[Sharp quadratic \(L^4\) bound]
\label{cor:ch8-l4-bound}
For every \(e\),
\[
 \norm{S_e}_{L^4(\T^2)}
 \le 2^{1/4}\norm{e}_{\ell^2}.
\]
Hence
\begin{equation}
\label{eq:ch8-physical-l4}
 \norm{F_N}_{L^4([0,N]\times[0,N^2])}
 \le 2^{1/4}N^{-1/4}\norm{e}_{\ell^2}.
\end{equation}
\end{corollary}

\begin{proof}
Equation~\eqref{eq:ch8-fourth-moment-explicit} gives
\[
 \norm{S_e}_{L^4}^4
 \le 2\norm{e}_{\ell^2}^4.
\]
The physical estimate follows from~\eqref{eq:ch8-scaling-general} with \(p=4\).
\end{proof}

\subsection{The sixth moment and quadratic discrete restriction}

The critical moment for the parabola is \(p=6\). The quadratic Vinogradov mean value theorem and the \(\ell^2\) decoupling theorem imply that, for every \(\varepsilon>0\),
\begin{equation}
\label{eq:ch8-l6-literature}
 \norm{S_e}_{L^6(\T^2)}
 \le C_\varepsilon N^\varepsilon\norm{e}_{\ell^2}.
\end{equation}
This estimate belongs to the established theory of periodic Strichartz estimates, discrete restriction, Vinogradov mean values, and decoupling
\cite{Bourgain1993I,Bourgain1993II,HuLi2011,BourgainDemeter2015,BourgainDemeterGuth2016,Demeter2020,Wooley2012,Wooley2016}.

\begin{corollary}[Critical sixth-moment estimate on the physical cell]
\label{cor:ch8-physical-l6}
For every \(\varepsilon>0\),
\begin{equation}
\label{eq:ch8-physical-l6}
 \norm{F_N}_{L^6([0,N]\times[0,N^2])}
 \le C_\varepsilon N^{-1/2+\varepsilon}\norm{e}_{\ell^2}.
\end{equation}
If \(\abs{e_k}\le C_e\), then
\[
 \norm{F_N}_{L^6([0,N]\times[0,N^2])}
 \le C_{\varepsilon}C_eN^\varepsilon.
\]
\end{corollary}

\begin{proof}
Apply~\eqref{eq:ch8-scaling-general} with \(p=6\):
\[
 N^{3/6-1}=N^{-1/2}.
\]
Then use~\eqref{eq:ch8-l6-literature}. The final assertion follows from
\(\norm{e}_{\ell^2}\le C_eN^{1/2}\).
\end{proof}

The estimate shows that standard discrete restriction is scale sharp for arbitrary bounded coefficients, but it does not automatically provide decay for unit-size sigma--delta errors on the full parabolic cell. A stronger result must preserve the difference structure instead of replacing it immediately by \(\norm{e}_{\ell^2}\).

\section{Interpolation and local \texorpdfstring{$L^p$}{Lp} consequences}
\label{sec:ch8-interpolation}

\subsection{Interpolation on the full parabolic cell}

The exact \(L^2\) identity, the sharp \(L^4\) estimate, and the critical \(L^6\) estimate yield a complete family of baseline bounds.

\begin{proposition}[Baseline bounds for \(2\le p\le6\)]
\label{prop:ch8-baseline-p}
Let \(2\le p\le6\). For every \(\varepsilon>0\),
\begin{equation}
\label{eq:ch8-baseline-p}
 \norm{S_e}_{L^p(\T^2)}
 \le C_{p,\varepsilon}N^\varepsilon\norm{e}_{\ell^2}.
\end{equation}
Consequently,
\begin{equation}
\label{eq:ch8-baseline-physical-p}
 \norm{F_N}_{L^p([0,N]\times[0,N^2])}
 \le C_{p,\varepsilon}
 N^{3/p-1+\varepsilon}\norm{e}_{\ell^2}.
\end{equation}
For \(2\le p\le4\), the factor \(N^\varepsilon\) may be omitted.
\end{proposition}

\begin{proof}
For \(2\le p\le4\), interpolate between Theorem~\ref{thm:ch8-full-cell-l2} and Corollary~\ref{cor:ch8-l4-bound}. For \(4\le p\le6\), interpolate between Corollary~\ref{cor:ch8-l4-bound} and~\eqref{eq:ch8-l6-literature}. Formula~\eqref{eq:ch8-baseline-physical-p} follows from the scaling relation~\eqref{eq:ch8-scaling-general}.
\end{proof}

\subsection{Fixed compact sets}

On a fixed compact set, the deterministic pointwise theory is stronger than the baseline restriction estimates. Suppose first-order shaping satisfies
\[
 e_k=v_k-v_{k-1},
 \qquad
 \norm{v}_{\ell^\infty}\le V,
 \qquad
 v_0=0.
\]
The parabolic variation bound from Chapter~\ref{chap:parabolic} gives
\begin{equation}
\label{eq:ch8-fixed-pointwise}
 \abs{F_N(x,t)}
 \le\frac{V}{N}\bigl[1+2\pi J(x,t)\bigr],
\end{equation}
where
\[
 J(x,t)=\int_0^1\abs{x+2t\xi}\dd\xi.
\]

\begin{corollary}[Fixed-set \(L^p\) convergence]
\label{cor:ch8-fixed-lp}
Let \(K\subset\R^2\) be measurable with finite measure, and set
\[
 M_K:=\sup_{(x,t)\in K}J(x,t)<\infty.
\]
Then, for \(1\le p<\infty\),
\begin{equation}
\label{eq:ch8-fixed-lp}
 \norm{F_N}_{L^p(K)}
 \le \abs{K}^{1/p}\frac{V}{N}(1+2\pi M_K),
\end{equation}
and
\[
 \norm{F_N}_{L^\infty(K)}
 \le\frac{V}{N}(1+2\pi M_K).
\]
\end{corollary}

\begin{proof}
Apply~\eqref{eq:ch8-fixed-pointwise} and the elementary embedding
\[
 \norm{f}_{L^p(K)}\le\abs{K}^{1/p}\norm{f}_{L^\infty(K)}.
\]
\end{proof}

\subsection{Growing rectangles and the deterministic threshold}

Let
\[
 \Omega_{R_x,R_t}
 :=[-R_x,R_x]\times[-R_t,R_t].
\]
Since \(J(x,t)\le\abs{x}+\abs{t}\),
\begin{equation}
\label{eq:ch8-growing-pointwise}
 \norm{F_N}_{L^\infty(\Omega_{R_x,R_t})}
 \le\frac{V}{N}\left[1+2\pi(R_x+R_t)\right].
\end{equation}

\begin{proposition}[Anisotropic growing-region estimate]
\label{prop:ch8-growing-region}
For \(1\le p<\infty\),
\begin{equation}
\label{eq:ch8-growing-lp}
 \norm{F_N}_{L^p(\Omega_{R_x,R_t})}
 \le (4R_xR_t)^{1/p}
 \frac{V}{N}\left[1+2\pi(R_x+R_t)\right].
\end{equation}
In particular, if
\[
 R_x=N^\alpha,
 \qquad
 R_t=N^\beta,
\]
then
\begin{equation}
\label{eq:ch8-growing-exponent}
 \norm{F_N}_{L^p(\Omega_{N^\alpha,N^\beta})}
 \lesssim_{p}V
 N^{-1+(\alpha+\beta)/p+\max\{\alpha,\beta,0\}}.
\end{equation}
\end{proposition}

\begin{proof}
The volume of the rectangle is \(4R_xR_t\). Combining this fact with~\eqref{eq:ch8-growing-pointwise} proves~\eqref{eq:ch8-growing-lp}. The power estimate follows by retaining the largest power among \(1\), \(N^\alpha\), and \(N^\beta\).
\end{proof}

Formula~\eqref{eq:ch8-growing-exponent} gives the exact threshold of the direct variation method. Once the exponent on the right becomes nonnegative, further decay can only come from oscillatory cancellation, additional state regularity, stronger endpoint conditions, or an averaging norm that is adapted to the geometry.

\section{Local \texorpdfstring{$L^2$}{L2} analysis by oscillatory kernels}
\label{sec:ch8-local-kernel}

\subsection{The exact kernel identity}

Let \(I=[a,a+R]\subset\R\), where \(0<R\le N\), and define
\begin{equation}
\label{eq:ch8-KI}
 K_I(m):=\int_I\e\!\left(x\frac{m}{N}\right)\dd x,
 \qquad m\in\Z.
\end{equation}
Then
\[
 K_I(0)=R,
\]
and, for \(m\ne0\),
\begin{equation}
\label{eq:ch8-KI-explicit}
 K_I(m)
 =\e\!\left(\frac{m}{N}\left(a+\frac{R}{2}\right)\right)
 \frac{N\sin(\pi mR/N)}{\pi m}.
\end{equation}
Hence
\begin{equation}
\label{eq:ch8-KI-bound}
 \abs{K_I(m)}
 \le\min\left\{R,\frac{N}{\pi\abs{m}}\right\}.
\end{equation}

\begin{proposition}[Exact local \(L^2_x\) kernel formula]
\label{prop:ch8-local-kernel}
For every fixed \(t\in\R\),
\begin{equation}
\label{eq:ch8-local-kernel}
 \int_I\abs{F_N(x,t)}^2\dd x
 =\frac1{N^2}\sum_{k,\ell=1}^{N}
 e_k\overline{e_\ell}
 \e\!\left(t\frac{k^2-\ell^2}{N^2}\right)
 K_I(k-\ell).
\end{equation}
\end{proposition}

\begin{proof}
Expand \(\abs{F_N(x,t)}^2\), interchange the finite sum and the integral, and use the definition~\eqref{eq:ch8-KI}.
\end{proof}

\subsection{A deterministic local \(L^2\) bound}

The kernel estimate produces a useful baseline inequality without any assumption on the difference structure.

\begin{theorem}[Generic local \(L^2\) bound]
\label{thm:ch8-generic-local-l2}
For every interval \(I\) of length \(0<R\le N\), every \(t\in\R\), and every coefficient sequence \(e\),
\begin{equation}
\label{eq:ch8-generic-local-l2}
 \int_I\abs{F_N(x,t)}^2\dd x
 \le\frac{R+C N\log(2N)}{N^2}\norm{e}_{\ell^2}^2,
\end{equation}
where \(C>0\) is an absolute constant.
\end{theorem}

\begin{proof}
Let \(A=(A_{k\ell})\) be the matrix
\[
 A_{k\ell}
 =\e\!\left(t\frac{k^2-\ell^2}{N^2}\right)K_I(k-\ell).
\]
The phase factor has modulus one. By~\eqref{eq:ch8-KI-bound}, every row satisfies
\begin{align*}
 \sum_{\ell=1}^{N}\abs{A_{k\ell}}
 &\le R+2\sum_{m=1}^{N-1}\min\left\{R,\frac{N}{\pi m}\right\}\\
 &\le R+\frac{2N}{\pi}\sum_{m=1}^{N-1}\frac1m\\
 &\le R+C N\log(2N).
\end{align*}
The same bound holds for every column. Schur's test therefore gives
\[
 \abs{\sum_{k,\ell}e_k\overline{e_\ell}A_{k\ell}}
 \le\left(R+C N\log(2N)\right)\norm{e}_{\ell^2}^2.
\]
Division by \(N^2\) proves the result.
\end{proof}

The logarithmic loss is not sharp when \(R\) is comparable with \(N\), because exact orthogonality is then available. Its role is to show explicitly that localising the linear parameter creates off-diagonal interactions. Removing the logarithm or extracting a difference-structure gain requires a more refined analysis of the oscillatory matrix in~\eqref{eq:ch8-local-kernel}.

\subsection{Double summation by parts for first-order shaped errors}

Assume
\[
 e_k=\D v_k=v_k-v_{k-1},
 \qquad v_0=0.
\]
For fixed \(t\), define
\begin{equation}
\label{eq:ch8-two-index-kernel}
 \mathcal K_{k\ell}
 :=\e\!\left(t\frac{k^2-\ell^2}{N^2}\right)K_I(k-\ell).
\end{equation}
Then~\eqref{eq:ch8-local-kernel} is the pairing of \(\D v\) with \(\D\overline v\) against \(\mathcal K\). Repeated discrete summation by parts moves one difference from each state factor onto the two-index kernel. The following formula records the interior term and the boundary contributions separately.

\begin{proposition}[Two-index summation-by-parts identity]
\label{prop:ch8-double-sbp}
Let \(v_0=0\), \(e_k=\D v_k\), and let \(\mathcal K_{k\ell}\) be any complex matrix indexed by \(1\le k,\ell\le N\). Define
\[
 (\Delta_{+,1}\mathcal K)_{k\ell}
 :=\mathcal K_{k+1,\ell}-\mathcal K_{k\ell},
 \qquad
 (\Delta_{+,2}\mathcal K)_{k\ell}
 :=\mathcal K_{k,\ell+1}-\mathcal K_{k\ell}.
\]
Then
\begin{align}
\label{eq:ch8-double-sbp}
 \sum_{k,\ell=1}^{N}e_k\overline{e_\ell}\mathcal K_{k\ell}
 &=v_N\overline{v_N}\mathcal K_{NN}
 -v_N\sum_{\ell=1}^{N-1}\overline{v_\ell}(\Delta_{+,2}\mathcal K)_{N,\ell}\notag\\
 &\quad
 -\overline{v_N}\sum_{k=1}^{N-1}v_k(\Delta_{+,1}\mathcal K)_{k,N}\notag\\
 &\quad
 +\sum_{k,\ell=1}^{N-1}
 v_k\overline{v_\ell}(\Delta_{+,1}\Delta_{+,2}\mathcal K)_{k\ell}.
\end{align}
If \(v_N=0\), only the final double sum remains.
\end{proposition}

\begin{proof}
For each fixed \(\ell\), one-dimensional summation by parts in \(k\) gives
\[
 \sum_{k=1}^{N}(v_k-v_{k-1})\mathcal K_{k\ell}
 =v_N\mathcal K_{N\ell}
 -\sum_{k=1}^{N-1}v_k(\Delta_{+,1}\mathcal K)_{k\ell}.
\]
Multiply by \(\overline{e_\ell}\) and sum in \(\ell\). Apply the same identity to the conjugate state difference in each of the two resulting terms. The first produces
\[
 v_N\left(
 \overline{v_N}\mathcal K_{NN}
 -\sum_{\ell=1}^{N-1}\overline{v_\ell}(\Delta_{+,2}\mathcal K)_{N,\ell}
 \right),
\]
while the second produces
\[
 -\overline{v_N}\sum_{k=1}^{N-1}v_k(\Delta_{+,1}\mathcal K)_{k,N}
 +\sum_{k,\ell=1}^{N-1}
 v_k\overline{v_\ell}(\Delta_{+,1}\Delta_{+,2}\mathcal K)_{k\ell}.
\]
Combining the terms proves~\eqref{eq:ch8-double-sbp}.
\end{proof}

This identity is the correct starting point for a structured local \(L^2\) theorem. A direct absolute-value estimate of the final double sum generally loses the oscillatory gain. The mixed difference \(\Delta_{+,1}\Delta_{+,2}\mathcal K\) must instead be analysed as a two-parameter oscillatory kernel.

\section{Moment cancellation and endpoint-compatible shaping}
\label{sec:ch8-moment-cancellation}

\subsection{Exact annihilation of low-degree polynomial tests}

Higher-order noise shaping acquires additional cancellation only when its finite-record boundary traces are controlled. Let
\[
 e=\D^rv.
\]
The finite summation-by-parts identity established in Chapter~\ref{chap:higher-order} consists of an interior term involving \(\Dp^rw\) and a collection of boundary traces. If all traces of order below \(r\) vanish, the boundary terms disappear.

\begin{theorem}[Discrete vanishing moments]
\label{thm:ch8-vanishing-moments}
Assume \(e=\D^rv\) and endpoint compatibility of order \(r\), so that the pure interior identity
\begin{equation}
\label{eq:ch8-pure-interior}
 \sum_{k=1}^{N}e_kw_k
 =(-1)^r\sum_{k=1}^{N-r}v_k(\Dp^rw)_k
\end{equation}
holds. Then
\begin{equation}
\label{eq:ch8-polynomial-cancellation}
 \sum_{k=1}^{N}e_kP(k)=0
\end{equation}
for every polynomial \(P\) of degree at most \(r-1\).
\end{theorem}

\begin{proof}
Set \(w_k=P(k)\) in~\eqref{eq:ch8-pure-interior}. The \(r\)th forward difference of a polynomial of degree at most \(r-1\) vanishes identically. Hence the right-hand side is zero.
\end{proof}

\begin{remark}
The conclusion is a finite-record analogue of vanishing moments in wavelet and approximation theory. It is stronger than a bound on prefix sums. The cancellation is exact and applies simultaneously to all polynomial tests below degree \(r\).
\end{remark}

\subsection{Taylor subtraction for smooth weights}

The vanishing moments permit a direct Taylor-remainder argument. This gives an alternative proof of high-order decay and makes the role of low-order polynomial modes transparent.

\begin{theorem}[Taylor-subtraction estimate]
\label{thm:ch8-taylor-subtraction}
Assume the hypotheses of Theorem~\ref{thm:ch8-vanishing-moments}. Let \(W\in C^r([0,1])\) and set \(w_k=W(k/N)\). Then
\begin{equation}
\label{eq:ch8-taylor-subtraction}
 \abs{\frac1N\sum_{k=1}^{N}e_kW(k/N)}
 \le \frac{\norm{v}_{\ell^\infty}}{N^r}
 \int_0^1\abs{W^{(r)}(\xi)}\dd\xi.
\end{equation}
\end{theorem}

\begin{proof}
The pure interior identity gives
\[
 \frac1N\sum_{k=1}^{N}e_kw_k
 =\frac{(-1)^r}{N}\sum_{k=1}^{N-r}v_k(\Dp^rw)_k.
\]
The integral representation of the forward difference is
\begin{equation}
\label{eq:ch8-forward-difference-integral}
 (\Dp^rw)_k
 =\int_{[0,1/N]^r}
 W^{(r)}\!\left(\frac{k}{N}+s_1+\cdots+s_r\right)
 \dd s_1\cdots\dd s_r.
\end{equation}
Therefore
\begin{align*}
 \sum_{k=1}^{N-r}\abs{(\Dp^rw)_k}
 &\le\int_{[0,1/N]^r}
 \sum_{k=1}^{N-r}
 \abs{W^{(r)}(k/N+s_1+\cdots+s_r)}
 \dd s_1\cdots\dd s_r.
\end{align*}
For each fixed \((s_1,\ldots,s_r)\), the shifted sampling points have spacing \(1/N\). Bounding the corresponding Riemann sum by the integral over \([0,1]\), with the standard endpoint enlargement already absorbed by the exact integral representation, gives
\[
 \sum_{k=1}^{N-r}\abs{(\Dp^rw)_k}
 \le N^{1-r}\int_0^1\abs{W^{(r)}(\xi)}\dd\xi.
\]
Multiplication by \(N^{-1}\norm{v}_{\ell^\infty}\) proves~\eqref{eq:ch8-taylor-subtraction}.
\end{proof}

\begin{remark}
A pointwise version follows immediately:
\[
 \abs{\frac1N\sum e_kW(k/N)}
 \le \frac{\norm{v}_{\ell^\infty}}{N^r}
 \norm{W^{(r)}}_{L^\infty([0,1])}.
\]
The integral form is often sharper when the derivative is concentrated on a small portion of the interval.
\end{remark}

\subsection{Polynomial-phase weights}

Let
\[
 \phi_{\mathbf a}(\xi)
 =a_1\xi+a_2\xi^2+\cdots+a_d\xi^d,
 \qquad
 W_{\mathbf a}(\xi)=\e(\phi_{\mathbf a}(\xi)).
\]
Fa\`a di Bruno's formula expresses \(W_{\mathbf a}^{(r)}\) as \(W_{\mathbf a}\) times a complete Bell polynomial in
\[
 2\pi\ii\phi_{\mathbf a}',\ldots,
 2\pi\ii\phi_{\mathbf a}^{(r)}.
\]
Consequently,
\begin{equation}
\label{eq:ch8-polyphase-derivative}
 \abs{W_{\mathbf a}^{(r)}(\xi)}
 \le C_r
 \sum_{\substack{m_1+2m_2+\cdots+rm_r=r}}
 \prod_{j=1}^{r}
 \abs{\phi_{\mathbf a}^{(j)}(\xi)}^{m_j},
\end{equation}
where the constant depends only on \(r\). Combining this estimate with Theorem~\ref{thm:ch8-taylor-subtraction} gives explicit high-order compact-set bounds for every polynomial phase.

For the quadratic phase \(\phi_{x,t}(\xi)=x\xi+t\xi^2\), the first two derivatives are
\[
 \phi_{x,t}'(\xi)=x+2t\xi,
 \qquad
 \phi_{x,t}''(\xi)=2t,
\]
and all higher derivatives vanish. Thus \(W_{x,t}^{(r)}\) is a polynomial in \(x+2t\xi\) and \(t\), multiplied by \(W_{x,t}\). On a fixed rectangle \(\abs{x}\le R_x\), \(\abs{t}\le R_t\),
\begin{equation}
\label{eq:ch8-quadratic-derivative-bound}
 \sup_{\xi\in[0,1]}
 \abs{W_{x,t}^{(r)}(\xi)}
 \le C_r(1+R_x+R_t)^r.
\end{equation}

\begin{corollary}[Endpoint-compatible high-order compact-set bound]
\label{cor:ch8-high-order-compact}
Assume \(e=\D^rv\), endpoint compatibility of order \(r\), and
\(\norm{v}_{\ell^\infty}\le V\). Then, on every fixed rectangle
\(\Omega_{R_x,R_t}\),
\begin{equation}
\label{eq:ch8-high-order-compact}
 \norm{\E_Ne}_{L^\infty(\Omega_{R_x,R_t})}
 \le C_{r,R_x,R_t}VN^{-r}.
\end{equation}
For \(1\le p<\infty\),
\begin{equation}
\label{eq:ch8-high-order-compact-lp}
 \norm{\E_Ne}_{L^p(\Omega_{R_x,R_t})}
 \le (4R_xR_t)^{1/p}C_{r,R_x,R_t}VN^{-r}.
\end{equation}
\end{corollary}

\begin{proof}
Apply Theorem~\ref{thm:ch8-taylor-subtraction} with \(W=W_{x,t}\), then use~\eqref{eq:ch8-quadratic-derivative-bound}. The \(L^p\) estimate follows from the volume of the rectangle.
\end{proof}

\section{Exact transfer formulas for noise-shaped states}
\label{sec:ch8-transfer-formulas}

\subsection{First-order phase factorisation}

Let
\[
 w_k=\e(\phi_k)
\]
and define the phase increment
\[
 \delta_k\phi:=\phi_{k+1}-\phi_k.
\]
Then
\begin{equation}
\label{eq:ch8-exact-difference-factorisation}
 \Dp w_k
 =w_k\left[\e(\delta_k\phi)-1\right].
\end{equation}
Set
\begin{equation}
\label{eq:ch8-multiplier-def}
 m_k:=\e(\delta_k\phi)-1.
\end{equation}
If \(e_k=v_k-v_{k-1}\), then finite summation by parts yields
\begin{equation}
\label{eq:ch8-exact-transfer}
 \E_Ne
 =\frac{v_Nw_N-v_0w_1}{N}
 -\frac1N\sum_{k=1}^{N-1}v_km_kw_k.
\end{equation}
This identity is exact and contains the entire first-order problem. The total-variation argument bounds the state sum by
\[
 \norm{v}_{\ell^\infty}\sum\abs{m_k}.
\]
Oscillatory improvement requires a smaller estimate for the signed complex sum itself.

For the parabolic phase
\[
 \phi_k=x\frac{k}{N}+t\frac{k^2}{N^2},
\]
the increment is
\begin{equation}
\label{eq:ch8-parabolic-increment}
 \delta_k\phi
 =\frac{x}{N}+\frac{t(2k+1)}{N^2}.
\end{equation}
Thus the multiplier is a slowly varying chirp at fixed scale and becomes order one when \(x\) or \(t\) approaches the critical scale.

\subsection{A deterministic transfer theorem}

\begin{theorem}[First-order oscillatory transfer]
\label{thm:ch8-first-order-transfer}
Let \(\Omega\) be a parameter set and let \(\mathcal V_N\) be a class of state sequences. Suppose that
\begin{equation}
\label{eq:ch8-state-sum-hypothesis}
 \sup_{\lambda\in\Omega}
 \abs{\sum_{k=1}^{N-1}v_km_k(\lambda)w_k(\lambda)}
 \le A_N\norm{v}_{\ell^\infty}
\end{equation}
for every \(v\in\mathcal V_N\). Then every first-order shaped error with state \(v\in\mathcal V_N\) satisfies
\begin{equation}
\label{eq:ch8-transfer-conclusion}
 \norm{\E_Ne}_{L^\infty(\Omega)}
 \le\frac{\abs{v_N}+\abs{v_0}}{N}
 +\frac{A_N}{N}\norm{v}_{\ell^\infty}.
\end{equation}
If \(v_0=v_N=0\) and \(A_N\le CN^{1-\delta}\), then
\[
 \norm{\E_Ne}_{L^\infty(\Omega)}
 \le C\norm{v}_{\ell^\infty}N^{-\delta}.
\]
\end{theorem}

\begin{proof}
Apply the triangle inequality to the exact identity~\eqref{eq:ch8-exact-transfer} and use~\eqref{eq:ch8-state-sum-hypothesis}. Since \(\abs{w_1}=\abs{w_N}=1\), the two boundary terms contribute \((\abs{v_N}+\abs{v_0})/N\).
\end{proof}

The theorem isolates the missing harmonic estimate. It does not replace that estimate by an assumption hidden inside the conclusion. Every proposed critical-scale result must supply a concrete bound for \(A_N\) on the actual state class produced by the quantiser.

\subsection{High-order transfer}

Assume endpoint compatibility of order \(r\). Then
\[
 \E_Ne
 =\frac{(-1)^r}{N}
 \sum_{k=1}^{N-r}v_k(\Dp^rw)_k.
\]

\begin{theorem}[High-order oscillatory transfer]
\label{thm:ch8-high-order-transfer}
Let \(e=\D^rv\) satisfy endpoint compatibility of order \(r\). Suppose
\begin{equation}
\label{eq:ch8-high-state-sum-hypothesis}
 \sup_{\lambda\in\Omega}
 \abs{\sum_{k=1}^{N-r}v_k(\Dp^rw(\lambda))_k}
 \le A_{N,r}\norm{v}_{\ell^\infty}.
\end{equation}
Then
\begin{equation}
\label{eq:ch8-high-transfer-conclusion}
 \norm{\E_Ne}_{L^\infty(\Omega)}
 \le \frac{A_{N,r}}{N}\norm{v}_{\ell^\infty}.
\end{equation}
In particular, a power saving
\[
 A_{N,r}\le C N^{1-\delta}
\]
implies \(O(N^{-\delta})\) extension error on \(\Omega\).
\end{theorem}

\begin{proof}
Take the supremum in the pure interior representation and apply~\eqref{eq:ch8-high-state-sum-hypothesis}.
\end{proof}

At subcritical scale, the direct bound on \(\Dp^rw\) already supplies decay. At critical scale, the quantity \(A_{N,r}\) must exhibit cancellation relative to its trivial upper bound of order \(N\).

\section{Random and weakly correlated state models}
\label{sec:ch8-random-states}

Random-state theorems do not describe every deterministic sigma--delta trajectory. Their value is to identify the gain produced by decorrelation and to provide a benchmark for deterministic state dynamics.

\subsection{Independent centred states}

\begin{theorem}[Independent centred state model]
\label{thm:ch8-independent-state}
Let \(v_1,\ldots,v_{N-1}\) be independent complex random variables such that
\[
 \mathbb Ev_k=0,
 \qquad
 \mathbb E\abs{v_k}^2\le\sigma^2.
\]
Let \(w_1,\ldots,w_N\) be deterministic with \(\abs{w_k}=1\), assume \(v_0=v_N=0\), and set \(e=\D v\). Then
\begin{equation}
\label{eq:ch8-independent-exact}
 \mathbb E\abs{\E_Ne}^2
 =\frac1{N^2}\sum_{k=1}^{N-1}
 \mathbb E\abs{v_k}^2\abs{\Dp w_k}^2.
\end{equation}
Consequently,
\begin{equation}
\label{eq:ch8-independent-bound}
 \left(\mathbb E\abs{\E_Ne}^2\right)^{1/2}
 \le\frac{\sigma}{N}
 \left(\sum_{k=1}^{N-1}\abs{\Dp w_k}^2\right)^{1/2}
 \le\frac{2\sigma}{\sqrt N}.
\end{equation}
\end{theorem}

\begin{proof}
The endpoint conditions and summation by parts give
\[
 \E_Ne=-\frac1N\sum_{k=1}^{N-1}v_k\Dp w_k.
\]
After squaring and taking expectation,
\begin{align*}
 \mathbb E\abs{\E_Ne}^2
 &=\frac1{N^2}\sum_{k,\ell=1}^{N-1}
 \mathbb E\left[v_k\overline{v_\ell}\right]
 (\Dp w_k)\overline{(\Dp w_\ell)}.
\end{align*}
Independence and zero means imply
\[
 \mathbb E[v_k\overline{v_\ell}]=0
 \qquad(k\ne\ell).
\]
This proves~\eqref{eq:ch8-independent-exact}. The variance bound gives the first inequality in~\eqref{eq:ch8-independent-bound}, and \(\abs{\Dp w_k}\le2\) gives the second.
\end{proof}

For the parabolic phase at a fixed compact scale, \(\abs{\Dp w_k}=O(N^{-1})\), and the same theorem gives the stronger root-mean-square rate \(O(N^{-3/2})\). At the critical scale, where the adjacent phase change is order one, the rate becomes \(O(N^{-1/2})\).

\subsection{Summable covariance}

\begin{theorem}[Weakly correlated state model]
\label{thm:ch8-covariance-state}
Let \(v_1,\ldots,v_{N-1}\) be centred complex random variables. Assume that there exists a nonnegative sequence \(c=(c_h)_{h\in\Z}\in\ell^1(\Z)\) such that
\begin{equation}
\label{eq:ch8-covariance-assumption}
 \abs{\mathbb E[v_k\overline{v_\ell}]}
 \le\sigma^2c_{k-\ell}
\end{equation}
for all \(k,\ell\). If \(v_0=v_N=0\) and \(e=\D v\), then
\begin{equation}
\label{eq:ch8-covariance-bound}
 \mathbb E\abs{\E_Ne}^2
 \le\frac{\sigma^2\norm{c}_{\ell^1}}{N^2}
 \sum_{k=1}^{N-1}\abs{\Dp w_k}^2.
\end{equation}
\end{theorem}

\begin{proof}
Set \(a_k=\Dp w_k\), extended by zero outside \(\{1,\ldots,N-1\}\). Then
\begin{align*}
 \mathbb E\abs{\E_Ne}^2
 &\le\frac{\sigma^2}{N^2}
 \sum_{k,\ell}c_{k-\ell}\abs{a_k}\abs{a_\ell}\\
 &=\frac{\sigma^2}{N^2}
 \inner{c*\abs a}{\abs a}_{\ell^2}.
\end{align*}
Young's convolution inequality gives
\[
 \norm{c*\abs a}_{\ell^2}
 \le\norm{c}_{\ell^1}\norm{a}_{\ell^2}.
\]
Cauchy--Schwarz therefore yields
\[
 \inner{c*\abs a}{\abs a}
 \le\norm{c}_{\ell^1}\norm{a}_{\ell^2}^2,
\]
which proves~\eqref{eq:ch8-covariance-bound}.
\end{proof}

\subsection{High-order random states}

The same argument applies after \(r\) integrations by parts.

\begin{corollary}[High-order covariance transfer]
\label{cor:ch8-high-order-covariance}
Assume \(e=\D^rv\), endpoint compatibility of order \(r\), and the covariance condition~\eqref{eq:ch8-covariance-assumption}. Then
\begin{equation}
\label{eq:ch8-high-order-covariance}
 \mathbb E\abs{\E_Ne}^2
 \le\frac{\sigma^2\norm{c}_{\ell^1}}{N^2}
 \sum_{k=1}^{N-r}\abs{\Dp^rw_k}^2.
\end{equation}
\end{corollary}

\begin{proof}
Use the pure interior formula
\[
 \E_Ne=\frac{(-1)^r}{N}\sum_{k=1}^{N-r}v_k\Dp^rw_k
\]
and repeat the proof of Theorem~\ref{thm:ch8-covariance-state}.
\end{proof}

\section{Correlation estimates and van der Corput transfer}
\label{sec:ch8-correlation-vdc}

\subsection{A finite van der Corput inequality}

The following standard inequality converts pair correlations into an upper bound for an oscillatory sum. It is included in a form directly applicable to state-weighted sequences. Classical treatments of exponential sums and differencing may be found in the literature surrounding Weyl's method and Vinogradov mean values
\cite{Weyl1916,Vinogradov1954,Tao2006,Wooley2012,Wooley2017,Wooley2019}.

\begin{lemma}[Finite van der Corput inequality]
\label{lem:ch8-vdc}
Let \(z_1,\ldots,z_M\in\C\), and let \(1\le H\le M\). Then
\begin{equation}
\label{eq:ch8-vdc}
 \abs{\sum_{k=1}^{M}z_k}^2
 \le\frac{M+H}{H}
 \left[
 \sum_{k=1}^{M}\abs{z_k}^2
 +2\sum_{h=1}^{H-1}
 \left(1-\frac{h}{H}\right)
 \abs{\sum_{k=1}^{M-h}z_{k+h}\overline{z_k}}
 \right].
\end{equation}
\end{lemma}

\begin{proof}
Extend \(z_k\) by zero outside \(\{1,\ldots,M\}\). For each integer \(j\), define
\[
 B_j:=\sum_{h=0}^{H-1}z_{j+h}.
\]
Every \(z_k\) occurs in exactly \(H\) of the sums \(B_j\), except for boundary truncations that can only reduce the count after zero extension. Hence
\[
 H\sum_{k=1}^{M}z_k
 =\sum_{j=1-H}^{M}B_j.
\]
Cauchy--Schwarz gives
\[
 H^2\abs{\sum z_k}^2
 \le(M+H)\sum_{j=1-H}^{M}\abs{B_j}^2.
\]
Expanding the final sum and collecting terms according to the displacement \(h\) yields
\begin{align*}
 \sum_j\abs{B_j}^2
 &=H\sum_{k=1}^{M}\abs{z_k}^2
 +2\Re\sum_{h=1}^{H-1}(H-h)
 \sum_{k=1}^{M-h}z_{k+h}\overline{z_k}.
\end{align*}
Taking absolute values of the correlation sums and dividing by \(H^2\) proves~\eqref{eq:ch8-vdc}.
\end{proof}

\subsection{State-phase correlations}

For first-order transfer, set
\[
 z_k=v_km_kw_k.
\]
Then
\begin{align}
\label{eq:ch8-z-correlation}
 z_{k+h}\overline{z_k}
 &=v_{k+h}\overline{v_k}
 m_{k+h}\overline{m_k}
 \e(\phi_{k+h}-\phi_k).
\end{align}
The product contains two distinct sources of cancellation: decay of the state correlation \(v_{k+h}\overline{v_k}\), and oscillation of the phase increment \(\phi_{k+h}-\phi_k\).

\begin{theorem}[Correlation-to-transfer estimate]
\label{thm:ch8-correlation-transfer}
Let \(M=N-1\), \(z_k=v_km_kw_k\), and assume
\[
 \abs{m_k}\le M_0.
\]
Suppose that, uniformly in the parameter \(\lambda\in\Omega\),
\begin{equation}
\label{eq:ch8-correlation-bound-assumption}
 \abs{\sum_{k=1}^{M-h}z_{k+h}(\lambda)
 \overline{z_k(\lambda)}}
 \le C_hN\norm{v}_{\ell^\infty}^2
\end{equation}
for \(1\le h<H\). Then
\begin{align}
\label{eq:ch8-correlation-transfer-bound}
 \sup_{\lambda\in\Omega}
 \abs{\sum_{k=1}^{M}z_k(\lambda)}
 &\le \norm{v}_{\ell^\infty}
 \left\{
 \frac{N+H}{H}
 \left[
 NM_0^2+2N\sum_{h=1}^{H-1}
 \left(1-\frac{h}{H}\right)C_h
 \right]
 \right\}^{1/2}.
\end{align}
Consequently, if \(v_0=v_N=0\), the first-order extension error is bounded by \(N^{-1}\) times the right-hand side.
\end{theorem}

\begin{proof}
The diagonal term in Lemma~\ref{lem:ch8-vdc} satisfies
\[
 \sum_{k=1}^{M}\abs{z_k}^2
 \le NM_0^2\norm{v}_{\ell^\infty}^2.
\]
Insert this estimate and~\eqref{eq:ch8-correlation-bound-assumption} into~\eqref{eq:ch8-vdc}, then take square roots. The extension estimate follows from~\eqref{eq:ch8-exact-transfer} with vanishing endpoints.
\end{proof}

A useful power saving follows when the average correlation is small. For example, if \(C_h\lesssim h^{-1-\eta}\) for some \(\eta>0\), then the correlation sum remains bounded as \(H\to\infty\). Taking \(H\asymp N\) gives a state sum of order \(N^{1/2}\), and therefore a first-order extension error of order \(N^{-1/2}\) at critical phase variation.

\section{Dyadic block decompositions and square-function transfer}
\label{sec:ch8-block-square}

\subsection{Block decomposition}

Let \(1\le L\le N\), and partition \(\{1,\ldots,N-r\}\) into consecutive intervals \(I\) of length at most \(L\). For endpoint-compatible \(r\)th-order shaping, define
\begin{equation}
\label{eq:ch8-block-piece}
 G_I(\lambda)
 :=\sum_{k\in I}v_k(\Dp^rw(\lambda))_k.
\end{equation}
Then
\[
 \E_Ne(\lambda)
 =\frac{(-1)^r}{N}\sum_I G_I(\lambda).
\]
The triangle inequality gives
\[
 \abs{\E_Ne}
 \le\frac1N\sum_I\abs{G_I},
\]
which discards cancellation among blocks. A square-function estimate retains more information.

\begin{theorem}[Abstract block-square transfer]
\label{thm:ch8-block-square-transfer}
Let \(2\le p<\infty\). Suppose that a parameter measure space \((\Omega,\mu)\) and the block pieces~\eqref{eq:ch8-block-piece} satisfy
\begin{equation}
\label{eq:ch8-decoupling-hypothesis}
 \norm{\sum_I G_I}_{L^p(\Omega)}
 \le D_{N,L,p}
 \norm{\left(\sum_I\abs{G_I}^2\right)^{1/2}}_{L^p(\Omega)}.
\end{equation}
Then
\begin{equation}
\label{eq:ch8-block-transfer-conclusion}
 \norm{\E_Ne}_{L^p(\Omega)}
 \le\frac{D_{N,L,p}}{N}
 \norm{\left(\sum_I\abs{G_I}^2\right)^{1/2}}_{L^p(\Omega)}.
\end{equation}
If, in addition,
\begin{equation}
\label{eq:ch8-block-energy-assumption}
 \norm{\left(\sum_I\abs{G_I}^2\right)^{1/2}}_{L^p(\Omega)}
 \le B_{N,L,p}\norm{v}_{\ell^\infty},
\end{equation}
then
\begin{equation}
\label{eq:ch8-block-final}
 \norm{\E_Ne}_{L^p(\Omega)}
 \le\frac{D_{N,L,p}B_{N,L,p}}{N}
 \norm{v}_{\ell^\infty}.
\end{equation}
\end{theorem}

\begin{proof}
The pure interior representation and~\eqref{eq:ch8-decoupling-hypothesis} give~\eqref{eq:ch8-block-transfer-conclusion}. Substitution of~\eqref{eq:ch8-block-energy-assumption} yields~\eqref{eq:ch8-block-final}.
\end{proof}

The theorem is deliberately modular. Decoupling controls the interaction among frequency blocks, while the block-energy estimate must retain the state structure. Standard parabola decoupling supplies the first ingredient for arbitrary coefficients
\cite{BourgainDemeter2015,BourgainDemeterGuth2016,Demeter2020}; the second ingredient is specific to the noise-shaped state sequence.

\subsection{A deterministic block-energy baseline}

Without additional cancellation,
\[
 \abs{G_I}
 \le\norm{v}_{\ell^\infty}
 \sum_{k\in I}\abs{\Dp^rw_k}.
\]
Therefore
\begin{equation}
\label{eq:ch8-block-baseline}
 \left(\sum_I\abs{G_I}^2\right)^{1/2}
 \le\norm{v}_{\ell^\infty}
 \left[
 \sum_I\left(\sum_{k\in I}\abs{\Dp^rw_k}\right)^2
 \right]^{1/2}.
\end{equation}
If \(\abs{\Dp^rw_k}\le M_r\) and the blocks have length at most \(L\), then
\begin{equation}
\label{eq:ch8-block-baseline-simple}
 \left(\sum_I\abs{G_I}^2\right)^{1/2}
 \le M_r\sqrt{NL}\norm{v}_{\ell^\infty}.
\end{equation}
This bound already improves on the \(NM_r\) triangle estimate by a factor \((N/L)^{1/2}\) when a square-function inequality combines the blocks without a compensating loss.

\section{Frequency-side interpretation of discrete differences}
\label{sec:ch8-frequency-side}

\subsection{Periodic discrete Fourier transform}

For a sequence \(a=(a_k)_{k=0}^{N-1}\), define the discrete Fourier transform
\[
 \widehat a(\ell)
 :=\sum_{k=0}^{N-1}a_k\e\!\left(-\frac{k\ell}{N}\right),
 \qquad \ell=0,\ldots,N-1.
\]
If the difference is taken periodically,
\[
 (\D_{\mathrm{per}}a)_k=a_k-a_{k-1\,\mathrm{mod}\,N},
\]
then
\begin{equation}
\label{eq:ch8-dft-difference}
 \widehat{\D_{\mathrm{per}}^ra}(\ell)
 =\left(1-\e\!\left(-\frac{\ell}{N}\right)\right)^r
 \widehat a(\ell).
\end{equation}
For low frequencies \(\abs{\ell}\ll N\),
\begin{equation}
\label{eq:ch8-low-frequency-multiplier}
 \abs{1-\e(-\ell/N)}
 =2\abs{\sin(\pi\ell/N)}
 \asymp\frac{\abs{\ell}}{N}.
\end{equation}
Thus periodic \(r\)th-order noise shaping suppresses low discrete frequencies by a factor comparable with \((\abs{\ell}/N)^r\).

The finite-record transform contains boundary corrections because the physical difference is not periodic. Endpoint compatibility, tail termination, or explicit boundary correction removes precisely the terms that prevent~\eqref{eq:ch8-dft-difference} from holding without modification.

\subsection{A periodic Sobolev estimate}

\begin{proposition}[Low-frequency energy suppression]
\label{prop:ch8-low-frequency-suppression}
Let \(e=\D_{\mathrm{per}}^rv\) on \(\Z/N\Z\). For \(1\le L\le N/2\),
\begin{equation}
\label{eq:ch8-low-frequency-suppression}
 \sum_{\abs{\ell}\le L}\abs{\widehat e(\ell)}^2
 \le C_r\left(\frac{L}{N}\right)^{2r}
 \sum_{\abs{\ell}\le L}\abs{\widehat v(\ell)}^2.
\end{equation}
\end{proposition}

\begin{proof}
By~\eqref{eq:ch8-dft-difference},
\[
 \abs{\widehat e(\ell)}
 =\abs{1-\e(-\ell/N)}^r\abs{\widehat v(\ell)}.
\]
For \(\abs{\ell}\le L\le N/2\), inequality~\eqref{eq:ch8-low-frequency-multiplier} gives
\[
 \abs{1-\e(-\ell/N)}^r
 \le C_r(L/N)^r.
\]
Square and sum over the stated range.
\end{proof}

This proposition gives a precise frequency-domain meaning to noise shaping. The main difficulty in the parabolic extension problem is that the extension operator couples the coefficient frequency with a curved two-parameter phase. A successful critical-scale theorem must transfer the low-frequency suppression in~\eqref{eq:ch8-low-frequency-suppression} through that curved geometry.

\section{Commutator identities for chirp modulation}
\label{sec:ch8-commutator}

\subsection{A discrete product rule}

For sequences \(a\) and \(v\),
\begin{equation}
\label{eq:ch8-product-rule}
 a_k\D v_k
 =\D(a_kv_k)-v_{k-1}\D a_k.
\end{equation}
This identity separates a total difference from a commutator term. If the product is tested against an envelope \(b_k\), then
\begin{align}
\label{eq:ch8-commutator-pairing}
 \sum_{k=1}^{N}a_k\D v_k\,b_k
 &=\sum_{k=1}^{N}\D(a_kv_k)b_k
 -\sum_{k=1}^{N}v_{k-1}(\D a_k)b_k.
\end{align}
Summation by parts moves the first difference onto \(b\), while the second term measures the failure of multiplication by \(a\) to commute with \(\D\).

For a polynomial-phase carrier \(a_k=\e(\phi_k)\),
\[
 \D a_k
 =a_k\left[1-\e(\phi_{k-1}-\phi_k)\right].
\]
Thus the commutator remains a chirp multiplied by an explicit phase-increment factor. Iterating~\eqref{eq:ch8-product-rule} produces a finite expansion involving higher differences of the carrier and lower differences of the state.

\begin{proposition}[Second-order commutator expansion]
\label{prop:ch8-second-commutator}
For arbitrary sequences \(a\) and \(v\),
\begin{align}
\label{eq:ch8-second-commutator}
 a_k\D^2v_k
 &=\D^2(a_kv_k)
 -2(\D a_k)\D v_{k-1}
 -(\D^2a_k)v_{k-2}.
\end{align}
\end{proposition}

\begin{proof}
Apply~\eqref{eq:ch8-product-rule} twice. First,
\[
 a_k\D^2v_k
 =\D(a_k\D v_k)-(\D a_k)\D v_{k-1}.
\]
Apply~\eqref{eq:ch8-product-rule} to \(a_k\D v_k\):
\[
 a_k\D v_k
 =\D(a_kv_k)-v_{k-1}\D a_k.
\]
Taking one further difference gives
\[
 \D(a_k\D v_k)
 =\D^2(a_kv_k)-\D(v_{k-1}\D a_k).
\]
The product rule applied to the final term yields
\[
 \D(v_{k-1}\D a_k)
 =(\D v_{k-1})(\D a_k)+v_{k-2}\D^2a_k.
\]
Substitution proves~\eqref{eq:ch8-second-commutator}.
\end{proof}

The expansion provides a discrete pseudodifferential viewpoint: differences may be distributed between the state and the oscillatory carrier, but each redistribution produces explicit commutator terms. Boundary conditions and phase regularity determine whether the resulting expression is useful.

\section{Sharp obstructions and consistency tests}
\label{sec:ch8-obstructions}

Any proposed global theorem should be tested against exact identities before a proof is attempted.

\subsection{Zero-frequency boundary obstruction}

At \((x,t)=(0,0)\),
\begin{equation}
\label{eq:ch8-origin}
 F_N(0,0)=\frac1N\sum_{k=1}^{N}e_k.
\end{equation}
For first-order shaping,
\[
 \sum_{k=1}^{N}e_k=v_N-v_0.
\]
Thus a nonzero terminal trace creates an unavoidable contribution of order \(N^{-1}\). For higher-order shaping, testing against polynomial weights reveals the corresponding hierarchy of boundary traces.

\subsection{Full-period \(L^2\) obstruction}

The exact identity~\eqref{eq:ch8-x-orthogonality} gives
\[
 \int_a^{a+N}\abs{F_N(x,t)}^2\dd x
 =N^{-1}\norm{e}_{\ell^2}^2.
\]
If \(\abs{e_k}\asymp1\), the unnormalised \(L^2_x\) norm is of order one. Therefore no uniform positive decay exponent can hold in this norm for all such errors, irrespective of difference structure.

\subsection{Random bounded states}

A bounded random state may produce \(\norm{\D v}_{\ell^2}\asymp N^{1/2}\). The exact orthogonality identities then force the generic full-period scale. This test prevents an invalid theorem from replacing a state norm by \(\norm{v}_{\ell^\infty}\) without accounting for the number of independent oscillations.

\subsection{Interior-supported states}

If a state vanishes near both endpoints, all finite-record boundary terms disappear. Such examples isolate genuine interior oscillation. A theorem that still fails for interior-supported states cannot attribute the failure to endpoint traces.

\subsection{Periodic state orbits}

Deterministic sigma--delta recursions may possess periodic or nearly periodic trajectories. A correlation assumption that predicts square-root cancellation must be checked against these orbits. Periodicity can preserve large correlations and invalidate a mixing-based estimate even when the state is uniformly bounded.

\section{A consolidated hierarchy of proved estimates}
\label{sec:ch8-hierarchy}

The preceding results can be organised by the information retained about the coefficient sequence.

\subsection{Arbitrary coefficients}

For arbitrary \(e\), the following statements hold:
\begin{enumerate}[label=(\roman*)]
\item exact \(L^2_x\) orthogonality on every interval of length \(N\), Theorem~\ref{thm:ch8-x-orthogonality};
\item exact full-cell \(L^2\), Theorem~\ref{thm:ch8-full-cell-l2};
\item exact quadratic fourth moment, Proposition~\ref{prop:ch8-fourth-moment};
\item critical sixth-moment restriction bound, Corollary~\ref{cor:ch8-physical-l6};
\item local kernel identity and generic local \(L^2\) estimate, Proposition~\ref{prop:ch8-local-kernel} and Theorem~\ref{thm:ch8-generic-local-l2}.
\end{enumerate}

\subsection{First-order bounded states}

For \(e=\D v\) with bounded state:
\begin{enumerate}[label=(\roman*)]
\item fixed compact sets admit deterministic \(O(N^{-1})\) convergence;
\item the exact transfer formula~\eqref{eq:ch8-exact-transfer} separates boundary traces from the oscillatory state sum;
\item the two-index identity~\eqref{eq:ch8-double-sbp} identifies the mixed kernel difference required for a structured local \(L^2\) estimate;
\item state correlation or covariance bounds transfer directly to Fourier error bounds.
\end{enumerate}

\subsection{Endpoint-compatible high-order states}

For \(e=\D^rv\) with compatible boundary traces:
\begin{enumerate}[label=(\roman*)]
\item polynomial moments below order \(r\) vanish exactly;
\item smooth compact-scale weights yield \(O(N^{-r})\) error;
\item random or weakly correlated states satisfy high-order mean-square transfer estimates;
\item block-square estimates can be combined with decoupling through Theorem~\ref{thm:ch8-block-square-transfer}.
\end{enumerate}

\section{Critical-scale research target}
\label{sec:ch8-critical-target}

The natural unresolved target is a deterministic estimate that improves on both the variation bound and the arbitrary-coefficient restriction bound at a critical or near-critical scale. A suitable theorem must specify three ingredients:

\begin{enumerate}[label=(\alph*)]
\item a concrete state class generated by a stable one-bit sigma--delta recursion;
\item endpoint conditions or an explicit boundary-correction operator;
\item an oscillatory norm and parameter region compatible with exact orthogonality.
\end{enumerate}

A representative target has the form
\begin{equation}
\label{eq:ch8-critical-target}
 \norm{\E_N(\D^rv)}_{L^p(\Omega_N)}
 \le C_{p,r,\varepsilon}
 N^{-\delta+\varepsilon}\mathcal S_N(v),
\end{equation}
where \(\mathcal S_N(v)\) is a state quantity stronger than \(\norm{v}_{\ell^\infty}\) but verifiable from the quantiser dynamics. Possible choices suggested by the proved transfer principles include:
\[
 \mathcal S_N(v)
 =\left(\sum_h\abs{\operatorname{Corr}_v(h)}\right)^{1/2},
\]
a block-square norm, a covariance norm, or a frequency-localised Sobolev norm.

The exact identities in this chapter impose necessary consistency conditions on \(p\), \(\delta\), the normalisation of \(\Omega_N\), and the state class. In particular, the full-period \(L^2_x\) identity rules out uniform unnormalised decay for bounded errors of unit \(\ell^2\)-density, while endpoint-compatible high-order cancellation remains effective on fixed and subcritical regions.

\section{Chapter conclusion}
\label{sec:ch8-conclusion}

The local and averaged theory separates into three regimes. On fixed compact sets, finite summation by parts and bounded states give deterministic algebraic convergence, with order \(N^{-r}\) under compatible \(r\)th-order shaping. On full periodic cells, exact orthogonality and discrete restriction determine the baseline scale through \(\norm{e}_{\ell^2}\). Between these regimes, local kernels, state correlations, covariance bounds, block-square estimates, and commutator expansions provide precise transfer mechanisms.

The strongest conclusions established in this chapter are the exact quadratic moment identities, the local two-index summation-by-parts formula, the endpoint-compatible vanishing-moment theorem, deterministic high-order compact-set bounds, covariance-based mean-square estimates, and abstract oscillatory transfer principles. These results identify the additional state information required for a genuine critical-scale improvement and provide a rigorous foundation for the next stage of the analysis.

\chapter{Robustness, Boundary Engineering, and Verification}
\label{chap:robustness}

A finite-record theory must account for residual error, leakage, jitter, overload, imperfect terminal control, and block processing. The first two sections derive deterministic robustness bounds and boundary-engineering principles. The final section reports reproducible calculations that verify the recursions, the exact sharpness example, the closed form of the phase variation, and the second-order endpoint-compatible scaling.

\section{Robustness to residual error, leakage, and model mismatch}
\label{sec:14_robustness}

\subsection{Exact identities and physical implementations}

The ideal relation $e=\D^rv$ can be perturbed by comparator offsets, finite gain, coefficient mismatch, clock error, state leakage, or imperfect digital reset. Robust sigma-delta theory studies the resulting changes in stability and reconstruction \cite{GunturkLagariasVaishampayan2001,DaubechiesDeVoreGunturkVaishampayan2006}. Assuming a bounded perturbed state, this chapter derives deterministic extension bounds for several approximate identities and separates the contributions of state error, residual error, timing error, and boundary traces.

\subsection{Additive residuals}

Suppose
\begin{equation}
\label{eq:additive-residual-model}
 e_k=(\D v)_k+r_k,
\end{equation}
where $r$ is an unshaped residual.

\begin{theorem}[First-order residual bound]
\label{thm:first-order-residual}
Assume $v_0=0$ and $\norm{v}_{\ell^\infty}\le V$. Then, for arbitrary weights,
\begin{equation}
\label{eq:first-order-residual}
 \abs{\frac1N\sum_{k=1}^{N}e_kw_k}
 \le\frac{V}{N}\left(\abs{w_N}+\TV_N(w)\right)
 +\frac1N\sum_{k=1}^{N}\abs{r_k}\abs{w_k}.
\end{equation}
If $\abs{w_k}=1$, the residual contribution is $\norm{r}_{\ell^1}/N$.
\end{theorem}

\begin{proof}
Apply the weighted variation theorem to $\D v$ and the triangle inequality to $r$.
\end{proof}

If $\norm{r}_{\ell^1}=o(N)$, the residual vanishes after normalisation. If each $\abs{r_k}\le\rho$ with no further structure, the residual term is at most $\rho$ and creates a nondecaying floor. A small pointwise mismatch is therefore not automatically harmless on a long record.

\subsection{Residuals with their own noise-shaping order}

A perturbation may itself be partially shaped. Suppose
\[
 r=\D^s z
\]
with $s\ge1$ and bounded $z$. Then the residual can be estimated by another summation-by-parts argument. If $s=1$ and the initial state of $z$ is zero,
\[
 \abs{\frac1N\sum r_kw_k}
 \le\frac{\norm{z}_\infty}{N}
 \left(\abs{w_N}+\TV_N(w)\right).
\]
Thus two separately shaped components may be combined without producing an error floor.

This observation is useful for modelling element mismatch that has been dynamically scrambled or digitally noise-shaped. The relevant question is not only the size of the mismatch but also whether its accumulated sum is controlled.

\subsection{Leaky first-order shaping}

A simple leakage model is
\begin{equation}
\label{eq:leaky-model}
 e_k=v_k-\rho v_{k-1},
 \qquad 0<\rho\le1.
\end{equation}
Write
\[
 e_k=(v_k-v_{k-1})+(1-\rho)v_{k-1}.
\]
The first term is shaped and the second is an unshaped leakage component.

\begin{proposition}[Leaky weighted estimate]
\label{prop:leaky-weighted}
Assume $v_0=0$ and $\norm{v}_{\ell^\infty}\le V$. Then
\begin{equation}
\label{eq:leaky-weighted}
 \abs{\frac1N\sum_{k=1}^{N}e_kw_k}
 \le\frac{V}{N}\left(\abs{w_N}+\TV_N(w)\right)
 +(1-\rho)V\frac1N\sum_{k=1}^{N}\abs{w_k}.
\end{equation}
For unit-modulus weights,
\begin{equation}
\label{eq:leaky-unit}
 \abs{\frac1N\sum e_kw_k}
 \le\frac{V}{N}\left(1+\TV_N(w)\right)+(1-\rho)V.
\end{equation}
\end{proposition}

The term $(1-\rho)V$ is a leakage floor. To retain an $N^{-1}$ compact-set rate, one needs $1-\rho=O(N^{-1})$, unless additional cancellation is known. For a fixed physical leakage factor, increasing the record length eventually stops improving the deterministic bound.

\subsection{Perturbed weights and sampling jitter}

Suppose the intended weight is $w_k$ but the realised weight is
\[
 \widetilde w_k=w_k+\delta w_k.
\]
Then
\[
 \frac1N\sum e_k\widetilde w_k
 =\frac1N\sum e_kw_k+\frac1N\sum e_k\delta w_k.
\]
A pointwise error bound $\abs{e_k}\le E$ gives
\[
 \abs{\frac1N\sum e_k\delta w_k}
 \le E\norm{\delta w}_{\ell^\infty}.
\]
This estimate ignores shaping. A stronger result follows if the perturbation sequence is smooth in $k$:
\[
 \abs{\frac1N\sum (\D v)_k\delta w_k}
 \le\frac{V}{N}
 \left(\abs{\delta w_N}+\TV_N(\delta w)\right).
\]
Thus slowly varying timing or phase errors are filtered by the same discrete BV mechanism as the nominal weight.

For sampling jitter, let the actual node be
\[
 \widetilde\xi_k=\frac{k}{N}+\varepsilon_k
\]
and set
\[
 \widetilde w_k=e^{2\pi\ii\phi(\widetilde\xi_k)}.
\]
If $\phi$ is Lipschitz and $\abs{\varepsilon_k}\le\eta$, then
\[
 \abs{\widetilde w_k-w_k}\le2\pi\operatorname{Lip}(\phi)\eta.
\]
A useful bound must also control the variation of $\varepsilon_k$. If the jitter alternates rapidly, $\TV_N(\widetilde w)$ may be large even when the pointwise displacement is small.

\begin{proposition}[Jitter with bounded variation]
\label{prop:jitter-bv}
Assume $\phi\in C^1$ with $\norm{\phi'}_\infty\le L$, and assume the perturbed nodes remain in a fixed interval on which this bound holds. Then
\[
 \TV_N(\widetilde w)
 \le2\pi L\left(1+\TV_N(\varepsilon)\right),
\]
where
\[
 \TV_N(\varepsilon)=\sum_{k=1}^{N-1}\abs{\varepsilon_{k+1}-\varepsilon_k}.
\]
\end{proposition}

\begin{proof}
The distance between adjacent perturbed nodes satisfies
\[
 \abs{\widetilde\xi_{k+1}-\widetilde\xi_k}
 \le N^{-1}+\abs{\varepsilon_{k+1}-\varepsilon_k}.
\]
Apply the mean-value estimate to the phase and sum.
\end{proof}

Random independent jitter typically has variation proportional to $N$ and is therefore not small in this deterministic metric. A probabilistic analysis may yield better average behaviour, but it requires a specified distribution.

\subsection{Coefficient saturation and overload events}

Let $G\subset\{1,\ldots,N\}$ be the set of indices at which the ideal shaped relation holds, and let $B$ be an overload set. Write
\[
 e=e^{\mathrm{sh}}+e^{\mathrm{bad}},
\]
where $e^{\mathrm{bad}}$ is supported on $B$. If $\abs{e_k^{\mathrm{bad}}}\le M$, then for unit weights
\[
 \abs{\frac1N\sum e_k^{\mathrm{bad}}w_k}
 \le M\frac{\abs{B}}{N}.
\]
A vanishing fraction of overload events is therefore sufficient for consistency, even if the individual events are not small. To preserve an $N^{-1}$ rate, however, the number of bad samples must remain $O(1)$ or their errors must cancel.

This separation is useful in circuit modelling. A rare large transition and a persistent small mismatch have different asymptotic effects. The former is controlled by event count; the latter may create a fixed floor.

\subsection{Approximate high-order shaping}

Suppose
\[
 e=\D^rv+r
\]
with trace-compatible $v$. For a sampled $C^r$ weight $W$,
\begin{equation}
\label{eq:approx-high-order}
 \abs{\frac1N\sum e_kW(k/N)}
 \le VN^{-r}\norm{W^{(r)}}_\infty
 +\frac1N\sum\abs{r_k}\abs{W(k/N)}.
\end{equation}
To observe the $N^{-r}$ rate, the residual average must be $O(N^{-r})$. This is a stringent requirement. A residual of size $N^{-s}$ per sample contributes $N^{-s}$ after normalisation, so it dominates whenever $s<r$.

The conclusion is practical: a nominally high-order quantizer can display only first-order or zero-order behaviour if leakage, mismatch, or imperfect termination is not reduced to the same asymptotic scale.

\subsection{Dither and probabilistic modelling}

Dither is often introduced to decorrelate quantization error from the input. Classical dithered quantization can convert deterministic distortion into a random error with tractable moments. In sigma-delta frame quantization, random dither has been used to obtain mean-square estimates without invoking a white-noise hypothesis. The present manuscript does not assume a stochastic model, but the deterministic formulas identify what a probabilistic analysis should estimate:

\begin{enumerate}[label=(\roman*)]
\item moments of the boundary traces;
\item moments or concentration of the interior pairing $\sum v_k\Dp^rw_k$;
\item the probability and magnitude of overload events;
\item correlations between the state and the oscillatory phase.
\end{enumerate}

A probabilistic theorem can improve an average norm even when the worst-case deterministic variation remains large. Its conclusion is an average-norm estimate under the stated stochastic model.

\subsection{A robustness checklist}

For a claimed experimental or simulated convergence law, the following quantities should be reported.

\begin{enumerate}[label=(\alph*)]
\item The maximum state magnitude and the state order.
\item The initial and final boundary traces.
\item The residual sequence $r=e-\D^rv$.
\item The fraction of overload or saturation samples.
\item The variation of timing and phase perturbations.
\item The range of $N$ over which a fitted slope is measured.
\end{enumerate}

Without these diagnostics, a measured $N^{-r}$ line may reflect a short pre-asymptotic range, and a failure to reach $N^{-r}$ may be incorrectly attributed to the Fourier analysis rather than to boundary or residual errors.

\section{Blockwise reset, smooth termination, and boundary engineering}
\label{sec:15_blockwise_reset}

\subsection{The purpose of block processing}

A long coefficient record may be divided into shorter blocks so that the internal state can be reset, transmitted, or terminated. This is attractive when a global terminal condition is difficult to enforce. It also permits parallel processing. The cost is the appearance of one boundary contribution per block.

Let
\[
 0=n_0<n_1<\cdots<n_B=N
\]
be a partition, and let
\[
 I_b=\{n_{b-1}+1,\ldots,n_b\},
 \qquad L_b=n_b-n_{b-1}.
\]
For clarity, begin with equal block length $L$, so $B=N/L$.

\subsection{First-order block identity}

Assume that on block $I_b$,
\[
 e_k=v_k^{(b)}-v_{k-1}^{(b)},
\]
with a local state initialised at
\[
 v_{n_{b-1}}^{(b)}=0.
\]
The state is allowed to end at a nonzero value.

\begin{proposition}[Blockwise first-order estimate]
\label{prop:blockwise-first}
If $\norm{v^{(b)}}_{\ell^\infty(I_b)}\le V$ for every block, then
\begin{align}
\label{eq:blockwise-first}
 \abs{\frac1N\sum_{k=1}^{N}e_kw_k}
 &\le\frac{V}{N}\sum_{b=1}^{B}\abs{w_{n_b}}
 +\frac{V}{N}\sum_{b=1}^{B}
 \sum_{k=n_{b-1}+1}^{n_b-1}\abs{w_{k+1}-w_k}.
\end{align}
For unit-modulus weights,
\begin{equation}
\label{eq:blockwise-first-unit}
 \abs{\frac1N\sum e_kw_k}
 \le\frac{VB}{N}+\frac{V}{N}\TV_N(w)
 =\frac{V}{L}+\frac{V}{N}\TV_N(w).
\end{equation}
\end{proposition}

\begin{proof}
Apply the first-order weighted theorem on each block. The local initial term vanishes. Sum the resulting estimates. The interior variations of the blocks are disjoint and are bounded by the global variation.
\end{proof}

The reset cost is $V/L$. A fixed block length therefore creates a fixed error floor. To retain convergence, the block length must grow with $N$ or the terminal state must be cancelled.

\subsection{Choosing the block length}

For a fixed compact phase region, $\TV_N(w)=O(1)$. The bound becomes
\[
 \abs{\frac1N\sum e_kw_k}
 \lesssim\frac{V}{L}+\frac{V}{N}.
\]
If $L=N^\gamma$ with $0<\gamma\le1$, then the rate is $O(N^{-\gamma})$. Full first-order decay requires $L\asymp N$, which reduces to a single block. Smaller blocks trade asymptotic accuracy for local state management.

On a growing region with variation $O(N^\alpha)$, the estimate is
\[
 O(N^{-\gamma})+O(N^{\alpha-1}).
\]
Balancing the two terms suggests
\[
 \gamma=1-\alpha
\]
when $0<\alpha<1$. This gives $O(N^{\alpha-1})$. The block length needed to avoid dominating the variation error is therefore $L\gtrsim N^{1-\alpha}$.

\subsection{Exact terminal reset}

Suppose every block satisfies
\[
 v_{n_{b-1}}^{(b)}=v_{n_b}^{(b)}=0.
\]
Then the block endpoint term vanishes and
\[
 \sum_{k\in I_b}e_kw_k
 =-\sum_{k=n_{b-1}+1}^{n_b-1}v_k^{(b)}(w_{k+1}-w_k).
\]
Summing over blocks yields
\begin{equation}
\label{eq:blockwise-both-zero}
 \abs{\frac1N\sum e_kw_k}
 \le\frac{V}{N}\TV_N(w).
\end{equation}
The rate is the same as for a single globally initialised state, with the terminal amplitude term removed. First-order exact reset therefore permits parallel blocks without an asymptotic reset penalty.

Achieving this condition may require appended symbols or a short transition alphabet. The existence and length of a terminal control sequence are determined by the state recursion, the admissible input range, and the available output alphabet.

\subsection{High-order block compatibility}

For order $r$, each block has its own family of traces. If every block is trace-compatible, then
\[
 \sum_{k\in I_b}(\D^rv^{(b)})_kw_k
 =(-1)^r\sum_{k\in I_b^\circ}v_k^{(b)}(\Dp^rw)_k.
\]
The interior index sets are disjoint, so
\begin{equation}
\label{eq:blockwise-high-order-compatible}
 \abs{\frac1N\sum e_kw_k}
 \le\frac{V}{N}\sum_{k=1}^{N-r}\abs{\Dp^rw_k}.
\end{equation}
For sampled $C^r$ weights, this is $O(N^{-r})$ and does not depend on the number of blocks.

This conclusion is strong but the hypothesis is correspondingly strong. Each block must begin and end with the complete trace state prepared. For a second-order scheme, setting only the scalar state $v$ to zero is not enough; the first difference trace must also vanish.

\subsection{Imperfect block termination}

Let the $j$th right trace on block $b$ be bounded by $\eta_{j,b}$. With zero left traces, the total boundary error is bounded by
\begin{equation}
\label{eq:blockwise-imperfect-boundary}
 \frac1N\sum_{b=1}^{B}\sum_{j=0}^{r-1}
 \eta_{j,b}\abs{(\Dp^jw)_{n_b-j}}.
\end{equation}
For a sampled smooth weight,
\[
 \abs{\Dp^jw}\lesssim N^{-j}.
\]
If $\eta_{j,b}\le\eta_j$ uniformly and $B=N/L$, the $j$th trace contributes
\[
 O\!\left(\frac{\eta_j}{L}N^{-j}\right).
\]
To preserve the global $N^{-r}$ rate, one needs
\[
 \eta_j=O(LN^{-(r-j)}).
\]
The required accuracy depends on both the total record length and the block length.

\subsection{Windowed reconstruction as an alternative}

Instead of forcing the state to vanish, one can modify the weight near block boundaries. Let $\chi_b(k)$ be a window supported on block $b$ and consider
\[
 \sum_{k\in I_b}e_k\chi_b(k)w_k.
\]
If $\chi_b$ and its first $r-1$ discrete differences vanish at both ends, then the boundary terms generated by repeated summation by parts are suppressed even when the state traces are nonzero.

The price is that the windows alter the target extension. A partition of unity can reduce the bias:
\[
 \sum_{b=1}^{B}\chi_b(k)=1.
\]
Constructing windows with overlap, endpoint flatness, and a controlled sum of $r$th differences is a discrete analogue of smooth localisation in harmonic analysis. The technique resembles smooth frame-path termination, where the analysis vectors are arranged to approach zero with several derivatives \cite{BodmannPaulsenAbdulbaki2007}.

\subsection{A discrete polynomial window}

For a block of length $L$, define a rescaled coordinate
\[
 s=\frac{k-n_{b-1}}{L}.
\]
The polynomial
\[
 \chi(s)=s^r(1-s)^r
\]
vanishes to order $r$ at both endpoints. After normalisation and overlap, such functions can be used to build smooth windows. Their $r$th derivatives scale like $L^{-r}$ in the block coordinate. Therefore the cost of localisation is governed by $L^{-r}$ rather than $N^{-r}$.

This reveals a second block tradeoff. A local window of width $L$ naturally produces an $L^{-r}$ scale. To achieve a global $N^{-r}$ rate without state termination, the window width must be comparable with $N$. Multiple short windows generally lose powers.

\subsection{State flushing by an appended tail}

Another option is to append $T$ additional quantizer inputs after the original record. Let the original data occupy $1\le k\le N$, and choose control inputs for $N<k\le N+T$ so that the final traces vanish. The extension can then be formed with a weight that is zero on the appended tail, or the tail can be removed after its effect on the boundary formula is accounted for.

A rigorous flushing theorem requires three ingredients:

\begin{enumerate}[label=(\roman*)]
\item reachability of the zero trace state under the allowed alphabet;
\item a bound on the required tail length $T$;
\item stability of all intermediate states.
\end{enumerate}

For the first-order greedy one-bit scheme with arbitrary fixed input, exact reachability may fail under a rigid sign rule. Allowing a small set of terminal control values or modifying the final threshold can restore reachability. These are control-design questions rather than consequences of the summation identity.

\subsection{Parallel implementation and numerical experiments}

Block processing is attractive computationally because each block can be quantized independently. Numerical studies should compare at least four variants:

\begin{enumerate}[label=(\alph*)]
\item one uninterrupted state;
\item reset to zero at the beginning of each block only;
\item exact or approximate terminal reset;
\item smooth windowing without state reset.
\end{enumerate}

For each variant, plot the extension error against $N$, the block length $L$, and the number of blocks. The slope should be interpreted together with measured boundary traces. A blockwise scheme that appears to have high-order decay for moderate $N$ may eventually reach the predicted $1/L$ floor.

\subsection{Design principle}

The analysis leads to a simple principle:

\begin{quote}
Boundary management is part of the noise-shaping order. A scheme is not effectively order $r$ on a finite record unless its initialisation, termination, or reconstruction window removes all traces below order $r$.
\end{quote}

This principle links the algebraic theory in Chapter~\ref{chap:higher-order} to implementation choices. It also provides a clear way to separate a quantizer-design contribution from a Fourier-analysis contribution in a future journal paper.

\section{Verification of the proved finite-record results}
\label{sec:20_numerical_validation}

\subsection{Computational verification protocol}

Numerical experiments verify recursions, finite-difference identities, indexing conventions, constants, convergence slopes, and resonant trajectories. A reproducible experiment records the quantizer rule, initial state, input range, phase parameters, record length, and normalisation.

\subsection{Core experiment}

For a chosen input $u_k$, compute
\[
 q_k=Q(u_k+v_{k-1}),\qquad
 v_k=v_{k-1}+u_k-q_k.
\]
Verify numerically that
\[
 \max_k\abs{v_k}\le1
\]
and
\[
 \max_k\abs{u_k-q_k-(v_k-v_{k-1})}
\]
is at machine precision.

Next evaluate
\[
 E_N(x,t)=\abs{\E_Nu(x,t)-\E_Nq(x,t)}
\]
and compare it with
\[
 B_N(x,t)=\frac1N\left[1+2\pi J(x,t)\right].
\]
The computed ratio satisfies $E_N/B_N\le1$ up to floating-point error.

A reproducible example uses
\[
 u_k=0.55\sin\!\left(2\pi\left(3\frac{k}{N}+0.17\right)\right),
 \qquad (x,t)=(0.73,-0.41).
\]
For $N=2^6,\ldots,2^{13}$, the computed state remained below $0.9999$ in magnitude and the largest observed ratio between the error and the theoretical $J$-bound was below $0.039$. The fitted slope was approximately $-1.41$, reflecting additional cancellation for the selected input and parameter pair. The zero-input experiment below recovers the uniform sharpness exponent $-1$.

\begin{figure}[ht]
\centering
\includegraphics[width=0.78\textwidth]{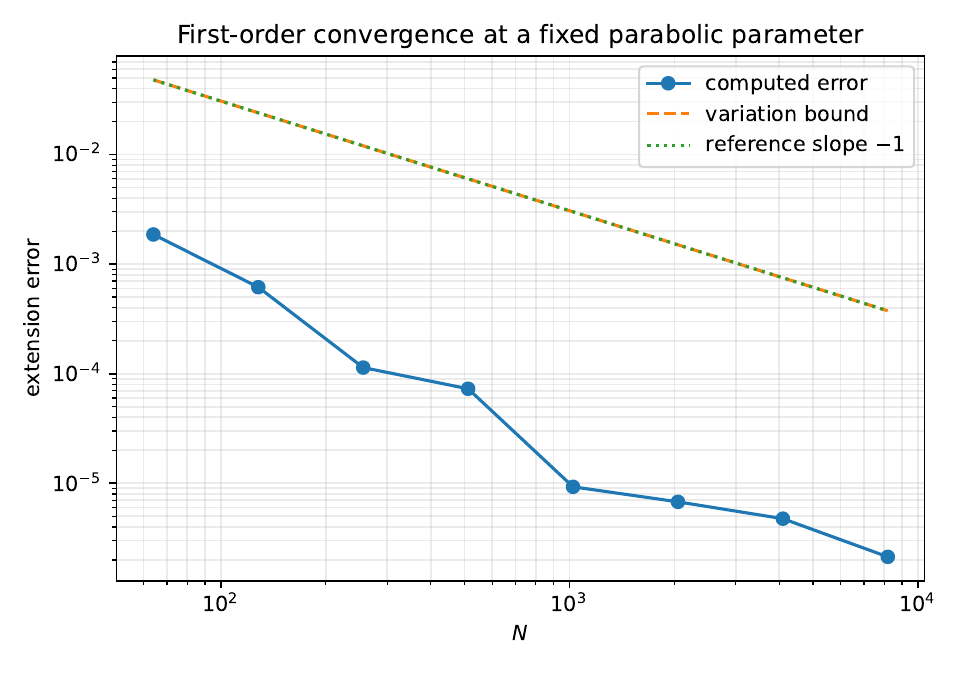}
\caption{Computed first-order error and the deterministic variation bound at a fixed parabolic parameter. The bound has the guaranteed $N^{-1}$ scale, while the selected input exhibits additional cancellation.}
\label{fig:first-order-numerical}
\end{figure}

\subsection{Convergence-order estimation}

Use a geometric sequence of lengths, for example
\[
 N=2^8,2^9,\ldots,2^{16}.
\]
At fixed $(x,t)$, fit a line to
\[
 \log E_N\quad\text{against}\quad\log N.
\]
A generic first-order trajectory has reference slope $-1$. Faster slopes identify additional terminal or oscillatory cancellation for the selected trajectory.

For the zero input and odd $N$, evaluate at $(0,0)$. The exact value is $1/N$, so the fitted slope should be exactly $-1$ up to numerical precision.

\begin{figure}[ht]
\centering
\includegraphics[width=0.75\textwidth]{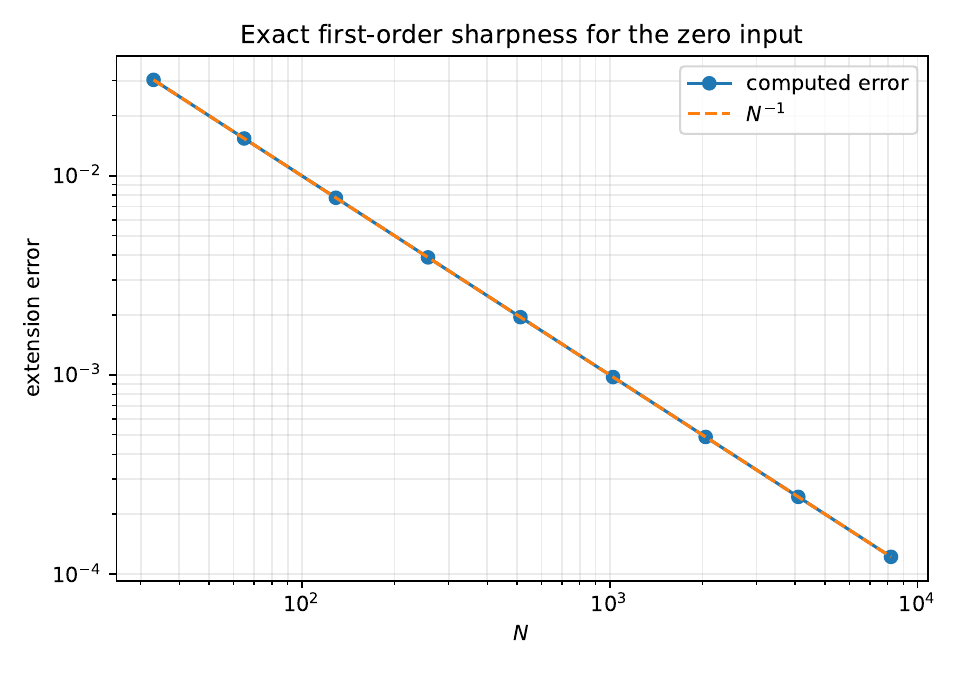}
\caption{Exact sharpness experiment for the zero input and odd record lengths. The computed values coincide with $N^{-1}$.}
\label{fig:zero-input-sharpness}
\end{figure}

\subsection{Testing the closed form of $J$}

Compute $J$ in two ways:

\begin{enumerate}[label=(\roman*)]
\item numerical quadrature of $\abs{x+2t\xi}$;
\item the piecewise formula in \cref{prop:J-closed}.
\end{enumerate}

Test points should include $t=0$, same-sign endpoints, and sign-changing derivatives. Near the transition $x(x+2t)=0$, both formulas should agree continuously.

\begin{figure}[ht]
\centering
\includegraphics[width=0.72\textwidth]{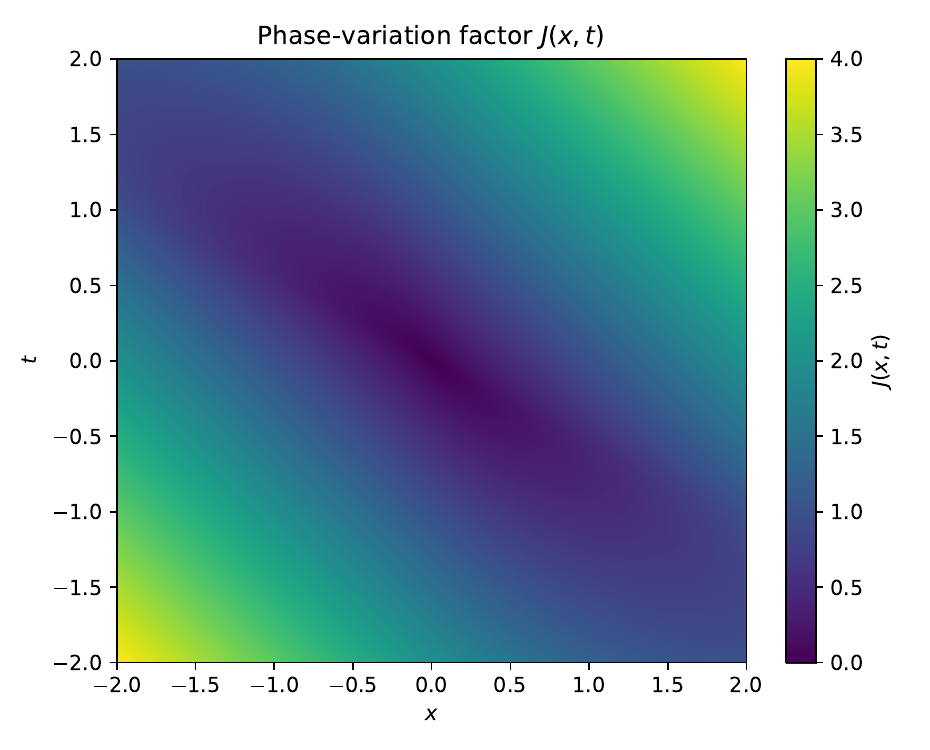}
\caption{The phase-variation factor $J(x,t)=\int_0^1\abs{x+2t\xi}\,\dd\xi$ on a bounded parameter square. The change in formula occurs where the derivative crosses zero inside the sampling interval.}
\label{fig:J-heatmap}
\end{figure}

\subsection{Exact sine variation}

Compare three bounds:
\[
 B_{\mathrm{simple}}=\frac1N[1+2\pi(\abs{x}+\abs{t})],
\]
\[
 B_J=\frac1N[1+2\pi J(x,t)],
\]
and
\[
 B_{\sin}=\frac1N\left[1+2\sum_{k=1}^{N-1}\abs{\sin(\pi\delta_k)}\right].
\]
The sine bound is the most faithful to the sampled weights. The $J$ bound is smooth and asymptotically natural. The simple bound is easiest to state.

\subsection{State-phase alignment}

To study the sharpness of the variation step, compute
\[
 A_k=v_k\Dp w_k.
\]
Plot or tabulate the arguments of $A_k$. If they are nearly aligned, the triangle inequality is close to equality. If their arguments spread around the circle, actual cancellation is stronger.

For realisable state trajectories, this experiment may reveal whether adversarial alignment occurs at specific rational parameters or inputs.

\subsection{Growing regions}

For each $N$, choose a grid in
\[
 \Omega_{N^\alpha}
\]
for several values of $\alpha$. Record the maximum error and compare the fitted exponent with $\alpha-1$. At $\alpha=1$, investigate whether the maximum remains bounded away from zero.

A uniform grid may miss narrow resonances. Adaptive refinement around large values is preferable.

\subsection{High-order synthetic validation}

Before implementing a physical high-order quantizer, one can validate the finite-interval identity with a synthetic bounded state. Choose $v_k$ directly, compute $e=\D^rv$, and evaluate both sides of \eqref{eq:repeated-sbp}. This test is effective for detecting off-by-one errors.

To test the endpoint-compatible theorem, construct a state that vanishes together with the required discrete traces. One method is to sample a smooth compactly supported profile and pad it with zeros. The resulting error need not be a one-bit quantization error, but it validates the analytic identity.

For the experiment in \cref{fig:second-order-numerical}, the state is zero at the first two and last two relevant grid locations, and its interior values sample a fourth power of a sine profile. The resulting second difference satisfies the exact finite-record compatibility conditions. A least-squares fit over $N=2^6,\ldots,2^{13}$ gives a slope of approximately $-1.991$, in agreement with the predicted second-order scale.

\begin{figure}[ht]
\centering
\includegraphics[width=0.78\textwidth]{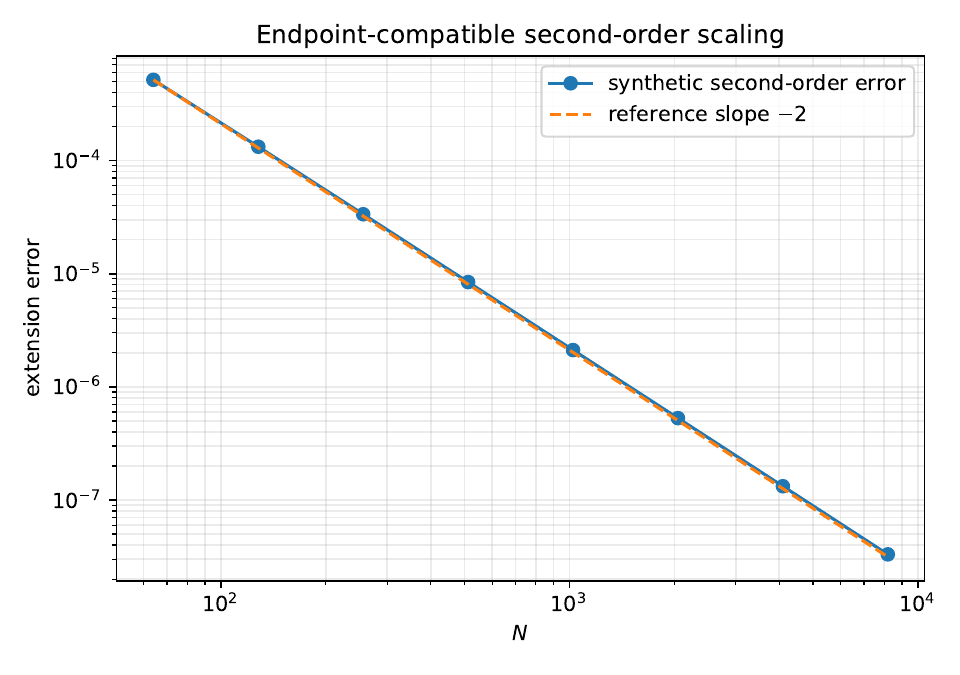}
\caption{Synthetic second-order validation with exact endpoint compatibility. The fitted exponent is approximately $-1.991$. This experiment verifies the analytic scaling but does not by itself construct a one-bit second-order quantizer with terminal control.}
\label{fig:second-order-numerical}
\end{figure}

\subsection{Monte Carlo decorrelation experiments}

For random inputs, estimate state correlations
\[
 C(h)=\frac1{N-h}\sum_{k=1}^{N-h}v_{k+h}v_k.
\]
Compare different input distributions, amplitudes, and deterministic sinusoids. Slow correlation decay or periodic peaks indicate possible resonances. Fast decay provides numerical evidence for the probabilistic hypotheses stated in Chapter~\ref{chap:lp-oscillation}.

\subsection{Reporting standards}

Every figure or table should include:

\begin{enumerate}[label=(\alph*)]
\item exact recursion and zero convention;
\item input definition and amplitude;
\item $N$ and phase grid;
\item whether $N$ is odd or even;
\item normalised or unnormalised extension;
\item measured state bound;
\item theoretical bound used for comparison;
\item numerical precision and software version.
\end{enumerate}

Appendix E provides Python and MATLAB-style code templates.

\chapter{Consolidated Derivations, Main Results, and Conclusions}
\label{chap:consolidated-results}

\section{Purpose of the consolidated chapter}

The preceding chapters develop the theory in its natural logical order. This chapter collects the principal assumptions, derivations, theorem dependencies, and scale relations in a compact report-style form. It also identifies the results that form a coherent focused journal manuscript.

\section{Master notation and standing assumptions}

Let $N\ge1$ and let $u=(u_k)_{k=1}^N$ be a real input record. A one-bit output record satisfies $q_k\in\{-1,+1\}$. The normalised polynomial Fourier extension associated with a phase $\phi_\theta$ is
\[
 \mathcal E_Na(\theta)
 =\frac1N\sum_{k=1}^Na_k\exp\!\left(2\pi\mathrm i\phi_\theta(k/N)\right).
\]
The first-order theory assumes
\begin{equation}
 u_k-q_k=v_k-v_{k-1},\qquad v_0=0,\qquad
 \max_{0\le k\le N}|v_k|\le V.
 \label{eq:master-first-order}
\end{equation}
The high-order theory assumes
\begin{equation}
 u_k-q_k=(\Delta^rv)_k,
 \label{eq:master-high-order}
\end{equation}
with all negative-index values required by the difference convention explicitly specified. Endpoint compatibility means that the initial and terminal traces occurring in the repeated finite summation formula vanish. Approximate compatibility means that these traces are small enough to be retained as controlled remainder terms.

\section{Complete first-order derivation}

\subsection{Stability of the greedy recursion}

The greedy first-order recursion is
\[
 q_k=Q(u_k+v_{k-1}),\qquad
 v_k=v_{k-1}+u_k-q_k,
\]
where $Q(y)=1$ for $y\ge0$ and $Q(y)=-1$ for $y<0$. Suppose that $|u_k|\le1$ and $|v_{k-1}|\le1$. Then $y_k=u_k+v_{k-1}$ belongs to $[-2,2]$. If $y_k\ge0$, then $v_k=y_k-1\in[-1,1]$. If $y_k<0$, then $v_k=y_k+1\in[-1,1)$. Since $v_0=0$, induction proves
\[
 |v_k|\le1\qquad (0\le k\le N).
\]
The proof uses no probabilistic model and no Fourier argument. It is an invariant-interval proof of deterministic stability, consistent with the classical mathematical theory of stable sigma-delta modulation \cite{DaubechiesDeVore2003,Gunturk2003CPAM,Gunturk2012}.

\subsection{Discrepancy identity}

For every interval $[m,n]\subseteq[1,N]$,
\begin{align}
 \sum_{k=m}^n(u_k-q_k)
 &=\sum_{k=m}^n(v_k-v_{k-1}) \\
 &=v_n-v_{m-1}.
\end{align}
Consequently,
\[
 \left|\sum_{k=m}^n(u_k-q_k)\right|\le2V.
\]
For a prefix interval with $m=1$ and $v_0=0$, the sharper bound is $V$. The estimate shows that first-order shaping controls accumulated error rather than pointwise error.

\subsection{Weighted variation theorem}

For arbitrary complex weights $w_1,\ldots,w_N$, exact summation by parts gives
\begin{align}
 \sum_{k=1}^N(u_k-q_k)w_k
 &=v_Nw_N-v_0w_1
 +\sum_{k=1}^{N-1}v_k(w_k-w_{k+1}).
 \label{eq:consolidated-weighted-identity}
\end{align}
Under \eqref{eq:master-first-order},
\begin{equation}
 \left|\frac1N\sum_{k=1}^N(u_k-q_k)w_k\right|
 \le\frac{V}{N}
 \left(|w_N|+\sum_{k=1}^{N-1}|w_{k+1}-w_k|\right).
 \label{eq:consolidated-weighted-bound}
\end{equation}
The proof is exact until the final application of the triangle inequality. The quantity in parentheses is an endpoint amplitude plus the discrete total variation of the weight.

\subsection{Absolutely continuous and BV phases}

Let $w_k=\exp(2\pi\mathrm i\phi(k/N))$. If $\phi$ is absolutely continuous, then
\[
 |w_{k+1}-w_k|
 \le2\pi\int_{k/N}^{(k+1)/N}|\phi'(\xi)|\,\mathrm d\xi.
\]
Summation and \eqref{eq:consolidated-weighted-bound} yield
\begin{equation}
 |\mathcal E_N^\phi u-\mathcal E_N^\phi q|
 \le\frac{V}{N}
 \left[1+2\pi\int_0^1|\phi'(\xi)|\,\mathrm d\xi\right].
 \label{eq:consolidated-ac-phase}
\end{equation}
If $\phi$ has bounded variation, the integral is replaced by the total variation of a suitable representative. This extension follows from the composition estimate for the Lipschitz map $s\mapsto e^{2\pi\mathrm i s}$ and standard BV theory \cite{AmbrosioFuscoPallara2000,EvansGariepy2015}.

\section{Parabolic result catalogue}

For $\phi_{x,t}(\xi)=x\xi+t\xi^2$, define
\[
 J(x,t)=\int_0^1|x+2t\xi|\,\mathrm d\xi.
\]
The principal first-order theorem is
\begin{equation}
 |\mathcal E_Nu(x,t)-\mathcal E_Nq(x,t)|
 \le\frac{V}{N}[1+2\pi J(x,t)].
 \label{eq:consolidated-parabolic}
\end{equation}
The exact continuous variation factor is
\[
 J(x,0)=|x|,
\]
\[
 J(x,t)=|x+t|\quad\text{when }t\ne0\text{ and }x(x+2t)\ge0,
\]
and
\[
 J(x,t)=\frac{x^2+(x+2t)^2}{4|t|}
 \quad\text{when }t\ne0\text{ and }x(x+2t)<0.
\]
The formula is obtained by locating the unique zero $\xi_*=-x/(2t)$ when it lies inside $(0,1)$ and integrating the absolute value separately on $[0,\xi_*]$ and $[\xi_*,1]$.

On a compact set $K\subset\mathbb R^2$,
\[
 \|\mathcal E_Nu-\mathcal E_Nq\|_{L^\infty(K)}
 \le\frac{V}{N}
 \left(1+2\pi\sup_{(x,t)\in K}J(x,t)\right).
\]
Thus the quantized extension converges uniformly to the unquantized extension at rate $N^{-1}$ on every fixed compact parameter set.

For $\Omega_R=[-R,R]^2$ and $1\le p<\infty$,
\[
 \|\mathcal E_Nu-\mathcal E_Nq\|_{L^p(\Omega_R)}
 \le(4R^2)^{1/p}\frac{V}{N}(1+4\pi R).
\]
The estimate is a direct consequence of the local $L^\infty$ bound and the volume of the parameter region. It does not use restriction theory.

\section{Sharpness of the first-order rate}

For the zero input, the greedy rule with $Q(0)=1$ generates
\[
 q_k=(-1)^{k+1},\qquad
 v_k=\begin{cases}-1,&k\text{ odd},\\0,&k\text{ even}.
 \end{cases}
\]
At $(x,t)=(0,0)$, every phase weight equals one. Therefore,
\[
 \mathcal E_Nu(0,0)-\mathcal E_Nq(0,0)
 =\frac1N\sum_{k=1}^N(u_k-q_k)
 =\frac{v_N}{N}.
\]
For odd $N$, the magnitude is exactly $1/N$. Hence no bound of order $o(N^{-1})$ can hold uniformly over all admissible inputs under only the first-order stability assumptions. The statement concerns uniform worst-case behaviour. It does not exclude faster convergence for a fixed input, for even record lengths, under terminal reset, or under additional cancellation conditions.

\section{Higher-order finite-record derivation}

\subsection{Why boundary traces appear}

Suppose $e=\Delta^rv$. On an infinite sequence, a formal Fourier calculation multiplies the transform of $v$ by $(1-e^{-\mathrm i\omega})^r$. A finite record is different. Repeated summation by parts transfers differences to the weight but also produces terms supported at the initial and terminal indices. These terms are not negligible merely because the interior weight is smooth.

For $r=2$, direct calculation gives a representative identity of the form
\begin{align}
 \sum_{k=1}^N(\Delta^2v)_kw_k
 &=\sum_{k=1}^{N-2}v_k(\Delta_+^2w)_k
 +\mathcal B_{2,N}(v,w),
\end{align}
where $\mathcal B_{2,N}$ contains the required initial and terminal traces. The exact indexing convention is stated and proved in \cref{chap:higher-order} and Appendix A.

\subsection{Endpoint-compatible rate}

Assume that all boundary traces in the order-$r$ finite identity vanish. Then
\[
 \sum_{k=1}^N(\Delta^rv)_kw_k
 =(-1)^r\sum_kv_k(\Delta_+^rw)_k
\]

after restriction to the valid interior indices. If $w_k=g(k/N)$ with $g\in C^r([0,1])$, the integral representation of finite differences gives
\[
 |(\Delta_+^rg)(k/N)|
 \le N^{-r}\int_{[0,1]^r}
 \left|g^{(r)}\!\left(\frac{k+s_1+\cdots+s_r}{N}\right)\right|
 \,\mathrm ds_1\cdots\mathrm ds_r.
\]
Summing over $k$ gives a bound of order $N^{1-r}$ for the unnormalised interior sum. The extension has an additional factor $N^{-1}$. Consequently,
\[
 \left|\frac1N\sum_{k=1}^N(\Delta^rv)_kg(k/N)\right|
 \le C_r\|v\|_\infty N^{-r}\|g^{(r)}\|_{L^1}
\]

after harmless endpoint enlargements. This is the mechanism behind the endpoint-compatible $O(N^{-r})$ theorem.

For oscillatory weights $g(\xi)=e^{2\pi\mathrm i\phi(\xi)}$, the derivative $g^{(r)}$ is a finite Bell-polynomial combination of $\phi',\ldots,\phi^{(r)}$. On compact parameter sets for polynomial phases, these derivatives are uniformly bounded. The high-order compact-set estimate follows.

\subsection{Approximate compatibility}

If a boundary trace does not vanish, its contribution must remain in the estimate. A typical order-$r$ bound has the form
\[
 |\mathcal E_Ne|
 \le C_r\|v\|_\infty N^{-r}
 +\frac1N\sum_{j=0}^{r-1}
 |\tau_{j,N}(v)|\,|\gamma_{j,N}(w)|,
\]
where $\tau_{j,N}$ denotes a state trace and $\gamma_{j,N}$ denotes the associated weight trace. High-order decay is preserved only if the second term is of order $N^{-r}$ or smaller. This formula explains the importance of smooth termination, reset sequences, and blockwise processing in frame quantization \cite{BodmannPaulsenAbdulbaki2007,BlumEtAl2010}.

\section{Boundary Correction and Fractional Rates}

The finite-record boundary terms can be handled either dynamically or analytically. Dynamic handling imposes endpoint compatibility on the state trajectory. Analytical handling records the boundary traces and adds their exact contribution to the reconstruction. For $e=\Delta^rv$, define
\[
 \mathcal B_{N,r}(v,w)
 =\sum_{j=0}^{r-1}(-1)^j
 \left[
 (\Delta^{r-1-j}v)_{N-j}(\Delta_+^jw)_{N-j}
 -(\Delta^{r-1-j}v)_0(\Delta_+^jw)_1
 \right].
\]
The boundary-corrected extension is
\[
 \mathcal R_{N,r}(q;v,w)
 =\frac1N\sum_{k=1}^Nq_kw_k+\frac1N\mathcal B_{N,r}(v,w).
\]
The exact error identity is
\[
 \frac1N\sum_{k=1}^Nu_kw_k-\mathcal R_{N,r}(q;v,w)
 =\frac{(-1)^r}{N}\sum_{k=1}^{N-r}v_k(\Delta_+^rw)_k.
\]
This result is proved in \cref{thm:boundary-corrected-reconstruction}. It has three immediate consequences.

\begin{enumerate}[label=(\roman*)]
\item A bounded state and a sampled $C^r$ weight give a corrected error of order $N^{-r}$ without terminal reset.
\item Zero initialisation eliminates all left-trace metadata. Only $r$ terminal trace values remain, independent of the record length.
\item Inexact metadata enters through an explicit weighted trace-error sum, so the precision required to preserve a target rate can be calculated before implementation.
\end{enumerate}

The fractional regularity theorem gives a second refinement. If $W\in C^{r-1,\alpha}$ with $0<\alpha\le1$, then
\[
 |\Delta_{1/N}^rW|\le
 N^{-(r-1+\alpha)}[W^{(r-1)}]_{C^{0,\alpha}}.
\]
Under endpoint compatibility, or after exact boundary correction,
\[
 \left|\frac1N\sum_{k=1}^N(u_k-q_k)W(k/N)\right|
 \le V[W^{(r-1)}]_{C^{0,\alpha}}N^{-(r-1+\alpha)}.
\]
This estimate places the finite-record theory on the standard modulus-of-smoothness scale used in constructive approximation \cite{DeVoreLorentz1993,DitzianTotik1987,Triebel1983}. It also shows that high-order noise shaping and test-function regularity contribute jointly to the final rate.

\section{Polynomial and multidimensional consequences}

For a degree-$d$ polynomial phase
\[
 \phi_{\boldsymbol\alpha}(\xi)
 =\alpha_1\xi+\alpha_2\xi^2+\cdots+\alpha_d\xi^d,
\]
the first-order estimate becomes
\[
 |\mathcal E_Nu(\boldsymbol\alpha)-\mathcal E_Nq(\boldsymbol\alpha)|
 \le\frac{V}{N}
 \left[1+2\pi\int_0^1
 \left|\sum_{j=1}^dj\alpha_j\xi^{j-1}\right|\,\mathrm d\xi\right].
\]
On a fixed coefficient box, the integral is uniformly bounded. Under order-$r$ endpoint compatibility, the rate improves to $N^{-r}$, with constants controlled by derivatives of the exponential weight through order $r$.

For a lattice array whose error has a discrete divergence representation,
\[
 e_{\mathbf k}=\sum_{j=1}^d\Delta_jv^{(j)}_{\mathbf k},
\]
discrete summation by parts in each coordinate yields an interior term involving directional differences of the weight and a boundary flux. If the vector state is bounded and the boundary flux is controlled, the normalised error is bounded by a sum of directional variations. The convergence rate remains first order in the grid spacing, while the constant depends on dimension and phase derivatives.

\section{Growing regions and scale transition}

For the parabolic box $|x|,|t|\le N^\alpha$, the simple first-order bound gives
\[
 \sup_{|x|,|t|\le N^\alpha}
 |\mathcal E_Nu-\mathcal E_Nq|
 \lesssim N^{\alpha-1}.
\]
Three regimes follow.
\begin{enumerate}[label=(\roman*)]
\item If $0\le\alpha<1$, the total-variation method gives decay.
\item If $\alpha=1$, it gives only a bounded estimate.
\item If $\alpha>1$, it does not provide a useful uniform approximation theorem.
\end{enumerate}
This transition is methodological. It does not prove that the error grows at critical or supercritical scales. It shows that cancellation must be used instead of taking absolute values of every adjacent phase increment.

Under order-$r$ endpoint compatibility, a derivative count suggests the subcritical scale
\[
 N^{-r}(1+N^\alpha)^r\sim N^{-r(1-\alpha)}.
\]
The rate still decays for $\alpha<1$. At the critical scale, further oscillatory information is again required.

\section{Local $L^p$, orthogonality, and external harmonic analysis}

Orthogonality in a linear phase parameter gives exact $L^2$ identities for complete periods. Such identities depend on the chosen normalisation and parameter domain. They do not by themselves produce a gain over the pointwise first-order bound for arbitrary shaped coefficients. Higher moments lead to Diophantine systems related to discrete restriction and Vinogradov mean values \cite{Bourgain1993I,HuLi2011,Wooley2012,BourgainDemeterGuth2016}.

Decoupling separates frequency intervals and controls an $L^p$ norm by a square function with $\ell^2$ coefficient structure \cite{BourgainDemeter2015,Demeter2020}. A sigma-delta state bound is naturally an $\ell^\infty$ or negative discrete Sobolev constraint. A direct decoupling gain therefore requires an additional transfer argument. The conditional results in \cref{chap:lp-oscillation} state explicitly which state-sum, covariance, or block-square hypotheses would be sufficient. None of those hypotheses is silently assumed in the first-order or endpoint-compatible theorems.

\section{Robustness decomposition}

Suppose the implemented error satisfies
\[
 u-q=\Delta v+\eta,
\]
where $\eta$ is an unshaped residual. The extension error separates as
\[
 \mathcal E_N(u-q)=\mathcal E_N(\Delta v)+\mathcal E_N\eta.
\]
The shaped term receives the variation gain. The residual satisfies
\[
 |\mathcal E_N\eta|
 \le\frac1N\sum_{k=1}^N|\eta_k|.
\]
A uniform residual of size $\varepsilon$ therefore creates an $O(\varepsilon)$ floor. If the residual has its own difference representation, another summation-by-parts gain is possible.

Leaky shaping changes the recursion and produces a geometric memory term. Sampling jitter perturbs the weights and can be controlled by a Lipschitz estimate in the phase. Saturation events are isolated by separating the regular and overloaded index sets. Each model introduces a specific remainder rather than invalidating the exact shaped component.

\section{Boundary engineering and block processing}

Blockwise reset localises the terminal contribution. If a record is divided into blocks of length $L$, first-order summation by parts is applied on each block. The resulting error consists of interior block variations and one boundary term per block. Exact reset eliminates the boundary terms. Approximate reset bounds them by the terminal mismatch.

At high order, reset must control a hierarchy of traces rather than a single state value. Smooth termination and appended correction tails are designed for this purpose. The analytic theorem and the quantizer construction should remain separate: the theorem states which traces are required, while the construction determines whether a one-bit or multibit rule can realise them.

\section{Numerical verification of exact statements}

The numerical experiments in \cref{chap:robustness} verify four features.
\begin{enumerate}[label=(\alph*)]
\item The first-order recursion satisfies the difference identity to machine precision and remains inside the invariant interval.
\item The computed parabolic extension error remains below the deterministic variation bound.
\item For the zero input and odd $N$, the error at the origin equals $1/N$.
\item A synthetic second-order state with exact endpoint compatibility exhibits an empirical slope close to $-2$.
\end{enumerate}
The fourth calculation verifies the finite-record analysis but does not construct a physical one-bit high-order quantizer with terminal control. This distinction is essential.

\section{Detailed treatment of nonzero initial states}

The zero-initial-state convention is natural for a newly started quantizer, but it should not be built into the general algebra. Let
\[
 e_k=v_k-v_{k-1},\qquad |v_k|\le V_+\quad(1\le k\le N),
\]
and assume only that $|v_0|\le V_0$. The exact identity becomes
\[
 \sum_{k=1}^Ne_kw_k
 =v_Nw_N-v_0w_1+\sum_{k=1}^{N-1}v_k(w_k-w_{k+1}).
\]
Hence
\begin{equation}
 \left|\frac1N\sum_{k=1}^Ne_kw_k\right|
 \le\frac1N\left(
 V_+|w_N|+V_0|w_1|+V_+\sum_{k=1}^{N-1}|w_{k+1}-w_k|
 \right).
 \label{eq:nonzero-initial-general}
\end{equation}
For unimodular phase weights, $|w_1|=|w_N|=1$, so
\[
 |\mathcal E_Nu-\mathcal E_Nq|
 \le\frac1N\left(V_0+V_++V_+\operatorname{TV}_N(w)\right).
\]
If the same bound $V$ applies to every state value, the endpoint constant is $2V$ rather than $V$. The asymptotic rate on fixed compact sets is unchanged, but the distinction matters in finite-length estimates.

A warm-started or block-processed quantizer may have a nonzero $v_0$. In that setting, the initial term represents information inherited from the preceding block. Three strategies are available. The first is exact reset, which sets $v_0=0$. The second is bounded carry-over, which keeps the term in \eqref{eq:nonzero-initial-general}. The third is state-aware reconstruction, in which the known value $v_0w_1/N$ is corrected explicitly. The third strategy removes the initial boundary contribution without changing the one-bit record.

For interval sums, the same distinction gives
\[
 \sum_{k=m}^ne_k=v_n-v_{m-1}.
\]
A prefix sum is bounded by $V_++V_0$ unless $v_0=0$. An interior interval is bounded by $2V_+$ when both endpoints belong to the same invariant state set. Thus the often-quoted prefix discrepancy bound $V$ is not a property of the difference representation alone. It also uses the initial condition.

\section{Worked second-order finite-record derivation}

The second-order case is sufficiently important to justify a line-by-line derivation. Let
\[
 e_k=(\Delta^2v)_k=v_k-2v_{k-1}+v_{k-2}.
\]
The first summation-by-parts step is applied to $\Delta v$:
\begin{align}
 \sum_{k=1}^N(\Delta^2v)_kw_k
 &=(\Delta v)_Nw_N-(\Delta v)_0w_1
 +\sum_{k=1}^{N-1}(\Delta v)_k(w_k-w_{k+1}).
 \label{eq:second-order-step-one-report}
\end{align}
Since $w_k-w_{k+1}=-(\Delta_+w)_k$, the remaining sum is
\[
 -\sum_{k=1}^{N-1}(\Delta v)_k(\Delta_+w)_k.
\]
A second summation-by-parts step gives
\begin{align}
 \sum_{k=1}^{N-1}(\Delta v)_k(\Delta_+w)_k
 &=v_{N-1}(\Delta_+w)_{N-1}
 -v_0(\Delta_+w)_1 \\
 &\quad+\sum_{k=1}^{N-2}v_k
 \left[(\Delta_+w)_k-(\Delta_+w)_{k+1}\right].
\end{align}
The bracket equals $-(\Delta_+^2w)_k$. Substitution into \eqref{eq:second-order-step-one-report} produces
\begin{align}
 \sum_{k=1}^N(\Delta^2v)_kw_k
 &= (\Delta v)_Nw_N-(\Delta v)_0w_1 \\
 &\quad-v_{N-1}(\Delta_+w)_{N-1}
 +v_0(\Delta_+w)_1 \\
 &\quad+\sum_{k=1}^{N-2}v_k(\Delta_+^2w)_k.
 \label{eq:second-order-full-audit}
\end{align}
This identity displays four boundary terms. Two involve the first difference of the state and an undifferenced weight. Two involve the state itself and a first difference of the weight. The interior term contains the desired second difference of the weight.

Under the common zero-initialisation convention $v_{-1}=v_0=0$, the left traces vanish because $(\Delta v)_0=v_0-v_{-1}=0$. The right traces remain:
\[
 (\Delta v)_Nw_N-v_{N-1}(\Delta_+w)_{N-1}.
\]
The condition $v_N=0$ alone is not enough, since $(\Delta v)_N=-v_{N-1}$ in that case. Exact second-order compatibility may be expressed as
\[
 (\Delta v)_N=0,
 \qquad
 v_{N-1}=0,
\]
together with the corresponding initial conditions. These two terminal requirements imply $v_N=v_{N-1}=0$, but the trace formulation is preferable because it generalises directly to higher order.

Taking absolute values in \eqref{eq:second-order-full-audit} gives
\begin{align}
 \left|\frac1N\sum_{k=1}^N(\Delta^2v)_kw_k\right|
 &\le\frac1N\Bigl(
 |(\Delta v)_N||w_N|+|(\Delta v)_0||w_1| \\
 &\qquad+|v_{N-1}||\Delta_+w_{N-1}|
 +|v_0||\Delta_+w_1|\Bigr) \\
 &\quad+\frac{\|v\|_\infty}{N}
 \sum_{k=1}^{N-2}|(\Delta_+^2w)_k|.
 \label{eq:second-order-full-bound-audit}
\end{align}
For a sampled $C^2$ weight, $|\Delta_+w|=O(N^{-1})$ and
\[
 \sum_{k=1}^{N-2}|\Delta_+^2w_k|=O(N^{-1}).
\]
The interior contribution in \eqref{eq:second-order-full-bound-audit} is therefore $O(N^{-2})$. An uncontrolled value of $(\Delta v)_N$ contributes $O(N^{-1})$ and dominates the interior term. An uncontrolled $v_{N-1}$ contributes $O(N^{-2})$ because it is multiplied by a first weight difference. This calculation explains the trace hierarchy: higher state differences require stronger terminal suppression.

\subsection{Second-order parabolic constant}

Let
\[
 W_{x,t}(\xi)=\exp(2\pi\mathrm i(x\xi+t\xi^2)).
\]
Differentiation gives
\[
 W'_{x,t}(\xi)=2\pi\mathrm i(x+2t\xi)W_{x,t}(\xi)
\]
and
\[
 W''_{x,t}(\xi)
 =\left[4\pi\mathrm i t-4\pi^2(x+2t\xi)^2\right]W_{x,t}(\xi).
\]
Since $|W_{x,t}|=1$,
\[
 |W''_{x,t}(\xi)|
 \le4\pi|t|+4\pi^2|x+2t\xi|^2.
\]
On $0\le\xi\le1$,
\[
 |x+2t\xi|\le|x|+2|t|.
\]
Consequently,
\[
 \|W''_{x,t}\|_{L^\infty(0,1)}
 \le4\pi|t|+4\pi^2(|x|+2|t|)^2.
\]
Under exact second-order endpoint compatibility, the integral representation of the second forward difference gives
\[
 \sum_{k=1}^{N-2}|(\Delta_+^2w)_k|
 \le\frac1N\|W''_{x,t}\|_{L^\infty(0,1)}.
\]
Substitution into the interior term of \eqref{eq:second-order-full-bound-audit} yields
\[
 |\mathcal E_Nu(x,t)-\mathcal E_Nq(x,t)|
 \le\frac{V}{N^2}
 \left[4\pi|t|+4\pi^2(|x|+2|t|)^2\right].
\]
This derivation identifies the origin of every constant and confirms that the second-order rate is a boundary-compatible result rather than an automatic consequence of $e=\Delta^2v$.

\section{Detailed proof of the piecewise formula for the parabolic variation}

The function
\[
 J(x,t)=\int_0^1|x+2t\xi|\,\mathrm d\xi
\]
is the total variation of the parabolic phase on the unit sampling interval. Its formula depends on whether the affine derivative changes sign.

If $t=0$, the integrand is constant and
\[
 J(x,0)=|x|.
\]
Assume next that $t\ne0$. The endpoint values of the affine function are $x$ and $x+2t$. If
\[
 x(x+2t)\ge0,
\]
the endpoint values have the same sign or one is zero. The affine function has no interior sign change. Therefore,
\begin{align}
 J(x,t)
 &=\left|\int_0^1(x+2t\xi)\,\mathrm d\xi\right| \\
 &=|x+t|.
\end{align}
The absolute value may be moved outside the integral because the integrand has a fixed sign.

If
\[
 x(x+2t)<0,
\]
the unique zero
\[
 \xi_*=-\frac{x}{2t}
\]
belongs to $(0,1)$. The integral splits at $\xi_*$. A sign-independent calculation can be obtained by using the areas of two triangles. The magnitude of the slope is $2|t|$. The first triangle has base $\xi_*$ and height $|x|$, while the second has base $1-\xi_*$ and height $|x+2t|$. Thus
\begin{align}
 J(x,t)
 &=\frac12|x|\xi_*
 +\frac12|x+2t|(1-\xi_*) \\
 &=\frac{x^2+(x+2t)^2}{4|t|}.
\end{align}
The final equality follows by substituting $\xi_*=-x/(2t)$ and using the opposite signs of the endpoint values.

The formulas agree continuously at the transition. For example, if $x=0$, the sign-changing formula has limiting value
\[
 \frac{0+(2t)^2}{4|t|}=|t|,
\]
which agrees with $|x+t|=|t|$. The same conclusion holds when $x+2t=0$. Hence $J$ is continuous on $\mathbb R^2$, positively homogeneous of degree one, and convex as an integral of absolute values of affine functions.

Positive homogeneity gives
\[
 J(\lambda x,\lambda t)=|\lambda|J(x,t).
\]
Convexity gives
\[
 J(\theta(x_1,t_1)+(1-\theta)(x_2,t_2))
 \le\theta J(x_1,t_1)+(1-\theta)J(x_2,t_2)
\]
for $0\le\theta\le1$. These properties are useful when estimating $J$ on convex parameter regions. In particular, the maximum of a convex function on a compact polytope occurs on the boundary, and simple upper bounds can be obtained from the vertices.

\section{Constant dependence and uniformity statements}

A convergence rate is incomplete unless the dependence of its constant is stated. The first-order parabolic estimate has the form
\[
 |\mathcal E_Nu(x,t)-\mathcal E_Nq(x,t)|
 \le C(x,t,V)N^{-1},
\]
where
\[
 C(x,t,V)=V[1+2\pi J(x,t)].
\]
For a fixed point $(x,t)$, the constant is independent of $N$ and of the particular input sequence, provided the same state bound $V$ applies. For a compact set $K$, the uniform constant is
\[
 C_K=V\left[1+2\pi\sup_{(x,t)\in K}J(x,t)\right].
\]
The estimate is therefore uniform over all admissible input records and all parameters in $K$.

For a family of growing sets $K_N$, the constant becomes $N$-dependent through
\[
 M_N=\sup_{(x,t)\in K_N}J(x,t).
\]
The bound is
\[
 \|\mathcal E_Nu-\mathcal E_Nq\|_{L^\infty(K_N)}
 \le\frac{V}{N}(1+2\pi M_N).
\]
Decay occurs whenever $M_N=o(N)$. This formulation is more general than a power-law box. It applies to anisotropic regions, curved parameter sets, and unions of local patches.

For the degree-$d$ polynomial phase, define
\[
 M_{1,K}=\sup_{\boldsymbol\alpha\in K}
 \int_0^1\left|\sum_{j=1}^dj\alpha_j\xi^{j-1}\right|\,\mathrm d\xi.
\]
Then
\[
 \sup_{\boldsymbol\alpha\in K}
 |\mathcal E_Nu(\boldsymbol\alpha)-\mathcal E_Nq(\boldsymbol\alpha)|
 \le\frac{V}{N}(1+2\pi M_{1,K}).
\]
For endpoint-compatible order-$r$ shaping, the corresponding constant depends on derivatives of the exponential weight through order $r$. A convenient compact-set quantity is
\[
 M_{r,K}=\sup_{\boldsymbol\alpha\in K}
 \sup_{0\le\xi\le1}
 B_r\!\left(2\pi|\phi'_{\boldsymbol\alpha}(\xi)|,
 \ldots,2\pi|\phi^{(r)}_{\boldsymbol\alpha}(\xi)|\right),
\]
where $B_r$ is the complete Bell polynomial. The high-order estimate then has the schematic form
\[
 \sup_{\boldsymbol\alpha\in K}|\mathcal E_Ne(\boldsymbol\alpha)|
 \le C_rVM_{r,K}N^{-r}
\]
when every boundary trace is compatible. The dependence on $K$, $r$, and $V$ is explicit. The constant is not dimension-free when the phase degree or the number of parameters increases.

\section{Assumption-to-conclusion map}

The following table records the exact role of each assumption.

\begin{longtable}{p{0.24\textwidth}p{0.31\textwidth}p{0.35\textwidth}}
\caption{Role of the principal assumptions in the proved estimates}\label{tab:assumption-audit}\\
\toprule
Assumption & Point of use & Consequence if removed\\
\midrule
\endfirsthead
\toprule
Assumption & Point of use & Consequence if removed\\
\midrule
\endhead
$e=\Delta v$ & Exact summation by parts & No variation gain follows from a pointwise error bound alone.\\
$\|v\|_\infty\le V$ & Triangle inequality after summation by parts & The weighted state sum may grow with $N$.\\
$v_0=0$ & Removal of the initial endpoint term & An additional contribution $|v_0w_1|/N$ remains.\\
Unimodular phase weight & Replacement of $|w_N|$ by one & General weights require their endpoint amplitude explicitly.\\
Absolute continuity of $\phi$ & Fundamental theorem of calculus on each cell & BV phases require a measure-theoretic variation argument.\\
Fixed compact parameter set & Uniform control of phase derivatives & Growing sets introduce an $N$-dependent constant.\\
$e=\Delta^rv$ & Repeated finite summation by parts & First-order structure alone cannot yield the order-$r$ interior difference.\\
Endpoint compatibility & Removal of high-order boundary traces & The leading rate may revert to $N^{-1}$ or an intermediate order.\\
$C^r$ sampled weight & $r$th finite-difference estimate & Lower regularity supports only the order allowed by the available modulus of smoothness.\\
Divergence-form grid error & Coordinatewise summation by parts & A general multidimensional array has no directional variation gain.\\
Residual decomposition $e=\Delta v+\eta$ & Separation of shaped and unshaped components & An unmodelled residual may create an error floor.\\
Additional state decorrelation hypothesis & Conditional critical-scale estimates & Stability alone does not justify oscillatory cancellation of the state sum.\\
\bottomrule
\end{longtable}

The map shows that the principal theorems are deterministic implications. Randomness, input smoothness, and bandlimiting enter only when constructing a stable high-order quantizer, estimating state correlations, or comparing with a continuous signal model; the finite-record summation identities depend solely on the assumptions displayed in the table.

\section{Separation between proved approximation and restriction estimates}

The approximation problem and the restriction problem use related exponential sums but ask different questions. The approximation problem controls
\[
 \mathcal E_Nu-\mathcal E_Nq=\mathcal E_N(u-q)
\]
by exploiting the difference representation of the coefficient error. The unquantized coefficients $u_k$ remain part of the comparison. A restriction estimate instead seeks a norm inequality for
\[
 \sum_{k=1}^Na_ke(\phi_\theta(k))
\]
for a broad class of coefficients $a_k$, often with an $\ell^2$ norm on the right-hand side and a scale-sensitive $L^p$ norm on the left.

The first-order theorem does not improve a restriction exponent for arbitrary binary sequences. It states that a binary sequence produced by a stable noise-shaping rule approximates a specified bounded input sequence after application of a low-variation extension functional. The distinction remains even when the phase is parabolic and the same geometry appears in Schrödinger estimates.

A transfer from approximation to restriction would require an independent bound on $\mathcal E_Nu$, a structural description of the admissible inputs, or a direct oscillatory estimate for the state term. The triangle inequality
\[
 \|\mathcal E_Nq\|_{L^p}
 \le\|\mathcal E_Nu\|_{L^p}
 +\|\mathcal E_Nu-\mathcal E_Nq\|_{L^p}
\]
is valid, but its usefulness depends on a suitable estimate for the first term. The local approximation theorem supplies only the second term.

The report therefore establishes a rigorous bridge between noise shaping and polynomial Fourier extension, while the transition to global restriction estimates is expressed through the additional state hypotheses developed in Chapter~\ref{chap:lp-oscillation}.

\section{Summary of principal results}

\begin{longtable}{p{0.25\textwidth}p{0.29\textwidth}p{0.38\textwidth}}
\caption{Principal mathematical results and analytical settings}\label{tab:status-summary}\\
\toprule
Result & Analytical setting & Main conclusion\\
\midrule
\endfirsthead
\toprule
Result & Analytical setting & Main conclusion\\
\midrule
\endhead
Greedy first-order stability & $|u_k|\le1$ & The state satisfies $|v_k|\le1$.\\
Interval discrepancy & $e=\Delta v$ with bounded state & Interval sums are controlled independently of interval length.\\
Weighted variation theorem & Arbitrary complex weights & The error is bounded by endpoint amplitude and discrete total variation.\\
AC and BV phase estimates & Absolutely continuous or BV phase & Sampled variation is controlled by continuous phase variation.\\
Parabolic compact-set rate & Stable first-order shaping & The extension error is uniformly $O(N^{-1})$ on fixed compact sets.\\
Uniform first-order sharpness & Zero input and odd record length & The error equals $1/N$ at the origin.\\
Finite order-$r$ identity & Finite records with explicit traces & Repeated summation by parts retains all initial and terminal terms.\\
Endpoint-compatible high-order rate & Vanishing traces and smooth weights & The compact-set error is $O(N^{-r})$.\\
Boundary-corrected reconstruction & Exact state-trace metadata & The interior $O(N^{-r})$ rate is recovered without terminal compatibility.\\
Fractional regularity rate & $C^{r-1,\alpha}$ sampled weights & The rate is $O(N^{-(r-1+\alpha)})$.\\
Polynomial and grid extensions & Polynomial phases and bounded divergence states & Constants depend on phase derivatives, dimension, and boundary flux.\\
Critical-scale oscillatory transfer & State-sum, covariance, or square-function estimates & The assumed cancellation transfers directly to the extension error.\\
Terminal state control & Quantizer-dependent terminal construction & Endpoint trace conditions are linked to appended symbols and admissible alphabets.\\
\bottomrule
\end{longtable}

\section{Scope of the established results}

The approximation theorem compares $\mathcal E_Nu$ with $\mathcal E_Nq$ for coefficients connected by a stable difference representation. The high-order compact-set theorem uses endpoint compatibility or boundary-trace correction. The multidimensional theorem uses a bounded divergence-form state representation. At critical parameter scales, state-sum, covariance, or square-function information supplements total variation. These assumptions define the analytical setting in which the report's estimates apply and identify the additional structure relevant to global restriction bounds.

\section{Conclusions}

A stable sigma-delta state converts a large pointwise one-bit error into a structured discrete derivative. Finite summation by parts is the central analytical mechanism. At first order, the mechanism gives an $N^{-1}$ approximation rate against weights of bounded variation. For the parabolic phase, the constant is controlled by an explicitly computable variation factor. The zero-input orbit proves that the uniform rate is sharp under the stated assumptions.

At higher order, the finite record introduces a hierarchy of boundary traces. Once those traces are controlled, repeated summation by parts produces the expected $N^{-r}$ compact-set rate. The same rate follows without terminal compatibility when the exact boundary functional is added to the reconstruction. Fractional weight regularity produces the intermediate scale $N^{-(r-1+\alpha)}$. Polynomial, multidimensional, and robustness extensions follow from the same principle after the correct directional derivatives and boundary fluxes are identified.

The resulting theory supports a focused mathematical paper centred on the finite-record connection between stable noise shaping, boundary-compatible higher-order differences, boundary-corrected reconstruction, and discrete polynomial Fourier extension. Critical-scale extensions naturally lead to state dynamics, correlation estimates, and oscillatory cancellation.

\appendix
\chapter{Expanded Discrete Identities}

\section{Binomial formulas}

For any sequence $a$ and integer $r\ge0$,
\[
 \D^ra_k=\sum_{j=0}^r(-1)^j\binom{r}{j}a_{k-j},
\]
\[
 \Dp^ra_k=\sum_{j=0}^r(-1)^{r-j}\binom{r}{j}a_{k+j}.
\]
These formulas follow from the binomial theorem applied to the shift operators
\[
 \D=I-S,
 \qquad
 \Dp=S^+-I.
\]
Here $Sa_k=a_{k-1}$ and $S^+a_k=a_{k+1}$.

\section{Low-order expansions}

The first four backward differences are
\begin{align*}
 \D a_k&=a_k-a_{k-1},\\
 \D^2a_k&=a_k-2a_{k-1}+a_{k-2},\\
 \D^3a_k&=a_k-3a_{k-1}+3a_{k-2}-a_{k-3},\\
 \D^4a_k&=a_k-4a_{k-1}+6a_{k-2}-4a_{k-3}+a_{k-4}.
\end{align*}
The corresponding forward differences are obtained by reversing the index direction.

\section{Verification of the second-order finite formula}

Starting from
\[
 \sum_{k=1}^N\D^2v_kw_k,
\]
apply first-order summation by parts to $a_k=\D v_k$:
\[
 \sum_{k=1}^N\D^2v_kw_k
 =(\D v)_Nw_N-(\D v)_0w_1
 -\sum_{k=1}^{N-1}\D v_k\Dp w_k.
\]
Apply first-order summation by parts again on the shorter interval:
\[
 \sum_{k=1}^{N-1}\D v_k\Dp w_k
 =v_{N-1}\Dp w_{N-1}-v_0\Dp w_1
 -\sum_{k=1}^{N-2}v_k\Dp^2w_k.
\]
Substitution yields \eqref{eq:second-order-sbp}.

\section{Fourth-order boundary table}

For $r=4$, the terminal terms are
\[
 (\D^3v)_Nw_N,
 \quad
 -(\D^2v)_{N-1}\Dp w_{N-1},
 \quad
 +(\D v)_{N-2}\Dp^2w_{N-2},
 \quad
 -v_{N-3}\Dp^3w_{N-3}.
\]
The initial terms have the opposite endpoint and alternating signs:
\[
 -(\D^3v)_0w_1,
 \quad
 +(\D^2v)_0\Dp w_1,
 \quad
 -(\D v)_0\Dp^2w_1,
 \quad
 +v_0\Dp^3w_1.
\]
The interior term is
\[
 +\sum_{k=1}^{N-4}v_k\Dp^4w_k.
\]

\begin{table}[ht]
\centering
\caption{Boundary hierarchy for order four}
\begin{tabular}{cccc}
\toprule
$j$ & state trace & weight difference & normalised smooth scale\\
\midrule
0 & $\D^3v$ & $w$ & $N^{-1}$\\
1 & $\D^2v$ & $\Dp w$ & $N^{-2}$\\
2 & $\D v$ & $\Dp^2w$ & $N^{-3}$\\
3 & $v$ & $\Dp^3w$ & $N^{-4}$\\
interior & $v$ & $\Dp^4w$ summed over $N$ points & $N^{-4}$\\
\bottomrule
\end{tabular}
\end{table}

The table shows that removing the first three boundary layers is necessary to expose fourth-order decay.

\section{A two-dimensional second-order example}

Let
\[
 e_{k,\ell}=\D_1\D_2v_{k,\ell}.
\]
Summation by parts in the first coordinate gives terminal and initial vertical faces plus an interior $\Delta_{+,1}w$. Applying summation by parts in the second coordinate to each term yields four corner traces, four edge sums, and one interior mixed difference
\[
 \sum v_{k,\ell}\Delta_{+,1}\Delta_{+,2}w_{k,\ell}.
\]
The exact signs depend on the zero-padding convention. The structure mirrors the continuous identity
\[
 \int\partial_1\partial_2v\,w
 =\int v\,\partial_1\partial_2w
 +\text{edge and corner traces}.
\]

\section{Discrete product rules}

The first-order product rules are
\[
 \D(a_kb_k)=a_k\D b_k+b_{k-1}\D a_k,
\]
\[
 \Dp(a_kb_k)=a_{k+1}\Dp b_k+b_k\Dp a_k.
\]
Equivalently,
\[
 a_k\D b_k=\D(a_kb_k)-b_{k-1}\D a_k.
\]
These identities are useful when separating a rapidly oscillating carrier from a slowly varying envelope.

For second order,
\begin{align*}
 \D^2(ab)_k
 &=a_k\D^2b_k+2(\D a_k)(\D b_{k-1})+b_{k-2}\D^2a_k,
\end{align*}
with appropriate index conventions. Higher-order formulas involve discrete binomial sums.

\section{Periodic differences}

On the discrete torus $\Z/N\Z$, define
\[
 \D_{\mathrm{per}}v_k=v_k-v_{k-1\,\mathrm{mod}\,N}.
\]
Then
\[
 \sum_{k=1}^N\D_{\mathrm{per}}v_kw_k
 =-\sum_{k=1}^Nv_k\Delta_{+,\mathrm{per}}w_k
\]
with no boundary terms. However, a periodic state must satisfy
\[
 \sum_{k=1}^Ne_k=0.
\]
A finite sigma-delta record with arbitrary terminal state is not periodic. Periodising it changes the error sequence by a boundary impulse.

\section{Zero extension and boundary impulses}

Extend a finite state by zero outside its record. The distributional finite difference of the zero extension contains impulses at both boundaries. The finite summation formula is another way of accounting for those impulses. This viewpoint is useful when embedding the problem into a bi-infinite Fourier transform.

\section{Norm bounds for state differences}

If $\norm{v}_\infty\le V$, then
\[
 \norm{\D^mv}_\infty\le2^mV.
\]
The constant follows from the binomial formula:
\[
 \abs{\D^mv_k}
 \le V\sum_{j=0}^m\binom{m}{j}=2^mV.
\]
This estimate is sharp over arbitrary bounded sequences, as shown by alternating signs.

\chapter{Derivative Bounds for Oscillatory Weights}

\section{Fa\`a di Bruno and Bell polynomials}

Let $f=e^g$. The first derivatives are
\begin{align*}
 f'&=e^gg',\\
 f''&=e^g\left((g')^2+g''\right),\\
 f'''&=e^g\left((g')^3+3g'g''+g'''\right),\\
 f^{(4)}&=e^g\left((g')^4+6(g')^2g''+3(g'')^2+4g'g'''+g^{(4)}\right).
\end{align*}
In general,
\[
 f^{(r)}=e^gB_r(g',\ldots,g^{(r)}),
\]
where $B_r$ is the complete exponential Bell polynomial.

For $g=2\pi\ii\phi$, the unimodular factor $e^g$ disappears after taking absolute values. Therefore
\[
 \norm{f^{(r)}}_\infty
 \le B_r\left(2\pi\norm{\phi'}_\infty,
 \ldots,2\pi\norm{\phi^{(r)}}_\infty\right).
\]

\section{Low-order phase bounds}

For $f=e^{2\pi\ii\phi}$,
\[
 \norm{f'}_\infty\le2\pi\norm{\phi'}_\infty.
\]
Also,
\[
 \norm{f''}_\infty
 \le(2\pi)^2\norm{\phi'}_\infty^2
 +2\pi\norm{\phi''}_\infty.
\]
At third order,
\[
 \norm{f'''}_\infty
 \le(2\pi)^3\norm{\phi'}_\infty^3
 +3(2\pi)^2\norm{\phi'}_\infty\norm{\phi''}_\infty
 +2\pi\norm{\phi'''}_\infty.
\]
These explicit forms are preferable to a generic constant when $r$ is small.

\section{Quadratic phase formulas}

Let
\[
 \phi(\xi)=x\xi+t\xi^2.
\]
Then
\[
 g'=2\pi\ii(x+2t\xi),
 \qquad g''=4\pi\ii t,
 \qquad g^{(m)}=0\quad(m\ge3).
\]
The first three derivatives satisfy
\begin{align*}
 f'&=f\,2\pi\ii(x+2t\xi),\\
 f''&=f\left[-(2\pi)^2(x+2t\xi)^2+4\pi\ii t\right],\\
 f'''&=f\left[(2\pi\ii(x+2t\xi))^3+3(2\pi\ii(x+2t\xi))(4\pi\ii t)\right].
\end{align*}
For bounding purposes, it is safer to use the Bell expression rather than simplifying complex signs.

Set
\[
 A=2\pi(\abs{x}+2\abs{t}),
 \qquad B=4\pi\abs{t}.
\]
Then
\[
 \norm{f'}_\infty\le A,
\]
\[
 \norm{f''}_\infty\le A^2+B,
\]
\[
 \norm{f'''}_\infty\le A^3+3AB,
\]
\[
 \norm{f^{(4)}}_\infty\le A^4+6A^2B+3B^2.
\]
These are the complete Bell polynomials with higher derivatives set to zero.

\section{Polynomial phases}

Let
\[
 \phi_x(\xi)=\sum_{j=1}^dx_j\xi^j.
\]
For $1\le m\le d$,
\[
 \phi_x^{(m)}(\xi)
 =\sum_{j=m}^d\frac{j!}{(j-m)!}x_j\xi^{j-m}.
\]
Hence
\[
 \norm{\phi_x^{(m)}}_\infty
 \le\sum_{j=m}^d\frac{j!}{(j-m)!}\abs{x_j}.
\]
If
\[
 X=1+\sum_{j=1}^d\abs{x_j},
\]
then for fixed $d$ and $r$,
\[
 \norm{f_x^{(r)}}_\infty\le C_{r,d}X^r.
\]
The exponent $r$ is sufficient for scale analysis, although the exact polynomial in the $x_j$ may be much smaller.

\section{Finite differences versus derivatives}

The integral representation gives
\[
 \Dp^rf(k/N)
 =\int_{[0,1/N]^r}
 f^{(r)}(k/N+s_1+\cdots+s_r)\dd s.
\]
Thus
\[
 \abs{\Dp^rf(k/N)}\le N^{-r}\norm{f^{(r)}}_\infty.
\]
For analytic phases, one may also expand the finite difference exactly in powers of $1/N$. The derivative bound is uniform and sufficient for the high-order theorem.

\section{Verification of the displayed third derivative}

For symbolic formulas, complex powers are simplified before the derivative is inserted into an asymptotic expansion. The identity $\ii^3=-\ii$ changes the signed third-order term while leaving the corresponding absolute-value Bell-polynomial bound unchanged. A direct differentiation provides the human-readable verification, and computer algebra supplies an independent check.

\chapter{Background on Bounded Variation}

\section{Definition}

A real function $f$ on $[a,b]$ belongs to $BV([a,b])$ if
\[
 \Var_{[a,b]}(f)
 =\sup_P\sum_{j=0}^{m-1}\abs{f(t_{j+1})-f(t_j)}<\infty,
\]
where the supremum is over all finite partitions
\[
 a=t_0<t_1<\cdots<t_m=b.
\]
Standard references include \cite{AmbrosioFuscoPallara2000,EvansGariepy2015}.

\section{Absolutely continuous functions}

If $f$ is absolutely continuous, then $f'$ exists almost everywhere, $f'\in L^1$, and
\[
 f(y)-f(x)=\int_x^yf'(s)\dd s.
\]
Moreover,
\[
 \Var_{[a,b]}(f)=\int_a^b\abs{f'(s)}\dd s.
\]
The parabolic phase belongs to this class.

\section{Composition with a Lipschitz map}

If $F:\R\to\C$ is Lipschitz with constant $L$ and $f\in BV$, then
\[
 \Var(F\circ f)\le L\Var(f).
\]
For
\[
 F(y)=e^{2\pi\ii y},
\]
one may take $L=2\pi$ because
\[
 \abs{F(y)-F(z)}\le2\pi\abs{y-z}.
\]
This is the functional principle behind the BV phase theorem.

\section{Choice of representative}

A BV function is often treated as an equivalence class modulo sets of measure zero. A sampled extension operator depends on point values, so a representative must be fixed. The right-continuous representative, left-continuous representative, or precise representative may be used. For smooth phases, all agree.

If a jump occurs exactly at a grid point, changing the representative can alter one weight. The weighted variation theorem remains valid after the chosen representative is used consistently, but the numerical value of the sum may change.

\section{Jordan decomposition}

Every real BV function can be written as the difference of two increasing functions. This is the Jordan decomposition. It implies that BV phases have finite one-dimensional oscillation even when they are not differentiable.

The first-order sigma-delta estimate therefore tolerates jump discontinuities. It does not require curvature or smoothness. Higher-order finite-difference estimates, by contrast, require higher regularity unless one replaces classical derivatives with measures and develops a higher-order BV theory.

\section{Complex BV functions}

For complex $w$, one may define
\[
 \Var(w)=\sup_P\sum\abs{w(t_{j+1})-w(t_j)}.
\]
This geometric definition is stronger than treating the real and imaginary parts independently in some estimates, and it is the natural quantity for unit-circle weights.

\section{Discrete sampling inequality}

For any ordered sample points
\[
 a\le t_1<\cdots<t_N\le b,
\]
one has
\[
 \sum_{k=1}^{N-1}\abs{f(t_{k+1})-f(t_k)}
 \le\Var_{[a,b]}(f).
\]
The inequality follows directly from the definition because the sample points form a partition after adding the endpoints if necessary.

\section{Higher-dimensional BV}

In several dimensions, BV functions have distributional gradients that are finite vector-valued measures. The grid theorem in Chapter~\ref{chap:extensions} uses $C^1$ phases for simplicity. A multidimensional BV extension may be possible by controlling directional discrete variations with the total variation measures of the partial derivatives. The boundary and representative issues are more involved.

\chapter{Reference Map and Theoretical Sources}

This appendix records why the major references are included. It is intended to prevent unsupported theoretical statements and to make later paper extraction easier.

\begin{longtable}{p{0.30\textwidth}p{0.64\textwidth}}
\toprule
Reference & Role in the manuscript\\
\midrule
\endfirsthead
\toprule
Reference & Role in the manuscript\\
\midrule
\endhead
\cite{Bennett1948} & Classical spectral model for quantization error and an early foundation for quantization-noise analysis.\\
\cite{Inose1962} & Early delta-sigma modulation architecture and historical origin of feedback noise shaping.\\
\cite{Candy1985} & Second-order integration in sigma-delta modulation and the engineering development of higher-order shaping.\\
\cite{GrayChouWong1989} & Spectral behaviour of single-loop sigma-delta modulation with sinusoidal inputs.\\
\cite{Gray1990} & Rigorous and critical treatment of quantization-noise spectra.\\
\cite{GunturkLagariasVaishampayan2001} & Robustness and dynamical stability of single-loop sigma-delta modulation.\\
\cite{DaubechiesDeVore2003} & Stable sigma-delta modulators of arbitrary order for bandlimited functions.\\
\cite{Gunturk2003CPAM} & Exponential accuracy for one-bit sigma-delta quantization.\\
\cite{Gunturk2004JAMS} & Improved error estimates for coarsely quantized bandlimited data.\\
\cite{DaubechiesDeVoreGunturkVaishampayan2006} & Imperfect quantizers and robustness of analogue-to-digital conversion.\\
\cite{GunturkThao2005} & Ergodic dynamics, invariant sets, and spectral analysis of sigma-delta error.\\
\cite{DeiftGunturkKrahmer2011} & Optimal families of exponentially accurate one-bit schemes.\\
\cite{KrahmerWard2010} & Lower bounds for coarse quantization accuracy.\\
\cite{Gunturk2012} & Authoritative overview of the mathematics of analogue-to-digital conversion.\\
\cite{DaubechiesSaab2015} & Deterministic decimation analysis and bit-rate reduction for sigma-delta streams.\\
\cite{ChouEtAl2015} & Survey of noise-shaping methods for frames and compressive sampling.\\
\cite{GoyalVetterliThao1998} & Quantized overcomplete expansions and the role of redundancy.\\
\cite{BoufounosOppenheim2006} & Noise shaping on arbitrary frame expansions.\\
\cite{BenedettoPowellYilmaz2006} & First-order sigma-delta quantization for finite frames and frame variation.\\
\cite{BenedettoPowellYilmazSecond2006} & Second-order finite-frame quantization and limits caused by frame geometry.\\
\cite{BodmannPaulsen2007} & Frame paths and geometric error bounds.\\
\cite{BodmannPaulsenAbdulbaki2007} & Smooth endpoint termination for higher-order frame quantization.\\
\cite{BlumEtAl2010} & Sobolev dual frames and deterministic $O(N^{-r})$ reconstruction.\\
\cite{GunturkPowellSaabYilmaz2013} & Sobolev duals for random frames and compressed sensing measurements.\\
\cite{KrahmerSaabYilmaz2014} & Sub-Gaussian frames, root-exponential accuracy, and compressed sensing.\\
\cite{SaabWangYilmaz2018} & Stable and robust recovery from quantized compressive samples.\\
\cite{FengKrahmerSaab2017} & Structured random circulant measurements with sigma-delta quantization.\\
\cite{GaoKrahmerPowell2021} & High-order low-bit quantization for fusion frames.\\
\cite{GrafKrahmerKrauseSolberg2019} & One-bit sigma-delta modulation on the circle and periodic-domain issues.\\
\cite{KrahmerVeselovska2023} & Weighted sigma-delta modulation for digital halftoning and multidimensional error diffusion.\\
\cite{DuffinSchaeffer1952} & Foundational frame theory through nonharmonic Fourier series.\\
\cite{BenedettoFickus2003} & Finite normalized tight frames.\\
\cite{CasazzaKutyniok2012} & General reference for finite-frame theory and applications.\\
\cite{CandesRombergTao2006a} & Exact recovery from incomplete Fourier information and compressed-sensing foundations.\\
\cite{CandesRombergTao2006b} & Stable recovery from inaccurate measurements.\\
\cite{Donoho2006} & Foundational compressed-sensing theory.\\
\cite{FoucartRauhut2013} & Comprehensive mathematical reference for compressed sensing.\\
\cite{Fefferman1970} & Early restriction-related inequalities for singular convolution operators.\\
\cite{Tomas1975} & Classical Fourier restriction theorem.\\
\cite{Strichartz1977} & Restriction to quadratic surfaces and dispersive space-time estimates.\\
\cite{KeelTao1998} & Endpoint Strichartz estimates.\\
\cite{Bourgain1993I} & Lattice restriction and periodic Schr\"odinger equations.\\
\cite{Bourgain1993II} & Lattice restriction and periodic KdV equations.\\
\cite{HuLi2011} & Discrete Fourier restriction associated with Schr\"odinger equations.\\
\cite{LaiDing2017} & General discrete Fourier restriction results and related applications.\\
\cite{TaoVargasVega1998} & Bilinear restriction and Kakeya methods.\\
\cite{Wolff2001} & Sharp bilinear cone restriction.\\
\cite{Tao2003} & Sharp bilinear restriction for paraboloids.\\
\cite{BennettCarberyTao2006} & Multilinear restriction and Kakeya inequalities.\\
\cite{BourgainGuth2011} & Oscillatory integral bounds based on multilinear estimates.\\
\cite{Guth2016} & Polynomial partitioning for restriction.\\
\cite{Guth2018} & Further restriction estimates through polynomial partitioning.\\
\cite{BourgainDemeter2015} & Proof of the $\ell^2$ decoupling conjecture and discrete restriction consequences.\\
\cite{BourgainDemeterGuth2016} & Decoupling proof of the main Vinogradov mean value conjecture in higher degree.\\
\cite{KillipVisan2014} & Scale-invariant Strichartz estimates on tori.\\
\cite{Demeter2020} & Systematic reference for restriction, decoupling, and applications.\\
\cite{Weyl1916} & Foundational equidistribution and polynomial exponential-sum method.\\
\cite{Vinogradov1954} & Classical trigonometric-sum method.\\
\cite{Hua1965} & Additive number theory and mean-value methods.\\
\cite{Vaughan1997} & Hardy--Littlewood circle method and exponential sums.\\
\cite{GrahamKolesnik1991} & Detailed van der Corput method.\\
\cite{Montgomery1994} & Interface between analytic number theory and harmonic analysis.\\
\cite{IwaniecKowalski2004} & Modern analytic number theory reference.\\
\cite{Wooley2012} & Efficient congruencing for Vinogradov mean values.\\
\cite{FordWooley2014} & Strongly diagonal behaviour via efficient congruencing.\\
\cite{Wooley2016} & Cubic case of the main Vinogradov conjecture.\\
\cite{Wooley2017} & Approximation to the main Vinogradov mean value conjecture.\\
\cite{Wooley2019} & Nested efficient congruencing.\\
\cite{ParsellPrendivilleWooley2013} & Multidimensional Weyl sums and near-optimal mean values.\\
\cite{GuoZhang2019} & Integer solutions of Parsell--Vinogradov systems.\\
\cite{GuoLiYungZorinKranich2020} & Short proof of moment-curve decoupling.\\
\cite{Bourgain2017} & Decoupling, exponential sums, and the Riemann zeta function.\\
\cite{AmbrosioFuscoPallara2000} & Standard monograph on BV functions and free-discontinuity theory.\\
\cite{EvansGariepy2015} & Measure-theoretic properties of Sobolev and BV functions.\\
\cite{Zygmund2002} & Trigonometric series, Abel summation, and classical Fourier analysis.\\
\cite{Grafakos2014} & Classical Fourier analysis and operator estimates.\\
\cite{Stein1993} & Oscillatory integrals and harmonic-analysis methods.\\
\cite{SteinShakarchi2003} & Introductory Fourier analysis and summation methods.\\
\cite{Tao2006} & Nonlinear dispersive equations and Strichartz theory.\\
\cite{Cazenave2003} & Schr\"odinger equations and dispersive PDE background.\\
\cite{Apostol1976} & Abel summation and analytic number-theory preliminaries.\\
\cite{HardyLittlewoodPolya1952} & Classical inequalities used in norm and interpolation estimates.\\
\bottomrule
\end{longtable}

The table contains more references than a single journal paper would normally require. When a chapter is extracted, retain only the sources that support its actual theory and context.

\chapter{Reproducible Code}

\section{Python implementation of the greedy quantizer}

\begin{lstlisting}[language=Python]
import numpy as np


def greedy_sigma_delta(u: np.ndarray) -> tuple[np.ndarray, np.ndarray]:
    """First-order one-bit greedy Sigma-Delta quantizer.

    Parameters
    ----------
    u : one-dimensional real array with |u[k]| <= 1

    Returns
    -------
    q : one-bit output in {-1, +1}
    v : state array including v[0] = 0
    """
    u = np.asarray(u, dtype=float)
    if u.ndim != 1:
        raise ValueError("u must be one-dimensional")
    if np.max(np.abs(u)) > 1 + 1e-12:
        raise ValueError("input exceeds the proved stability range")

    n = u.size
    q = np.empty(n, dtype=float)
    v = np.zeros(n + 1, dtype=float)

    for k in range(n):
        y = u[k] + v[k]
        q[k] = 1.0 if y >= 0.0 else -1.0
        v[k + 1] = v[k] + u[k] - q[k]

    return q, v
\end{lstlisting}

\section{Parabolic extension and theoretical bound}

\begin{lstlisting}[language=Python]
def parabolic_extension(a, x, t):
    a = np.asarray(a, dtype=complex)
    n = a.size
    k = np.arange(1, n + 1, dtype=float)
    phase = x * k / n + t * (k / n) ** 2
    return np.mean(a * np.exp(2j * np.pi * phase))


def J_closed(x, t):
    if abs(t) < 1e-15:
        return abs(x)
    if x * (x + 2.0 * t) >= 0.0:
        return abs(x + t)
    return (x * x + (x + 2.0 * t) ** 2) / (4.0 * abs(t))


def first_order_bound(n, x, t, state_bound=1.0):
    return state_bound / n * (1.0 + 2.0 * np.pi * J_closed(x, t))
\end{lstlisting}

\section{Identity check}

\begin{lstlisting}[language=Python]
def check_first_order_identity(u, q, v):
    error = np.asarray(u) - np.asarray(q)
    state_difference = np.diff(np.asarray(v))
    return np.max(np.abs(error - state_difference))


def experiment_zero_input(n, x=0.0, t=0.0):
    u = np.zeros(n)
    q, v = greedy_sigma_delta(u)
    err = abs(parabolic_extension(u, x, t)
              - parabolic_extension(q, x, t))
    return {
        "N": n,
        "max_state": float(np.max(np.abs(v))),
        "terminal_state": float(v[-1]),
        "extension_error": float(err),
        "theoretical_bound": float(first_order_bound(n, x, t)),
    }
\end{lstlisting}

\section{Synthetic high-order identity test}

\begin{lstlisting}[language=Python]
def backward_difference(a, order=1):
    out = np.asarray(a, dtype=complex)
    for _ in range(order):
        out = np.diff(out)
    return out


def forward_difference(a, order=1):
    out = np.asarray(a, dtype=complex)
    for _ in range(order):
        out = np.diff(out)
    return out

# For an order-r test, create v on indices 1-r,...,N,
# compute e = Delta^r v on k=1,...,N, and compare the
# direct weighted sum with the boundary formula in Theorem 8.1.
\end{lstlisting}

\section{MATLAB-style implementation}

\begin{lstlisting}[language=Matlab]
function [q,v] = greedy_sigma_delta(u)
    u = u(:);
    N = length(u);
    if max(abs(u)) > 1 + 1e-12
        error('Input exceeds the proved stability range');
    end

    q = zeros(N,1);
    v = zeros(N+1,1);

    for k = 1:N
        y = u(k) + v(k);
        if y >= 0
            q(k) = 1;
        else
            q(k) = -1;
        end
        v(k+1) = v(k) + u(k) - q(k);
    end
end
\end{lstlisting}

\section{Convergence-order script outline}

\begin{lstlisting}[language=Python]
lengths = 2 ** np.arange(8, 17)
errors = []

for n in lengths:
    k = np.arange(1, n + 1)
    u = 0.7 * np.sin(2 * np.pi * 7 * k / n)
    q, v = greedy_sigma_delta(u)
    errors.append(abs(parabolic_extension(u, 0.4, -0.3)
                      - parabolic_extension(q, 0.4, -0.3)))

slope, intercept = np.polyfit(np.log(lengths), np.log(errors), 1)
print("estimated order:", slope)
\end{lstlisting}

\section{Numerical verification notes}

Evaluate the phase in double precision with moderate parameter sizes. For very large $x$ or $t$, reduction modulo one before the exponential improves numerical conditioning. The zero convention in the quantizer is matched to the theorem. Odd and even record lengths are fitted separately when the terminal state produces a parity effect.

\chapter{Glossary of Terms}

\begin{description}[style=nextline,leftmargin=4.2cm,labelwidth=3.8cm]
\item[Admissible input] An input sequence satisfying the amplitude or regularity assumptions required by a stated quantizer theorem.

\item[Boundary trace] A value of the state or one of its finite differences appearing at the initial or terminal index after summation by parts.

\item[Bounded variation] Finite total accumulated absolute change of a function or sequence.

\item[Critical scale] A parameter range where adjacent phase increments are order one and the elementary total-variation gain no longer yields decay.

\item[Decoupling] A family of harmonic-analysis inequalities that control an extension operator by an $\ell^2$ combination of contributions from smaller frequency pieces.

\item[Discrete extension] A finite oscillatory sum that is dual to a restriction problem or represents a periodic dispersive evolution.

\item[Endpoint compatibility] Vanishing of the initial and terminal discrete traces needed to remove boundary terms in high-order summation by parts.

\item[Error shaping] Organisation of quantization error so that it lies in the range of a difference or filter operator.

\item[Greedy quantizer] A rule that selects the current alphabet symbol according to the sign or nearest-level decision based on the current input and state.

\item[Invariant region] A set of states mapped into itself by the nonlinear quantizer recursion for every admissible input.

\item[Major arc] A parameter region close to a rational point with small denominator in the circle method.

\item[Minor arc] The complement of the major arcs, where stronger exponential-sum cancellation is often available.

\item[Moment curve] The curve $\xi\mapsto(\xi,\xi^2,\ldots,\xi^d)$.

\item[Noise transfer function] The linearised transfer function from quantizer noise to output in an engineering sigma-delta model.

\item[One-bit alphabet] The two-level set $\{-1,+1\}$, possibly after scaling.

\item[Oversampling] Sampling above the minimum rate needed for representation, thereby creating redundancy that may be used for quantization accuracy.

\item[Parabolic phase] A phase of the form $x\xi+t\xi^2$.

\item[Restriction estimate] An inequality controlling the Fourier transform on a curved set, or equivalently the associated extension operator.

\item[Shaped residual] The component of the error represented as a finite difference of a bounded state.

\item[State sequence] The internal memory variable of a feedback quantizer.

\item[Strichartz estimate] A space-time integrability estimate for a dispersive PDE solution.

\item[Subcritical region] A growing parameter region in which adjacent phase increments still tend to zero.

\item[Total variation bound] An estimate obtained by summation by parts followed by absolute values of all weight increments.

\item[Unshaped residual] An error component not represented by the chosen difference operator.

\item[Vinogradov mean value theorem] A family of sharp mean-value estimates for polynomial exponential sums, equivalent to counting solutions of systems of equal sums of powers.

\item[Weyl sum] A finite exponential sum with polynomial phase.
\end{description}

\section{Symbols most likely to be confused}

The symbol $\D$ is the backward finite difference on sequences. The symbol $\Dp$ is the forward finite difference. The operator $\E_N$ is normalised by $1/N$. The unnormalised exponential sum is $N\E_N$. The state bound $V$ is independent of $N$ unless explicitly stated otherwise.

\bibliographystyle{unsrtnat}
\bibliography{references}

@article{Bennett1948,
  author={Bennett, W. R.}, title={Spectra of quantized signals},
  journal={Bell System Technical Journal}, volume={27}, year={1948}, pages={446--472}}

@article{Inose1962,
  author={Inose, Hiroshi and Yasuda, Yasuhiko and Murakami, Jiro},
  title={A telemetering system by code modulation: Delta-Sigma modulation},
  journal={IRE Transactions on Space Electronics and Telemetry}, volume={SET-8}, number={3}, year={1962}, pages={204--209}}

@article{Candy1985,
  author={Candy, James C.}, title={A use of double integration in sigma delta modulation},
  journal={IEEE Transactions on Communications}, volume={33}, number={3}, year={1985}, pages={249--258}}

@article{Gray1990,
  author={Gray, Robert M.}, title={Quantization noise spectra},
  journal={IEEE Transactions on Information Theory}, volume={36}, number={6}, year={1990}, pages={1220--1244}}

@article{GrayChouWong1989,
  author={Gray, Robert M. and Chou, Wen-Hsiung and Wong, Ping W.},
  title={Quantization noise in single-loop sigma-delta modulation with sinusoidal inputs},
  journal={IEEE Transactions on Communications}, volume={37}, number={9}, year={1989}, pages={956--968}}

@article{GunturkLagariasVaishampayan2001,
  author={G{\"u}nt{\"u}rk, C. Sinan and Lagarias, Jeffrey C. and Vaishampayan, Vinay A.},
  title={On the robustness of single-loop sigma-delta modulation},
  journal={IEEE Transactions on Information Theory}, volume={47}, number={5}, year={2001}, pages={1735--1744}}

@article{DaubechiesDeVore2003,
  author={Daubechies, Ingrid and DeVore, Ronald},
  title={Approximating a bandlimited function using very coarsely quantized data: A family of stable sigma-delta modulators of arbitrary order},
  journal={Annals of Mathematics}, volume={158}, number={2}, year={2003}, pages={679--710}}

@article{Gunturk2003CPAM,
  author={G{\"u}nt{\"u}rk, C. Sinan}, title={One-bit sigma-delta quantization with exponential accuracy},
  journal={Communications on Pure and Applied Mathematics}, volume={56}, year={2003}, pages={1608--1630}}

@article{Gunturk2004JAMS,
  author={G{\"u}nt{\"u}rk, C. Sinan},
  title={Approximating a bandlimited function using very coarsely quantized data: Improved error estimates in sigma-delta modulation},
  journal={Journal of the American Mathematical Society}, volume={17}, number={1}, year={2004}, pages={229--242}}

@article{DaubechiesDeVoreGunturkVaishampayan2006,
  author={Daubechies, Ingrid and DeVore, Ronald and G{\"u}nt{\"u}rk, C. Sinan and Vaishampayan, Vinay A.},
  title={A/D conversion with imperfect quantizers}, journal={IEEE Transactions on Information Theory},
  volume={52}, number={3}, year={2006}, pages={874--885}}

@article{GunturkThao2005,
  author={G{\"u}nt{\"u}rk, C. Sinan and Thao, Nguyen T.},
  title={Ergodic dynamics in sigma-delta quantization: Tiling invariant sets and spectral analysis of error},
  journal={Advances in Applied Mathematics}, volume={34}, number={3}, year={2005}, pages={523--560}}

@article{DeiftGunturkKrahmer2011,
  author={Deift, Percy and Krahmer, Felix and G{\"u}nt{\"u}rk, C. Sinan},
  title={An optimal family of exponentially accurate one-bit sigma-delta quantization schemes},
  journal={Communications on Pure and Applied Mathematics},
  volume={64}, number={7}, year={2011}, pages={883--919},
  doi={10.1002/cpa.20367}}

@article{KrahmerWard2010,
  author={Krahmer, Felix and Ward, Rachel},
  title={Lower bounds for the error decay incurred by coarse quantization schemes},
  journal={Applied and Computational Harmonic Analysis},
  volume={32}, number={1}, year={2012}, pages={131--138},
  doi={10.1016/j.acha.2011.06.003},
  eprint={1004.3517}, archivePrefix={arXiv}}

@article{Gunturk2012,
  author={G{\"u}nt{\"u}rk, C. Sinan}, title={Mathematics of analog-to-digital conversion},
  journal={Communications on Pure and Applied Mathematics}, volume={65}, number={12}, year={2012}, pages={1671--1696}}

@article{DaubechiesSaab2015,
  author={Daubechies, Ingrid and Saab, Rayan}, title={A deterministic analysis of decimation for sigma-delta quantization of bandlimited functions},
  journal={IEEE Signal Processing Letters}, volume={22}, number={11}, year={2015}, pages={2093--2096}, doi={10.1109/LSP.2015.2459758}}

@misc{ChouEtAl2015,
  author={Chou, Evan and G{\"u}nt{\"u}rk, C. Sinan and Krahmer, Felix and Saab, Rayan and Yilmaz, {\"O}zg{\"u}r},
  title={Noise-shaping quantization methods for frame-based and compressive sampling systems},
  year={2015}, eprint={1502.05807}, archivePrefix={arXiv}, primaryClass={math.NA}}

@article{BenedettoPowellYilmaz2006,
  author={Benedetto, John J. and Powell, Alexander M. and Yilmaz, {\"O}zg{\"u}r},
  title={Sigma-delta quantization and finite frames}, journal={IEEE Transactions on Information Theory},
  volume={52}, number={5}, year={2006}, pages={1990--2005}}

@article{BenedettoPowellYilmazSecond2006,
  author={Benedetto, John J. and Powell, Alexander M. and Yilmaz, {\"O}zg{\"u}r},
  title={Second-order sigma-delta quantization of finite frame expansions},
  journal={Applied and Computational Harmonic Analysis}, volume={20}, number={1}, year={2006}, pages={126--148}}

@article{BodmannPaulsen2007,
  author={Bodmann, Bernhard G. and Paulsen, Vern I.}, title={Frame paths and error bounds for sigma-delta quantization},
  journal={Applied and Computational Harmonic Analysis}, volume={22}, number={2}, year={2007}, pages={176--197}}

@article{BodmannPaulsenAbdulbaki2007,
  author={Bodmann, Bernhard G. and Paulsen, Vern I. and Abdulbaki, Soha A.},
  title={Smooth frame-path termination for higher order sigma-delta quantization},
  journal={Journal of Fourier Analysis and Applications}, volume={13}, number={3}, year={2007}, pages={285--307}}

@article{BlumEtAl2010,
  author={Blum, James E. and Lammers, Mark C. and Powell, Alexander M. and Yilmaz, {\"O}zg{\"u}r},
  title={Sobolev duals in frame theory and sigma-delta quantization},
  journal={Journal of Fourier Analysis and Applications}, volume={16}, number={3}, year={2010}, pages={365--381}}

@article{GunturkPowellSaabYilmaz2013,
  author={G{\"u}nt{\"u}rk, C. Sinan and Lammers, Mark C. and Powell, Alexander M. and Saab, Rayan and Yilmaz, {\"O}zg{\"u}r},
  title={Sobolev duals for random frames and sigma-delta quantization of compressed sensing measurements},
  journal={Foundations of Computational Mathematics},
  volume={13}, number={1}, year={2013}, pages={1--36},
  doi={10.1007/s10208-012-9140-x}}

@article{KrahmerSaabYilmaz2014,
  author={Krahmer, Felix and Saab, Rayan and Yilmaz, {\"O}zg{\"u}r},
  title={Sigma-delta quantization of sub-Gaussian frame expansions and its application to compressed sensing},
  journal={Information and Inference}, volume={3}, number={1}, year={2014}, pages={40--58}}

@article{SaabWangYilmaz2018,
  author={Saab, Rayan and Wang, Rongrong and Yilmaz, {\"O}zg{\"u}r},
  title={Quantization of compressive samples with stable and robust recovery},
  journal={Applied and Computational Harmonic Analysis}, volume={44}, number={1}, year={2018}, pages={123--143}}

@article{FengKrahmerSaab2017,
  author={Feng, Joe-Mei and Krahmer, Felix and Saab, Rayan},
  title={Quantized compressed sensing for random circulant matrices},
  journal={Applied and Computational Harmonic Analysis},
  volume={47}, number={3}, year={2019}, pages={1014--1032},
  doi={10.1016/j.acha.2019.03.004},
  eprint={1702.04711}, archivePrefix={arXiv}}

@article{GaoKrahmerPowell2021,
  author={Gao, Zhen and Krahmer, Felix and Powell, Alexander M.},
  title={High-order low-bit sigma-delta quantization for fusion frames},
  journal={Analysis and Applications}, volume={19}, number={1}, year={2021}, pages={1--20},
  doi={10.1142/S0219530520400096},
  eprint={2006.09732}, archivePrefix={arXiv}}

@misc{GrafKrahmerKrauseSolberg2019,
  author={Graf, Olga and Krahmer, Felix and Krause-Solberg, Sara},
  title={One-bit sigma-delta modulation on the circle}, year={2019}, eprint={1911.07647}, archivePrefix={arXiv}, primaryClass={math.NA}}

@article{GoyalVetterliThao1998,
  author={Goyal, Vivek K. and Vetterli, Martin and Thao, Nguyen T.},
  title={Quantized overcomplete expansions in $\mathbb{R}^N$: Analysis, synthesis, and algorithms},
  journal={IEEE Transactions on Information Theory}, volume={44}, number={1}, year={1998}, pages={16--31}}

@article{BoufounosOppenheim2006,
  author={Boufounos, Petros T. and Oppenheim, Alan V.}, title={Quantization noise shaping on arbitrary frame expansions},
  journal={EURASIP Journal on Applied Signal Processing}, year={2006}, pages={1--12}}

@article{CandesRombergTao2006a,
  author={Cand{\`e}s, Emmanuel J. and Romberg, Justin and Tao, Terence},
  title={Robust uncertainty principles: Exact signal reconstruction from highly incomplete frequency information},
  journal={IEEE Transactions on Information Theory}, volume={52}, number={2}, year={2006}, pages={489--509}}

@article{CandesRombergTao2006b,
  author={Cand{\`e}s, Emmanuel J. and Romberg, Justin K. and Tao, Terence},
  title={Stable signal recovery from incomplete and inaccurate measurements},
  journal={Communications on Pure and Applied Mathematics}, volume={59}, number={8}, year={2006}, pages={1207--1223}}

@article{Donoho2006,
  author={Donoho, David L.}, title={Compressed sensing}, journal={IEEE Transactions on Information Theory},
  volume={52}, number={4}, year={2006}, pages={1289--1306}}

@book{FoucartRauhut2013,
  author={Foucart, Simon and Rauhut, Holger}, title={A Mathematical Introduction to Compressive Sensing},
  publisher={Birkh{\"a}user}, year={2013}}

@article{DuffinSchaeffer1952,
  author={Duffin, Richard J. and Schaeffer, Albert C.}, title={A class of nonharmonic Fourier series},
  journal={Transactions of the American Mathematical Society}, volume={72}, year={1952}, pages={341--366}}

@article{BenedettoFickus2003,
  author={Benedetto, John J. and Fickus, Matthew}, title={Finite normalized tight frames},
  journal={Advances in Computational Mathematics}, volume={18}, year={2003}, pages={357--385}}

@book{CasazzaKutyniok2012,
  editor={Casazza, Peter G. and Kutyniok, Gitta}, title={Finite Frames: Theory and Applications},
  publisher={Birkh{\"a}user}, year={2012}}

@article{Fefferman1970,
  author={Fefferman, Charles}, title={Inequalities for strongly singular convolution operators},
  journal={Acta Mathematica}, volume={124}, year={1970}, pages={9--36}}

@article{Tomas1975,
  author={Tomas, Peter A.}, title={A restriction theorem for the Fourier transform},
  journal={Bulletin of the American Mathematical Society}, volume={81}, year={1975}, pages={477--478}}

@article{Strichartz1977,
  author={Strichartz, Robert S.}, title={Restrictions of Fourier transforms to quadratic surfaces and decay of solutions of wave equations},
  journal={Duke Mathematical Journal}, volume={44}, number={3}, year={1977}, pages={705--714}}

@article{KeelTao1998,
  author={Keel, Markus and Tao, Terence}, title={Endpoint Strichartz estimates},
  journal={American Journal of Mathematics}, volume={120}, number={5}, year={1998}, pages={955--980}}

@article{Bourgain1993I,
  author={Bourgain, Jean}, title={Fourier transform restriction phenomena for certain lattice subsets and applications to nonlinear evolution equations. Part I: Schr{\"o}dinger equations},
  journal={Geometric and Functional Analysis}, volume={3}, year={1993}, pages={107--156}}

@article{Bourgain1993II,
  author={Bourgain, Jean}, title={Fourier transform restriction phenomena for certain lattice subsets and applications to nonlinear evolution equations. Part II: The KdV equation},
  journal={Geometric and Functional Analysis}, volume={3}, year={1993}, pages={209--262}}

@article{HuLi2011,
  author={Hu, Yi and Li, Xiaochun},
  title={Discrete Fourier restriction associated with Schr{\"o}dinger equations},
  journal={Revista Matem\'atica Iberoamericana}, volume={30}, number={4}, year={2014}, pages={1281--1300},
  doi={10.4171/RMI/815},
  eprint={1108.5164}, archivePrefix={arXiv}}

@misc{LaiDing2017,
  author={Lai, Xudong and Ding, Yong}, title={A note on the discrete Fourier restriction problem},
  year={2017}, eprint={1710.01481}, archivePrefix={arXiv}, primaryClass={math.CA}}

@article{BourgainDemeter2015,
  author={Bourgain, Jean and Demeter, Ciprian}, title={The proof of the $\ell^2$ decoupling conjecture},
  journal={Annals of Mathematics}, volume={182}, number={1}, year={2015}, pages={351--389}}

@article{BourgainDemeterGuth2016,
  author={Bourgain, Jean and Demeter, Ciprian and Guth, Larry}, title={Proof of the main conjecture in Vinogradov's mean value theorem for degrees higher than three},
  journal={Annals of Mathematics}, volume={184}, number={2}, year={2016}, pages={633--682}}

@article{BourgainGuth2011,
  author={Bourgain, Jean and Guth, Larry}, title={Bounds on oscillatory integral operators based on multilinear estimates},
  journal={Geometric and Functional Analysis}, volume={21}, number={6}, year={2011}, pages={1239--1295}}

@article{BennettCarberyTao2006,
  author={Bennett, Jonathan and Carbery, Anthony and Tao, Terence}, title={On the multilinear restriction and Kakeya conjectures},
  journal={Acta Mathematica}, volume={196}, number={2}, year={2006}, pages={261--302}}

@article{TaoVargasVega1998,
  author={Tao, Terence and Vargas, Ana and Vega, Luis}, title={A bilinear approach to the restriction and Kakeya conjectures},
  journal={Journal of the American Mathematical Society}, volume={11}, number={4}, year={1998}, pages={967--1000}}

@article{Tao2003,
  author={Tao, Terence}, title={A sharp bilinear restriction estimate for paraboloids},
  journal={Geometric and Functional Analysis}, volume={13}, year={2003}, pages={1359--1384}}

@article{Wolff2001,
  author={Wolff, Thomas}, title={A sharp bilinear cone restriction estimate},
  journal={Annals of Mathematics}, volume={153}, number={3}, year={2001}, pages={661--698}}

@article{Guth2016,
  author={Guth, Larry}, title={A restriction estimate using polynomial partitioning},
  journal={Journal of the American Mathematical Society}, volume={29}, number={2}, year={2016}, pages={371--413}}

@article{Guth2018,
  author={Guth, Larry}, title={Restriction estimates using polynomial partitioning II},
  journal={Acta Mathematica}, volume={221}, number={1}, year={2018}, pages={81--142}}

@misc{KillipVisan2014,
  author={Killip, Rowan and Visan, Monica}, title={Scale invariant Strichartz estimates on tori and applications},
  year={2014}, eprint={1409.3603}, archivePrefix={arXiv}, primaryClass={math.AP}}

@book{Demeter2020,
  author={Demeter, Ciprian}, title={Fourier Restriction, Decoupling, and Applications},
  publisher={Cambridge University Press}, year={2020}}

@article{Weyl1916,
  author={Weyl, Hermann}, title={Ueber die Gleichverteilung von Zahlen mod. Eins},
  journal={Mathematische Annalen}, volume={77}, year={1916}, pages={313--352}}

@book{Vinogradov1954,
  author={Vinogradov, Ivan M.}, title={The Method of Trigonometrical Sums in the Theory of Numbers},
  publisher={Interscience}, year={1954}}

@book{Hua1965,
  author={Hua, Loo-Keng}, title={Additive Theory of Prime Numbers}, publisher={American Mathematical Society}, year={1965}}

@book{Vaughan1997,
  author={Vaughan, Robert C.}, title={The Hardy--Littlewood Method}, edition={2}, publisher={Cambridge University Press}, year={1997}}

@book{GrahamKolesnik1991,
  author={Graham, S. W. and Kolesnik, Grigori}, title={Van der Corput's Method of Exponential Sums},
  publisher={Cambridge University Press}, year={1991}}

@book{Montgomery1994,
  author={Montgomery, Hugh L.}, title={Ten Lectures on the Interface Between Analytic Number Theory and Harmonic Analysis},
  publisher={American Mathematical Society}, year={1994}}

@book{IwaniecKowalski2004,
  author={Iwaniec, Henryk and Kowalski, Emmanuel}, title={Analytic Number Theory}, publisher={American Mathematical Society}, year={2004}}

@article{Wooley2012,
  author={Wooley, Trevor D.}, title={Vinogradov's mean value theorem via efficient congruencing},
  journal={Annals of Mathematics}, volume={175}, number={3}, year={2012}, pages={1575--1627}}

@article{FordWooley2014,
  author={Ford, Kevin B. and Wooley, Trevor D.}, title={On Vinogradov's mean value theorem: Strongly diagonal behaviour via efficient congruencing},
  journal={Acta Mathematica}, volume={213}, number={2}, year={2014}, pages={199--236}}

@article{Wooley2016,
  author={Wooley, Trevor D.}, title={The cubic case of the main conjecture in Vinogradov's mean value theorem},
  journal={Advances in Mathematics}, volume={294}, year={2016}, pages={532--561}}

@article{Wooley2017,
  author={Wooley, Trevor D.}, title={Approximating the main conjecture in Vinogradov's mean value theorem},
  journal={Mathematika}, volume={63}, number={1}, year={2017}, pages={292--350}}

@article{Wooley2019,
  author={Wooley, Trevor D.}, title={Nested efficient congruencing and relatives of Vinogradov's mean value theorem},
  journal={Proceedings of the London Mathematical Society}, volume={118}, number={4}, year={2019}, pages={942--1016}}

@article{ParsellPrendivilleWooley2013,
  author={Parsell, Scott T. and Prendiville, Sean M. and Wooley, Trevor D.},
  title={Near-optimal mean value estimates for multidimensional Weyl sums},
  journal={Geometric and Functional Analysis}, volume={23}, year={2013}, pages={1962--2024}}

@article{GuoZhang2019,
  author={Guo, Shaoming and Zhang, Ruixiang}, title={On integer solutions of Parsell--Vinogradov systems},
  journal={Inventiones Mathematicae}, volume={218}, number={1}, year={2019}, pages={1--81}}

@article{GuoLiYungZorinKranich2020,
  author={Guo, Shaoming and Li, Zane Kun and Yung, Po-Lam and Zorin-Kranich, Pavel},
  title={A short proof of $\ell^2$ decoupling for the moment curve},
  journal={American Journal of Mathematics}, volume={143}, number={6}, year={2021}, pages={1983--1998},
  doi={10.1353/ajm.2021.0048},
  eprint={1912.09798}, archivePrefix={arXiv}}

@article{Bourgain2017,
  author={Bourgain, Jean}, title={Decoupling, exponential sums and the Riemann zeta function},
  journal={Journal of the American Mathematical Society}, volume={30}, number={1}, year={2017}, pages={205--224}}

@book{AmbrosioFuscoPallara2000,
  author={Ambrosio, Luigi and Fusco, Nicola and Pallara, Diego}, title={Functions of Bounded Variation and Free Discontinuity Problems},
  publisher={Oxford University Press}, year={2000}}

@book{EvansGariepy2015,
  author={Evans, Lawrence C. and Gariepy, Ronald F.}, title={Measure Theory and Fine Properties of Functions},
  edition={Revised}, publisher={CRC Press}, year={2015}}

@book{Zygmund2002,
  author={Zygmund, Antoni}, title={Trigonometric Series}, edition={3}, publisher={Cambridge University Press}, year={2002}}

@book{Grafakos2014,
  author={Grafakos, Loukas}, title={Classical Fourier Analysis}, edition={3}, publisher={Springer}, year={2014}}

@book{Stein1993,
  author={Stein, Elias M.}, title={Harmonic Analysis: Real-Variable Methods, Orthogonality, and Oscillatory Integrals},
  publisher={Princeton University Press}, year={1993}}

@book{SteinShakarchi2003,
  author={Stein, Elias M. and Shakarchi, Rami}, title={Fourier Analysis: An Introduction},
  publisher={Princeton University Press}, year={2003}}

@book{Tao2006,
  author={Tao, Terence}, title={Nonlinear Dispersive Equations: Local and Global Analysis},
  publisher={American Mathematical Society}, year={2006}}

@book{Cazenave2003,
  author={Cazenave, Thierry}, title={Semilinear Schr{\"o}dinger Equations}, publisher={American Mathematical Society}, year={2003}}

@book{Apostol1976,
  author={Apostol, Tom M.}, title={Introduction to Analytic Number Theory}, publisher={Springer}, year={1976}}

@book{HardyLittlewoodPolya1952,
  author={Hardy, G. H. and Littlewood, J. E. and P{\'o}lya, G.}, title={Inequalities}, edition={2}, publisher={Cambridge University Press}, year={1952}}

@article{KrahmerVeselovska2023,
  author={Krahmer, Felix and Veselovska, Anna},
  title={Enhanced digital halftoning via weighted sigma-delta modulation},
  journal={SIAM Journal on Imaging Sciences}, volume={16}, number={3}, pages={1727--1761}, year={2023},
  doi={10.1137/22M151786X}}

@book{DeVoreLorentz1993,
  author={DeVore, Ronald A. and Lorentz, George G.},
  title={Constructive Approximation},
  series={Grundlehren der mathematischen Wissenschaften},
  volume={303},
  publisher={Springer},
  address={Berlin and Heidelberg},
  year={1993},
  isbn={978-3-540-50627-0}}

@book{DitzianTotik1987,
  author={Ditzian, Zeev and Totik, Vilmos},
  title={Moduli of Smoothness},
  series={Springer Series in Computational Mathematics},
  volume={9},
  publisher={Springer},
  address={New York},
  year={1987},
  doi={10.1007/978-1-4612-4778-4}}

@book{Triebel1983,
  author={Triebel, Hans},
  title={Theory of Function Spaces},
  publisher={Birkh{\"a}user},
  address={Basel},
  year={1983},
  doi={10.1007/978-3-0346-0416-1}}
\printindex
\end{document}